\documentclass[letterpaper]{article} 
\usepackage{aaai24}  
\usepackage{times}  
\usepackage{helvet}  
\usepackage{courier}  
\usepackage[hyphens]{url}  
\usepackage{graphicx} 
\usepackage{natbib}  
\usepackage{caption} 
\usepackage{algorithm}
\usepackage{algorithmic}

\usepackage{amsmath,amssymb,amsfonts}
\usepackage{mathtools}
\usepackage{amsthm}
\usepackage{dsfont}
\usepackage{subfigure}
\usepackage{xcolor}
\usepackage{enumitem}
\theoremstyle{plain}
\newtheorem{theorem}{Theorem}
\newtheorem*{theoremstar}{Theorem}
\newtheorem{prop}{Proposition}
\newtheorem*{propstar}{Proposition}
\newtheorem{lemma}{Lemma}

\newtheorem{corollary}{Corollary}
\newtheorem*{corollarystar}{Corollary}
\theoremstyle{definition}
\newtheorem{definition}{Definition}
\newtheorem{assumption}{Assumption}

\newtheorem{remarkapp}{Remark}[section]

\DeclareMathOperator{\sgn}{sign}
\DeclareMathOperator{\cA}{\mathcal{A}}
\DeclareMathOperator{\cS}{\mathcal{S}}
\DeclareMathOperator{\cD}{\mathcal{D}}

\def\sT{{\mathbb{T}}}

\def\gT{{\mathcal{T}}}

\usepackage{newfloat}
\usepackage{listings}
\DeclareCaptionStyle{ruled}{labelfont=normalfont,labelsep=colon,strut=off} 
\floatstyle{ruled}
\newfloat{listing}{tb}{lst}{}
\floatname{listing}{Listing}
\title{PPO-Clip Attains Global Optimality: Towards Deeper Understandings of Clipping}
\author{
    Nai-Chieh Huang, Ping-Chun Hsieh, Kuo-Hao Ho, I-Chen Wu 
}
\affiliations{

    Department of Computer Science, National Yang Ming Chiao Tung University, Hsinchu, Taiwan\\
    
    \texttt{\{naich.cs09, pinghsieh\}@nycu.edu.tw}
}

\usepackage{bibentry}

\begin{document}

\maketitle

\begin{abstract}
Proximal Policy Optimization algorithm employing a clipped surrogate objective (PPO-Clip) is a prominent exemplar of the policy optimization methods. 
However, despite its remarkable empirical success, PPO-Clip lacks theoretical substantiation to date.
In this paper, we contribute to the field by establishing the first global convergence results of a PPO-Clip variant in both tabular and neural function approximation settings. Our findings highlight the $O(1/\sqrt{T})$ min-iterate convergence rate specifically in the context of neural function approximation.
We tackle the inherent challenges in analyzing PPO-Clip through three central concepts: (i) We introduce a generalized version of the PPO-Clip objective, illuminated by its connection with the hinge loss.
(ii) Employing entropic mirror descent, we establish asymptotic convergence for tabular PPO-Clip with direct policy parameterization.
(iii) Inspired by the tabular analysis, we streamline convergence analysis by introducing a two-step policy improvement approach. This decouples policy search from complex neural policy parameterization using a regression-based update scheme.
Furthermore, we gain deeper insights into the efficacy of PPO-Clip by interpreting these generalized objectives. Our theoretical findings also mark the first characterization of the influence of the clipping mechanism on PPO-Clip convergence. Importantly, the clipping range affects only the pre-constant of the convergence rate.
\end{abstract}

\section{Introduction}

Policy optimization is a prevalent method for solving reinforcement learning problems, involving iterative parameter updates to maximize objectives. Policy gradient methods, a prominent subset of this approach, were introduced as a direct solution using gradient descent.
Their primary aim is to identify an optimal policy that maximizes the total expected reward through interactions with the environment.
The selection of an appropriate step size is crucial as it significantly influences policy gradient algorithm performance. Addressing this challenge, Trust Region Policy Optimization (TRPO) was created \citep{schulman2015trust}. Utilizing a trust-region approach with a second-order approximation, TRPO guarantees substantial policy improvement.
Unlike computationally intensive TRPO, Proximal Policy Optimization (PPO) \citep{schulman2017proximal} leverages first-order derivatives for policy improvement. PPO encompasses two main variants: PPO-KL and PPO-Clip, each with distinct characteristics. PPO-KL adds a Kullback-Leibler divergence penalty to the objective, while PPO-Clip integrates probability ratio clipping.
These variants showcase remarkable performance across various environments, with PPO standing out for its computational efficiency \citep{chen2018adaptive, ye2020mastering, byun2020proximal}.

Given the empirical success of these policy optimization algorithms, recent works have made significant strides in enhancing their theoretical guarantees.
In particular, \citep{agarwal2020optimality,bhandari2019global} prove the global convergence result of the policy gradient algorithm under different settings. Additionally, \citep{mei2020global} establishes the convergence rates of the softmax policy gradient in both the standard and the entropy-regularized settings. Furthermore, it has been shown that various policy gradient algorithms also enjoy global convergence \citep{fazel2018global,liu2020improved, wang2021global}.
In the context of TRPO and PPO, \citep{shani2020adaptive} have utilized the mirror descent method to establish the convergence rate of adaptive TRPO under both the standard and entropy-regularized settings. Furthermore, \citep{liu2019neural} have provided the convergence rate of PPO-KL and TRPO under neural function approximation.\footnote{For the detailed discussion about related work, please refer to Appendix H.}
By contrast, despite that PPO-Clip is computationally efficient and empirically successful, the following question about the theory of PPO-Clip remains largely open: \textit{Does PPO-Clip enjoy provable global convergence or have any convergence rate guarantee?}

In this paper, we answer the above question affirmatively.
To begin with, we generalize the PPO-Clip objective to encompass a wider range of variants, enhancing our comprehension of its efficacy.
Accordingly, we present the first-ever global convergence guarantee for a PPO-Clip variant under both tabular and neural function approximation.
Notably, through convergence analysis, we offer two pivotal insights into the clipping mechanism:
{(i) Under PPO-Clip, the policy updates scale with advantage magnitudes, while the sign dictates whether to increase or decrease the action probabilities. Notably, given the representation power of neural networks, incorrect signs typically emerge when the advantage magnitudes are nearly zero. In such cases, these values insignificantly contribute to the objective, preserving the objective accuracy despite the incorrect signs. This perspective illuminates the robustness and empirical success of PPO-Clip.}
{(ii) Through our convergence analysis, we demonstrate that the clipping range merely affects the pre-constant of the convergence rate, not the asymptotic behavior.} All the code is available at https://github.com/NYCU-RL-Bandits-Lab/Neural-PPO-Clip

\noindent \textbf{Our Contributions.} We summarize the main contributions of this paper as follows:
\begin{itemize}[leftmargin=*]
    \item {To establish the global convergence of PPO-Clip, we leverage the connection between PPO-Clip and the hinge loss, leading to the formulation of generalized PPO-Clip objectives. Additionally, we harness the power of the entropic mirror descent (EMDA) \citep{beck2003mirror} for tabular PPO-Clip under direct policy parameterization, thereby demonstrating its asymptotic convergence.}
    \item {Inspired by the tabular analysis, we present a two-step policy improvement framework based on EMDA for Neural PPO-Clip. This framework enhances the manageability of the analysis by effectively separating policy search from policy parameterization. Accordingly, we establish the first global convergence result and explicitly characterize the $O(1/\sqrt{T})$ min-iterate convergence rate for the generalized PPO-Clip and hence provide an affirmative answer to one critical open question about PPO-Clip.}
    \item {We gain deeper insights into the PPO-Clip performance. Our theoretical findings yield two key insights into the clipping mechanism, as mentioned earlier. Furthermore, our analysis extends seamlessly to various Neural PPO-Clip variants with different classifiers, guided by the provided sufficient conditions.}
\end{itemize}

\section{Preliminaries}
\label{section:pre}

\textbf{Markov Decision Processes.}
Consider a discounted Markov Decision Process $(\mathcal{S}, \mathcal{A}, \mathcal{P}, R, \gamma, \mu)$, where $\mathcal{S}$ is the state space (possibly \textit{infinite}), $\mathcal{A}$ is a \textit{finite} action space, $\mathcal{P}: \mathcal{S} \times \mathcal{A} \times \mathcal{S} \rightarrow [0, 1]$ is the transition dynamic of the environment, $R: \mathcal{S} \times \mathcal{A} \rightarrow [0, R_{\max}]$ is the bounded reward function, $\gamma \in (0,1)$ is the discount factor, and $\mu$ is the initial state distribution.
Given a policy $\pi: \mathcal{S} \rightarrow \Delta(\mathcal{A})$, where $\Delta(\mathcal{A})$ is the unit simplex over $\mathcal{A}$, we define the state-action value function $Q^{\pi}(\cdot, \cdot) \coloneqq \mathbb{E}_{a_t \sim \pi(\cdot|s_t), s_{t+1} \sim \mathcal{P}(\cdot|s_t, a_t)}[\sum_{t=0}^{\infty} \gamma^t R(s_t, a_t) \rvert s_0 = s, a_0 = a]$.
Moreover, we define $V^{\pi}(s) \coloneqq \mathbb{E}_{a \sim \pi(\cdot\rvert s)}[Q^{\pi}(s, a)]$ and $A^{\pi}(s, a) \coloneqq Q^{\pi}(s, a) - V^{\pi}(s)$.
Also, we denote $\pi^*$ as an optimal policy that attains the maximum total expected reward and denote $\pi_0$ as the uniform policy. We introduce $\nu_{\pi}(s) = (1 - \gamma) \sum_{t=0}^{\infty} \gamma^t \mathbb{P}(s_t = s | s_0 \sim \mu, \pi)$ as the discounted state visitation distribution induced by $\pi$ and $\sigma_{\pi}(s, a) = \nu_{\pi}(s) \cdot \pi(a|s)$ as the state-action visitation distribution induced by $\pi$. In addition, we define the distribution $\nu^*$ and $\sigma^*$ as the discounted state visitation distribution and the state-action visitation distribution induced by the optimal policy $\pi^*$, respectively. Moreover, we define $\tilde{\sigma}_{\pi} = \nu_{\pi} \pi_0$ as the state-action distribution induced by interactions with the environment through $\pi$, sampling actions from the uniform policy $\pi_0$. 
We use $\mathbb{E}_{\nu_{\pi}}[\cdot]$ and $\mathbb{E}_{\sigma_{\pi}}[\cdot]$ as the shorthand notations of $\mathbb{E}_{s \sim \nu_{\pi}}[\cdot]$ and $\mathbb{E}_{(s,a) \sim \sigma_{\pi}}[\cdot]$, respectively.

For the convergence property, we define the total expected reward over the state distribution $\nu^*$ as
\begin{align}
\label{eq:cL}
    \mathcal{L}(\pi) := \mathbb{E}_{\nu^*}[V^{\pi}(s)].
\end{align}
Here, a maximizer of (\ref{eq:cL}) is equivalent to the original definition of the optimal policy $\pi^*$. We will prove the global convergence by analyzing the difference in $\mathcal{L}$ between our policy and the optimal policy and show that the total expected reward monotonically increases.

\noindent\textbf{Proximal Policy Optimization (PPO).}
{PPO is an empirically successful algorithm that achieves policy improvement by maximizing a surrogate lower bound of the original objective, either through the Kullback-Leibler penalty (termed PPO-KL) or the clipped probability ratio (termed PPO-Clip). PPO-KL and PPO-Clip represent the two main branches of PPO, both aiming to enforce policy constraints during updates for policy improvement. It is crucial to emphasize that PPO-Clip represents a conceptual approach, utilizing the clipping mechanism to achieve policy constraints, rather than being a precise algorithm itself. In this paper, our focus is PPO-Clip.}
Let $\rho_{s, a}(\theta)$ denote the probability ratio $\frac{\pi_{\theta}(a|s)}{\pi_{\theta_{t}}(a|s)}$. PPO-Clip avoids large policy updates by applying a simple heuristic that clips the probability ratio by the clipping range $\epsilon$ and thereby removes the incentive for moving $\rho_{s,a}(\theta)$ away from 1. Specifically, the PPO-Clip objective is
\begin{align}
\label{eq:clipobject}
    L^{\text{clip}}(\theta) = \mathbb{E}_{\sigma_{t}}[\min\{&\rho_{s, a}(\theta) A^{\pi_{\theta_{t}}}(s, a), \nonumber \\ &\hspace{-36pt}\text{clip}(\rho_{s, a}(\theta), 1-\epsilon, 1+\epsilon) A^{\pi_{\theta_{t}}}(s, a)\}].
\end{align}
\noindent\textbf{Neural Networks.}
We introduce the notations and assumptions relevant to neural networks. {It is important to highlight that our analysis of neural networks draws inspiration from \citep{liu2019neural}, and we adopt their notations to ensure compatibility.} Specifically, this paper centers around the analysis of two-layer neural networks. For simplicity, let us consider $(s, a) \in \mathbb{R}^d$ for all $(s, a) \in \mathcal{S} \times \mathcal{A}$. We represent the two-layer neural network as $\text{NN}(\alpha;m)$, where $\alpha$ denotes the network input weights and $m$ represents the network width. These neural networks act as the parameterization for both our policy $\pi_{\theta}$ and the $Q$ function. The parameterized function associated with $\text{NN}(\alpha;m)$ is depicted as follows:
\begin{align}
    u_{\alpha}(s, a) = \frac{1}{\sqrt{m}} \sum_{i=1}^{m} b_i \cdot \sigma([\alpha]_i^{\top} (s, a)),
\end{align}
where $\alpha = ([\alpha]_1^{\top}, \dots, [\alpha]_m^{\top})^{\top} \in \mathbb{R}^{md}$ is the input weights, with $[\alpha]_i \in \mathbb{R}^d$, $b_i \in \{-1, 1\}$ are the weights of the output, and $\sigma(\cdot)$ refers to the Rectified Linear Unit (ReLU) activation function. The initializations for the input weights $\alpha_0$ and $b_i$ are provided as follows:
\begin{align}
\label{init}
    b_i \sim \text{Unif}(\{1, -1\}), [\alpha_0]_i \sim \mathcal{N}(0, I_d/d),
\end{align}
where both $b_i$ and $[\alpha_0]_i$ are i.i.d. for each $i \in [m]$ and $I_d$ is the $d \times d$ identity matrix. The values of $b_i$ remain fixed following initialization, with the training exclusively focused on adjusting the weights $\alpha$. To uphold the local linearization characteristics, we employ a projection mechanism that confines the training weights $\alpha$ within an $\ell_2$-ball centered at $\alpha_0$, which is represented as $B_{f} = \{ \alpha: \lVert\alpha - \alpha_0 \rVert_2 \le R_f\}$, where $f$ is the canonical name of the networks (It will be $f$ for the policy network and $Q$ for the Q function network in the following section). 

Our examination of neural networks is grounded in the subsequent assumptions, which are widely adopted regularity conditions for neural networks in the reinforcement learning literature \citep{liu2019neural, antos2007fitted, farahmand2016regularized}:

\begin{assumption}[Q Function Class]
\label{assump:func}
    We assume that the our neural network class possesses sufficient representational capacity to model the $Q$ function of any given policy $\pi$. Specifically, for any $R>0$, define a function class
    \begin{align}
    \label{function_class}
        \mathcal{F}_{R,m} = \Big\{\frac{1}{\sqrt{m}} \sum_{i=1}^{m} b_i \cdot \mathds{1}\{[\alpha_0]_i^{\top} (s, a) > 0\} \cdot [\alpha]_i^{\top} (s, a) \Big\},
    \end{align}
    for all $\alpha$ satisfying $\lVert\alpha - \alpha_0\rVert_2 \le R$, where $b_i$ and $\alpha_0$ are initialized as (\ref{init}).
    We assume that $Q^{\pi}(s, a) \in \mathcal{F}_{R_Q, m_Q}$ for any policy $\pi$, where $R_Q$ and $m_Q$ are the projection radius and width of the neural network for $Q$ function.
\end{assumption}

Given that $\mathcal{T}^{\pi} Q^{\pi}$ remains a $Q$ function, Assumption \ref{assump:func} affords us the property of {completeness} within our function class under the Bellman operator $\mathcal{T}^{\pi}$.

\noindent \textbf{Notations:} We use $\langle a, b \rangle$ and $a \circ b$ to denote the inner product and the Hadamard product, respectively.

\section{{Generalized PPO-Clip Objectives}}
\label{section:HPO}

\textbf{Connecting PPO-Clip and Hinge Loss.} According to \citep{hu2020rethinking, pi2020low}, the original PPO-Clip objective could be connected with the hinge loss. 
Specifically, the gradient of the clipped objective is indeed the negative of the gradient of hinge loss objective, i.e.,
\begin{align}
    \nonumber&\frac{\partial}{\partial\theta}   \min\{\rho_{s,a}(\theta){A}^{\pi}(s,a),\text{clip}(\rho_{s,a}(\theta),1-\epsilon,1+\epsilon){A}^{\pi}(s,a)\} \\
    &=-\frac{\partial}{\partial\theta}\ \lvert {A}^{\pi}(s,a)\rvert\ \ell(\sgn({A}^{\pi}(s,a)), \rho_{s,a}(\theta)-1, \epsilon),
\end{align}
where $\ell(y_{i}, f_{\theta}(x_i), \epsilon)$ is the hinge loss defined as $\max\{0, \epsilon-y_{i} \cdot f_{\theta}(x_i)\}$, $\epsilon$ is the margin, $y_{i}\in\{-1,1\}$ the label corresponding to the data $x_{i}$, and $f_\theta(x_{i})$ serves as the binary classifier. For completeness, please see Appendix \ref{app:compare PPO-Clip and HPO} for a detailed comparison of the two objectives.
From the above, maximizing the objective in (\ref{eq:clipobject}) can be rewritten as minimizing the following loss: 
\begin{align}
    \label{eq:hingeobject}
    L(\theta) = \sum_{s\in\mathcal{S}}d_{\mu}^{\pi}(s)\sum_{a\in\mathcal{A}}&\Big(\pi(a|s)\lvert {A}^{\pi}(s,a)\rvert\cdot \nonumber \\
    &\hspace{-36pt}\ell(\sgn({A}^{\pi}(s,a)), \rho_{s,a}(\theta)-1, \epsilon)\Big). 
\end{align}
In practice, we draw a batch of state-action pairs and use the sample average to approximately minimize the loss in (\ref{eq:hingeobject}).
\noindent \textbf{Generalized PPO-Clip Objectives.} {Based on the above reinterpretation of PPO-Clip, we provide a general form of the PPO-Clip loss function from a hinge loss perspective as follows,}
\begin{equation}
    {L_{\text{Hinge}}(\theta)}=\frac{1}{\lvert\mathcal{D}\rvert}\sum_{(s,a)\in\mathcal{D}}\text{weight}\times\ell(\text{label},\text{classifier},\text{margin}).\label{eq:HPO loss}
\end{equation}
Different combinations of classifiers, margins, and weights lead to different loss functions, thereby representing diverse algorithms. PPO-Clip is a special case of (\ref{eq:HPO loss}) with a specific classifier $\rho_{s,a}(\theta)-1$. Another variant, termed PPO-Clip-sub in this paper, can be obtained by employing a subtraction classifier, i.e., $\pi_{\theta}(a|s) - \pi_{\theta_t}(a|s)$.
There are several other variants under this generalized objective by employing distinct classifiers, e.g., $\log(\pi_{\theta}(a|s)) - \log(\pi_{\theta_t}(a|s))$ and $\sqrt{\rho_{s,a}(\theta)}-1$.
We demonstrate the empirical evaluation of these variants in Section \ref{section:Discussions}.
Given the above examples, the proposed objective provides to generalizing PPO-Clip via various classifiers, thereby expanding the objective choices within the context of PPO-Clip. This generalization also connects the PPO-Clip with the classifier selection paradigm. Additionally, this generalized objective provide an intution to understand more about the clipping mechanism. Please refer to Section \ref{subsection: understanding}.


\section{Tabular PPO-Clip}

\subsection{Direct Policy Parameterization}
In this section, we study the global convergence of PPO-Clip with direct parameterization, i.e., 
policies are parameterized by $\pi(a|s)=\theta_{s,a}$, where $ \theta_s\in\Delta(\mathcal{A})$ denotes the vector $\theta_{s,\cdot}$ and $\theta\in\Delta(\mathcal{A})^{\lvert\mathcal{S}\rvert}$.
We use $V^{(t)}(s)$ and $A^{(t)}(s,a)$ as the shorthands for $V^{\pi^{(t)}}(s)$ and $A^{\pi^{(t)}}(s,a)$, respectively.

For the sake of clarity, we focus our discussion on the original PPO-Clip rather than delving into the broader scope of the generalized objective (\ref{eq:HPO loss}). 
Furthermore, we also provide additional analysis for other PPO-Clip variants with different classifiers in Appendix \ref{app:secondthm}.
Note that by choosing the weight as $\lvert {A}^{(t)}(s,a)\rvert$, the classifier as $\rho_{s,a}^{(t)}(\theta)-1$, and the margin as $\epsilon$ in (\ref{eq:HPO loss}) at the $t$-th iteration, the generalized objective would recover the form of the objective of PPO-Clip, which denoted as $\hat{L}^{(t)}(\theta)$.
The detailed algorithm is shown in Appendix \ref{app:pseudo_code} as Algorithm \ref{algo:HPO-AM}.

In each iteration, PPO-Clip updates the policy by minimizing the loss $\hat{L}^{(t)}(\theta)$ via the EMDA \citep{beck2003mirror}. 
While there are alternative ways to minimize the loss $\hat{L}^{(t)}(\theta)$ over $\Delta(\mathcal{A})^{\lvert \mathcal{S}\rvert}$ (e.g., the projected subgradient method), we leverage EMDA for the following two reasons:
(i) PPO-Clip achieves policy improvement by increasing or decreasing the probability of those state-action pairs in $\cD^{(t)}$ based on the sign of ${{A}^{(t)}(s,a)}$ as well as properly reallocating the probabilities of those state-action pairs \textit{not} contained in the batch (to ensure the probability sum is one). Using EMDA enforces a proper reallocation in PPO-Clip, as shown later in the proof of Theorem \ref{thm:mini-batch} in Appendix \ref{app:mini-batch thm}; (ii) The exponentiated gradient scheme of EMDA guarantees $\pi^{(t)}$ remains strictly positive for all state-action pairs in each iteration $t$, ensuring the well-defined nature of the probability ratio $\rho_{s,a}(\theta)$ used in PPO-Clip.
In this section, we consider the stylized setting with tabular policy and true advantage mainly for motivating the PPO-Clip method and its analysis. 

\subsection{Global Convergence of PPO-Clip with Direct Parameterization}
We first make the following assumptions. Note that we only consider these assumptions in the tabular case.

\begin{assumption}[Infinite Visitation to Each State-Action Pair]
\label{assumption:infinite visit}
    Each state-action pair $(s,a)$ appears infinitely often in $\{\cD^{(\tau)}\}$, i.e., $\lim_{t\rightarrow \infty} \sum_{\tau=0}^{t}\mathds{1}\{(s,a)\in \cD^{(\tau)}\}=\infty$, with probability one.
\end{assumption}

\begin{assumption}
\label{assumption:distinct states}
    In each iteration $t$, the state-action pairs in $\cD^{(t)}$ have distinct states.
\end{assumption}

Assumption \ref{assumption:infinite visit} resembles the standard infinite-exploration condition commonly used in the temporal-difference methods, such as Sarsa \citep{singh2000convergence}. Assumption \ref{assumption:distinct states} is rather mild: 
(i) This can be met by post-processing the mini-batch of state-action pairs via an additional sub-sampling step; (ii) In most RL problems with discrete actions, the state space is typically much larger than the action space. 

\begin{theorem}[Global Convergence of PPO-Clip]
\label{thm:mini-batch}
Under PPO-Clip, we have $V^{(t)}(s)\rightarrow V^{\pi^*}(s)\text{ as }t\rightarrow\infty,\ \forall s\in\mathcal{S}$, with probability one.
\end{theorem}

The proof of Theorem \ref{thm:mini-batch} is provided in Appendix \ref{app:mini-batch thm}. 
We highlight the main ideas behind the proof of Theorem \ref{thm:mini-batch}:
(i) \textit{State-wise policy improvement:} Through the lens of generalized objective, we show that PPO-Clip enjoys state-wise policy improvement in every iteration with the help of the EMDA subroutine. 
This property greatly facilitates the rest of the convergence analysis.
(ii) \textit{Quantifying the probabilities of those actions with positive or negative advantages in the limit}: By (i), we know the limits of the value functions and the advantage function all exist. Then, we proceed to show that the actions with positive advantages in the limit cannot exist by establishing a contradiction.
The above also manifests how reinterpreting PPO-Clip helps with establishing the convergence guarantee.

\section{Neural PPO-Clip}
\label{section:Neural}
In this section, we begin by illustrating the process of decoupling policy search and policy parameterization, drawing inspiration from the tabular case. Subsequently, we provide a comprehensive overview of the neural PPO-Clip algorithm. We proceed to delineate the intricacies posed by our analysis and present our results on the min-iterate convergence rate, both for the generalized PPO-Clip. In particular, the convergence rate of PPO-Clip can be view as a special case of our general results. Lastly, we offer a profound insight into the understanding of the clipping mechanism.

\subsection{{EMDA-Based Policy Search}}
\label{section:NeuralHPO:NAPS}
Drawing inspiration from the tabular case, we proceed to present our two-step policy improvement scheme based on EMDA, and we call it EMDA-based Policy Search.
Specifically, this scheme consists of two subroutines:
\begin{itemize}[leftmargin=*]
    \item {\textbf{Direct policy search}:
In this step, we directly search for an improved policy in the policy space by EMDA.
More specifically, in each iteration $t$, we do a policy search by applying EMDA with direct parameterization to minimize the generalized PPO-Clip objective in (\ref{eq:HPO loss}) for finitely many iterations $K$ and thereby obtain an improved policy $\widehat{\pi}_{t+1}$ as the target policy. The pseudo code of EMDA is provided in Algorithm \ref{algo:2}. Notably, under EMDA, we can obtain an explicit expression of the target policy $\widehat{\pi}_{t+1}$. } 
\item {\textbf{Neural approximation for the target policy}: Given the target policy $\widehat{\pi}_{t+1}$ obtained by EMDA, we then approximate it in the parameter space by utilizing the representation power of neural networks via a regression-based policy update scheme (e.g., by using the mean-squared error loss).
The detailed neural parameterization will be described in the next subsection.}
\end{itemize}
While the decision to employ EMDA is inspired by the tabular case, there are two primary motivations and benefits for integrating EMDA with direct parameterization:
\begin{itemize}[leftmargin=*]
    \item {\textbf{Decoupling improvement and approximation:} One major goal of this paper is to provide rigorous theoretical guarantees for PPO-Clip under neural function approximation. To make the analysis tractable and general, we would like to decouple policy improvement and function approximation of the policy. To achieve this, we adopt the EMDA-based two-step approach outlined previously.}
    \item {\textbf{EMDA-induced closed-form expression of the target policy:} For policy optimization analysis, the goal is often to derive a closed-form optimal solution for the policy improvement objective as the ideal target policy. However, such a closed-form optimal solution of an \textit{arbitrary} objective function does not always exist. A case in point is the loss function of PPO-Clip. From this view, EMDA, which enjoys closed-form updates, substantially facilitates the convergence analysis, as can be observed in Proposition \ref{pp:PI} presented in the subsequent subsection \ref{subsection: Neural PPO-Clip}.}
\end{itemize}

\subsection{{Neural PPO-Clip}}
\label{subsection: Neural PPO-Clip}

\textbf{Parameterization Setting.} At each iteration $t$, we parameterize our policy as an energy-based policy $\pi_{\theta_t}(a|s) \propto \exp\{\tau_t^{-1} f_{\theta_t}(s, a)\}$, where $\tau_t$ denotes the temperature parameter and $f_{\theta_t}(s, a) = \text{NN}(\theta_t;m_f)$ corresponds to the energy functions. The width of the neural network $f_\theta$ is denoted as $m_f$, as defined in Section \ref{section:pre}. 
Likewise, we parameterize our state-action value function as $Q_{\omega}(s, a) = \text{NN}(\omega;m_Q)$, with width $m_Q$ of the neural network $Q_\omega$. Concurrently, we define $V_{\omega}(s)$ as the value function derived from the Bellman Expectation Equation. Also, we define $A_{\omega}(s, a) := Q_{\omega}(s, a) - V_{\omega}(s)$ to be the advantage function.

\noindent \textbf{Policy Improvement.}
According to the {EMDA-based Policy Search} framework presented above, we first give the closed-form of the obtained target policy of Neural PPO-Clip as follows. The detailed proof is in Appendix \ref{app:B}.

\begin{prop}[EMDA Target Policy]
\label{pp:PI}
    For the target policy obtained by the EMDA subroutine at the $t$-th iteration, we have
    \begin{align}
        \log \widehat{\pi}_{t+1}(a|s) \propto C_t(s, a) A_{\omega_t}(s, a) + \tau_{t}^{-1} f_{\theta_t}(s, a), 
    \end{align}
    where $C_t(s, a) A_{\omega_t}(s, a) = -\sum_{k=0}^{K-1} \eta g_{s,a}^{(k)}$ as given in Algorithm \ref{algo:2}.
\end{prop}

Recall that the target policy $\widehat{\pi}$ is the direct parameterization in the policy space, but our policy $\pi_{\theta}$ is an energy-based (softmax) policy that is proportional to the exponentiated energy function. 
This explains why we consider the $\log \widehat{\pi}_{t+1}(a|s)$ in Proposition \ref{pp:PI}.
{Another benefit of using EMDA is that it closely matches the energy-based policies considered in Neural PPO-Clip due to the inherent exponentiated gradient update.}

Then, we discuss the details of the neural function approximation of our policy. After obtaining the target policy by Proposition \ref{pp:PI}, we solve the Mean Squared Error (MSE) subproblem with respect to $\theta$ to approximate the target policy as follows:
\begin{align}
    \mathbb{E}_{\tilde{\sigma}_t}[(f_{\theta}(s, a) - \tau_{t+1} (C_t(s, a) A_{\omega_t}(s, a) + \tau_{t}^{-1} f_{\theta_t}(s, a)))^2].
\end{align}
Notice that we consider the state-action distribution $\tilde{\sigma}_t$ sampling the action through a uniform policy $\pi_0$. In this manner, we use more exploratory data to improve our current policy. In particular, we use the SGD to tackle the above subproblem, and the pseudo code is provided in Appendix \ref{app:pseudo_code}.

        
\noindent \textbf{Policy Evaluation.} To evaluate $Q$, we use a neural network to approximate the true state-action value function $Q^{\pi_{\theta_t}}$ by solving the Mean Square Bellman Error (MSBE) subproblem. The MSBE subproblem is to minimize the following objective with respect to $\omega$ at each iteration $t$:
\begin{align}
    \mathbb{E}_{\sigma_t}[(Q_{\omega}(s, a) - [\mathcal{T}^{\pi_{\theta_t}}  Q_{\omega}](s, a))^2],
\end{align}
where $\mathcal{T}^{\pi_{\theta_t}} $ is the Bellman operator of policy $\pi_{\theta_t}$ such that
\begin{align}
    &[\mathcal{T}^{\pi_{\theta_t}} Q_{\omega}](s, a)\nonumber\\&= \mathbb{E}[r(s, a) + \gamma Q_{\omega}(s',a') \mid s' \sim \mathcal{P}(\cdot|s, a), a' \sim \pi_{\theta_t}(\cdot|s')].
\end{align}
The pseudo code of neural TD update for state-action value function $Q_{\omega}$ is in Appendix \ref{app:pseudo_code}.
It is worth mentioning that this variant of Neural PPO-Clip is not a fully on-policy algorithm. Although we interact with the environment by our current policy, we sample the actions by the uniform policy $\pi_0$ for policy improvement. {We provide the pseudo code of Neural PPO-Clip as the following Algorithm \ref{algo:1 incomplete} 
\begin{algorithm}[!htbp]
\caption{Neural PPO-Clip}
\label{algo:1 incomplete}
\textbf{Input}: $L_{\text{Hinge}}(\theta)$, $T$, $\epsilon$, EMDA step size $\eta$, number of EMDA iterations $K$, number of SGD, TD update iterations $T_{\text{upd}}$ \\
\textbf{Initialization}: uniform policy $\pi_{\theta_0}$
    \begin{algorithmic}[1]
        \FOR{$t=1,\cdots,T-1$}
            \STATE Set temperature parameter $\tau_{t+1}$\;
            \STATE Sample the tuple $\{s_i, a_i, a_i^0, s_i',a_i'\}_{i=1}^{T_{\text{upd}}}$\;
            \STATE Run TD as Algorithm \ref{algo:3}: $Q_{\omega_t} = \text{NN}(\omega_t;m_{Q})$\;
            \STATE Calculate $V_{\omega_t}$ and the advantage $A_{\omega_t} = Q_{\omega_t} - V_{\omega_t}$\;
            \STATE Run EMDA as Algorithm \ref{algo:2} with $L_{\text{Hinge}}(\theta)$\;
            \STATE Run SGD as Algorithm \ref{algo:4}: $f_{\theta_{t+1}} = \text{NN}(\theta_{t+1};m_f)$\;
            \STATE Update the policy $\pi_{\theta_{t+1}} \propto \exp \{ \tau_{t+1}^{-1} f_{\theta_{t+1}}\}$\;
        \ENDFOR
    \end{algorithmic}
\end{algorithm}
\begin{algorithm}[!htbp]
\caption{EMDA}
\label{algo:2}
\textbf{Input}: $L_{\text{Hinge}}(\theta)$, EMDA step size $\eta$, number of EMDA iterations $K$, initial policy $\pi_{\theta_{t}}$, sample batch $\{s_i\}_{i=1}^{T_{\text{upd}}}$ \\
\textbf{Initialization}: $\tilde{\theta}^{(0)} = \pi_{\theta_{t}}$, $C_t(s, a) = 0$, for all $s, a$ \\
\textbf{Output}: $\widehat{\pi}_{t+1}$ and $C_t$
\begin{algorithmic}[1]
    \FOR{$k=0,\cdots,K-1$}
        \FOR{\text{each state} $s$ \text{in the batch}}
        \STATE Find $g_{s,a}^{(k)} = \left.\frac{\partial L_{\text{Hinge}}(\theta)}{\partial \theta_{s, a}}\right|_{\theta = \tilde{\theta}^{(k)}}$, for each $a$\;
        \STATE Let $w_s = (e^{-\eta g_{s,1}}, \dots, e^{-\eta g_{s,|\mathcal{A}|}})$\;
        \STATE $\tilde{\theta}^{(k+1)} = \frac{1}{\langle w_s, \tilde{\theta}^{(k)} \rangle} (w_s \circ \tilde{\theta}^{(k)})$\;
        \STATE $C_t(s, a) \leftarrow C_t(s, a) - \eta g_{s,a}^{(k)} / A_{\omega_t}(s, a)$, for 
        
        \ \ \ \qquad each $a$ with $A_{\omega_t}(s, a) \neq 0$\;
        \ENDFOR
    \ENDFOR
    \STATE $\widehat{\pi}_{t+1} = \tilde{\theta}^{(K)}$\;
\end{algorithmic}
\end{algorithm}
(please refer to Algorithm \ref{algo:1} in Appendix \ref{app:pseudo_code} for the complete version) and the pseudo code of EMDA as Algorithm \ref{algo:2}.
The pseudo code of Algorithms \ref{algo:3}-\ref{algo:4} used by Algorithm \ref{algo:1 incomplete} is in Appendix \ref{app:pseudo_code}.}

Regarding our analyses, we need assumptions about distribution density. Assumption \ref{assump:reg} states that the distribution $\sigma_{\pi}$ is sufficiently regular, which is required to analyze the neural network error.
Additionally, the common theory works \citep{antos2007fitted, farahmand2010error, farahmand2016regularized, chen2019information, liu2019neural} have the concentrability assumption, we also have this common regularity condition.

\begin{assumption}[Regularity of Stationary Distribution]
\label{assump:reg}
    Given any state-action visitation distribution $\sigma_{\pi}$, there exists a universal upper bounding constant $c > 0$ for any weight vector $z \in \mathbb{R}^d$ and $\zeta > 0$, such that $\mathbb{E}_{\sigma_{\pi}}[\mathds{1}\{|z^{\top}(s, a)| \le \zeta\} | z] \le c \cdot \zeta/ \lVert z \rVert_2$ holds almost surely.
\end{assumption}

\begin{assumption}[Concentrability Coefficient and Ratio]
\label{assump:con}
    Define the density ratio between the policy-induced distributions and the policies,
    \begin{align}
        \phi^*_t = \mathbb{E}_{\tilde{\sigma}_t}[\left|\frac{d \pi^*}{d\pi_0} - \frac{d \pi_{\theta_t}}{d \pi_0}\right|^2]^{\frac{1}{2}}, \psi^*_t = \mathbb{E}_{\sigma_t}[\left|\frac{d \sigma^*}{d \sigma_t} - \frac{d \nu^*}{d \nu_t}\right|^2]^{\frac{1}{2}},
    \end{align}
    where the above fractions are the Radon–Nikodym Derivatives. {We assume that there exist $0<\phi^*,\psi^*<\infty$ such that $\phi^*_t < \phi^*$ and $\psi^*_t < \psi^*$, for all $t$. Also, let $C_{\infty}<\infty$ be the concentrability coefficient. We assume that the density ratio between the optimal state distribution and any state distribution, i.e. $\lVert \nu^* / \nu\rVert_{\infty}< C_{\infty}$ for any $\nu$. }
\end{assumption}


            

\subsection{Convergence Guarantee of Neural PPO-Clip}
\label{section:analysis}

In this subsection, we present the convergence analysis of Neural PPO-Clip. Inspired by the analysis of \citep{liu2019neural}, we analyze the convergence behavior of Neural PPO-Clip based on the neural networks analysis technique. Nevertheless, the analysis presents several unique technical challenges in establishing its convergence: (i) \textit{Tight coupling between function approximation error and the clipping behavior}: The clipping mechanism can be viewed as an indicator function. The function approximation for advantage would significantly influence the value of the indicator function in a highly complex manner. As a result, handling the error between the neural approximated advantage and the true advantage serves as one major challenge in the analysis
{(please refer to the proof of Lemma \ref{lm:ep} in Appendix \ref{app:main_thm} for more details)}; (ii) \textit{Lack of a closed-form expression of policy update}: Due to the clipping function in the hinge loss objective and the iterative updates in the EMDA subroutine, the new policy does not have a simple closed-form expression. This is one salient difference between the analysis of Neural PPO-Clip and other neural algorithms (cf. \citep{liu2019neural});
(iii) \textit{Neural networks analysis on advantage function}: Another technicality is that the advantage function requires the neural network projection and linearization properties to characterize the approximation error. However, since we use the neural network to approximate the state-action value function instead of the advantage function, it requires additional effort to establish the error bound of the advantage function {(please refer to the proof of Lemma \ref{lm:PI_error})}.

Given that we need to analyze the error between our approximation and the true function, we further define the target policy under the true advantage function $A^{\pi_{\theta_t}}$ as {$\pi_{t+1}(a|s) := \bar{C}_t(s, a)A^{\pi_{\theta_t}}(s, a) + \tau_t^{-1} f_{\theta_t}(s, a)$, where $\bar{C}_t(s, a)$ is the $C_t(s,a)$ obtained under $A^{\pi_{\theta_t}}$}.
Moreover, all the expectations about $A_{\omega}$ throughout the analysis are with respect to the randomness of the neural network initialization.
Below we state the min-iterate convergence rate and the sufficient condition of Neural PPO-Clip, which is also the main theorem of our paper. Throughout this section, we solely suppose Assumptions \ref{assump:func}, \ref{assump:reg}, and \ref{assump:con} hold.

The central result of this paper is Theorem \ref{thm:main}. In this theorem, $L_C(T)$ and $U_C(T)$ are functions influenced by $T$ and determined by $\bar{C}_t$, a classifier-specific attribute. For detailed supporting lemmas and proofs, see Appendix \ref{app:main_thm}.
\begin{theorem}[{General} Convergence Rate of Neural PPO-Clip]
\label{thm:main}
    Consider the Neural PPO-Clip with the classifier satisfying the following conditions for all $t$,
    \begin{alignat}{2}
        \label{suff:1}
        &\text{(i) } L_{C}(T) \cdot |A^{\pi}(s, a)| \le \bar{C}_t(s, a) &&\cdot |A^{\pi}(s, a)| \nonumber \\
        & &&\le U_{C}(T) \cdot |A^{\pi}(s, a)|, \\
        \label{suff:2}
        &\text{(ii) } L_C(T) = \omega(T^{-1}), U_C(T) = &&O(T^{-1/2}).
    \end{alignat}
    Then, the policy sequence $\{\pi_{\theta_t}\}_{t=0}^{T}$ obtained by Neural PPO-Clip satisfies
    \begin{align}
    \label{thm:main:eq}
        &\min_{0\le t \le T} \{\mathcal{L}(\pi^*) - \mathcal{L}(\pi_{\theta_t})\} \nonumber \\ &\le \frac{\log |\mathcal{A}| + \sum_{t=0}^{T-1} (\varepsilon_t + \varepsilon_t') + T U_{C}^2 (2 \psi^* + M)}{T L_{C} (1 - \gamma)},
    \end{align}
    where $\varepsilon_t = C_{\infty} \tau_{t+1}^{-1} \phi^* \epsilon_{t+1}^{1/2} + Y^{1/2} \psi^* \epsilon_t'^{1/2}$, $\varepsilon_t' = |\mathcal{A}| \cdot C_{\infty} \tau_{t+1}^{-2} \epsilon_{t+1}$, $M = 4\mathbb{E}_{\nu_t}[\max_{a} (Q_{\omega_0}(s, a))^2] + 4R_f^2$, and $Y = 2M + 2(R_{\max} / (1 - \gamma))^2$.
    
\end{theorem}
{To demonstrate that our convergence analysis is general for Neural PPO-Clip with various classifiers, we choose to state Theorem \ref{thm:main} in a general form utilizing the condition (\ref{suff:1}) and (\ref{suff:2}).
Indeed, we show that (\ref{suff:1}) and (\ref{suff:2}) can be naturally satisfied by using the standard PPO-Clip classifier described in (\ref{eq:hingeobject}) in the following Corollary \ref{cor:PPO-Clip}. Importantly, these conditions are not technical assumptions for our theorem. Notably, we also establish that PPO-Clip-sub (a variant of generalized PPO-Clip utilizing a distinct classifier) aligns with the result presented in Theorem \ref{thm:main}. For a comprehensive statement and analysis, please refer to Appendix \ref{app:add:cor}.

\begin{corollary}[Global Convergence of {Neural PPO-Clip},  Informal]
\label{cor:PPO-Clip}
    Consider Neural PPO-Clip with the standard PPO-Clip classifier $\rho_{s, a}(\theta) - 1$ and the objective function $L^{(t)}(\theta)$ in each iteration $t$ as 
    \begin{align}
         \mathbb{E}_{\nu_t}[\langle \pi_{\theta_t}(\cdot|s), |A^{\pi_{\theta_t}}(s, \cdot)| \circ \ell (\sgn(A^{\pi_{\theta_t}}(s, \cdot)), \rho_{s, \cdot}(\theta) - 1, \epsilon) \rangle].
    \end{align}
    (i) If we specify the EMDA step size $\eta = T^{-\alpha}$ where $\alpha \in [1/2, 1)$ and the temperature parameter $\tau_t = T^{\alpha} / (Kt)$. Recall that $K$ is the number of EMDA iterations.
    Let the neural networks' widths be $m_f, m_Q$, and the SGD and TD updates $T_{\text{upd}}$ be configured as in Appendix \ref{app:add:cor}, we have
    \begin{align}
        \min_{0\le t \le T} &\{\mathcal{L}(\pi^*) - \mathcal{L}(\pi_{\theta_t})\} \nonumber \\ & \hspace{0pt}\le \frac{\log |\mathcal{A}| + K^2 (2 \psi^* + M) + O(1)}{T^{\alpha} (1 - \gamma)}. \label{cor:PPO-Clip:eq}
    \end{align}
    {Hence, Neural PPO-Clip has $O(T^{-\alpha})$ convergence rate.}
    (ii) Furthermore, let the $\alpha = 1/2$, we obtain the fastest convergence rate, which is $O(1 / \sqrt{T})$.
\end{corollary}

Notably, the min-iterate convergence rates presented in (\ref{thm:main:eq}) and (\ref{cor:PPO-Clip:eq}) are commonly observed in the realms of nonconvex optimization and neural network theory \citep{lacoste2016convergence, ghadimi2016accelerated, liu2019neural}, and they do not constitute stringent results. Furthermore, it is worth pointing out that in (\ref{thm:main:eq}), the terms $\varepsilon_t$ and $\varepsilon_t'$ correspond to the errors introduced by policy improvement and policy evaluation, respectively. These errors can be controlled by adjusting neural network widths and the number of TD and SGD iterations $T_{\text{upd}}$, and they can be made arbitrarily small. Further details can be found in Appendix \ref{app:main_thm}.

Consequently, the convergence rate obtained by our analysis is determined by $U_C(T)^2 / L_C(T)$. After a brief calculation, it becomes evident that under conditions (\ref{suff:1}) and (\ref{suff:2}), the most optimal convergence rate achievable through (\ref{thm:main:eq}) is $O(1 / \sqrt{T})$. This scenario arises when $L_C(T) = U_C(T) = O(T^{-1/2})$. This insight underscores that within our analysis, the original PPO-Clip stands as the algorithm that achieves the most favorable bound.

\begin{figure*}[!ht]
\centering
    \subfigure[Space Invaders]{\includegraphics[width=0.268\linewidth]{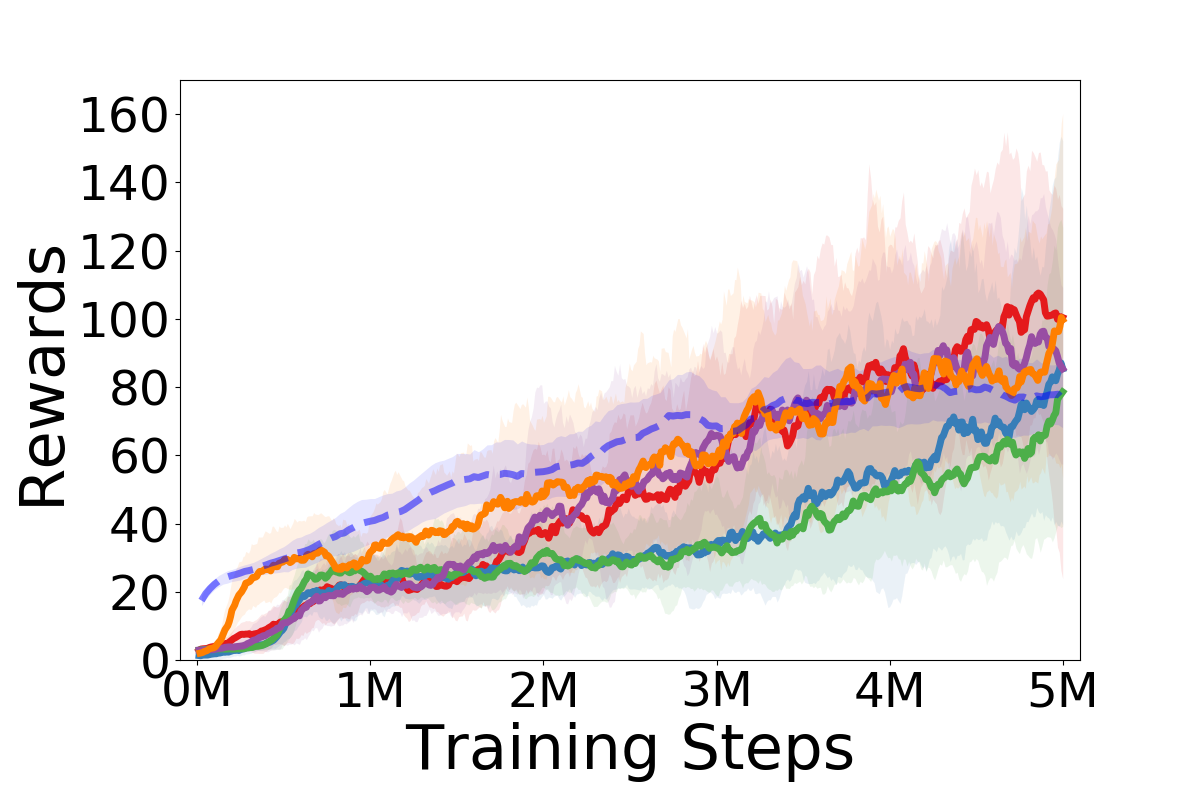}}
    \hfill
\hspace{-7mm}
    \subfigure[Breakout]{\includegraphics[width=0.268\linewidth]{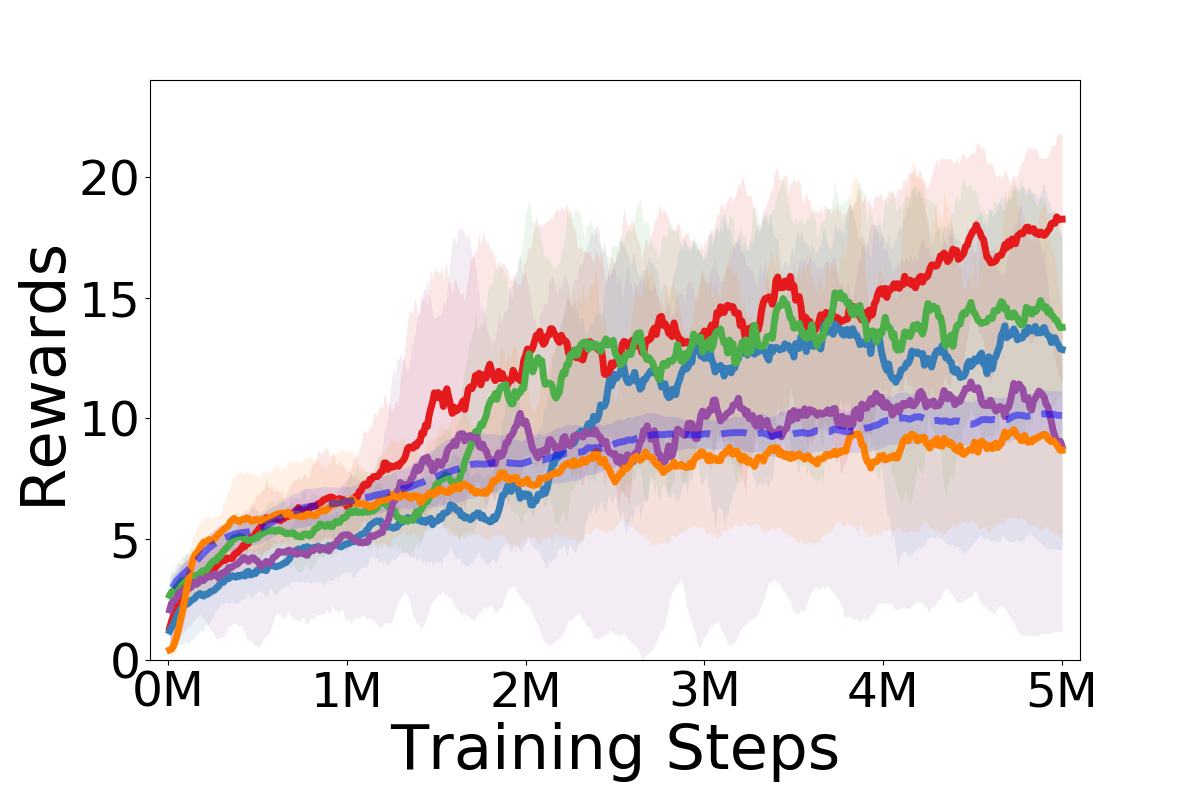}}
\hspace{-7mm}
 \hfill
    \subfigure[LunarLander]{\includegraphics[width=0.268\linewidth]{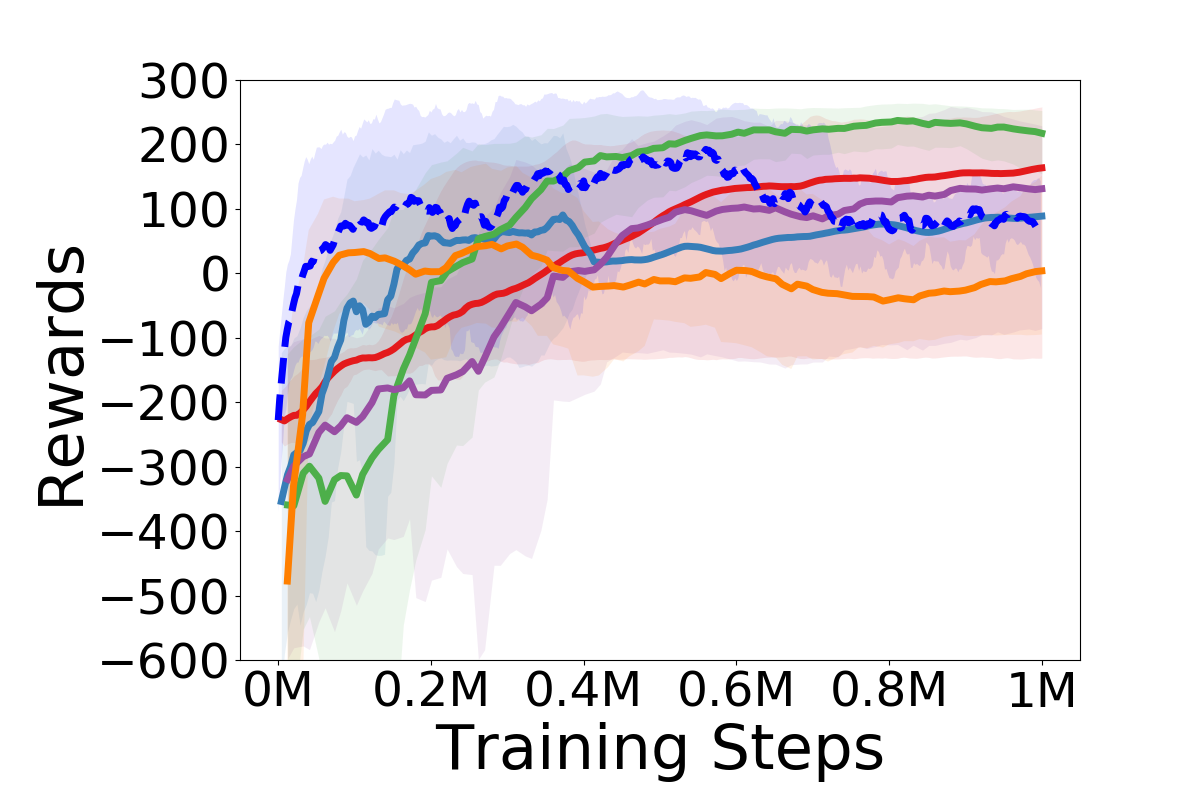}}
\hspace{-5.8mm}
 \hfill
    \subfigure[CartPole]{\includegraphics[width=0.268\linewidth]{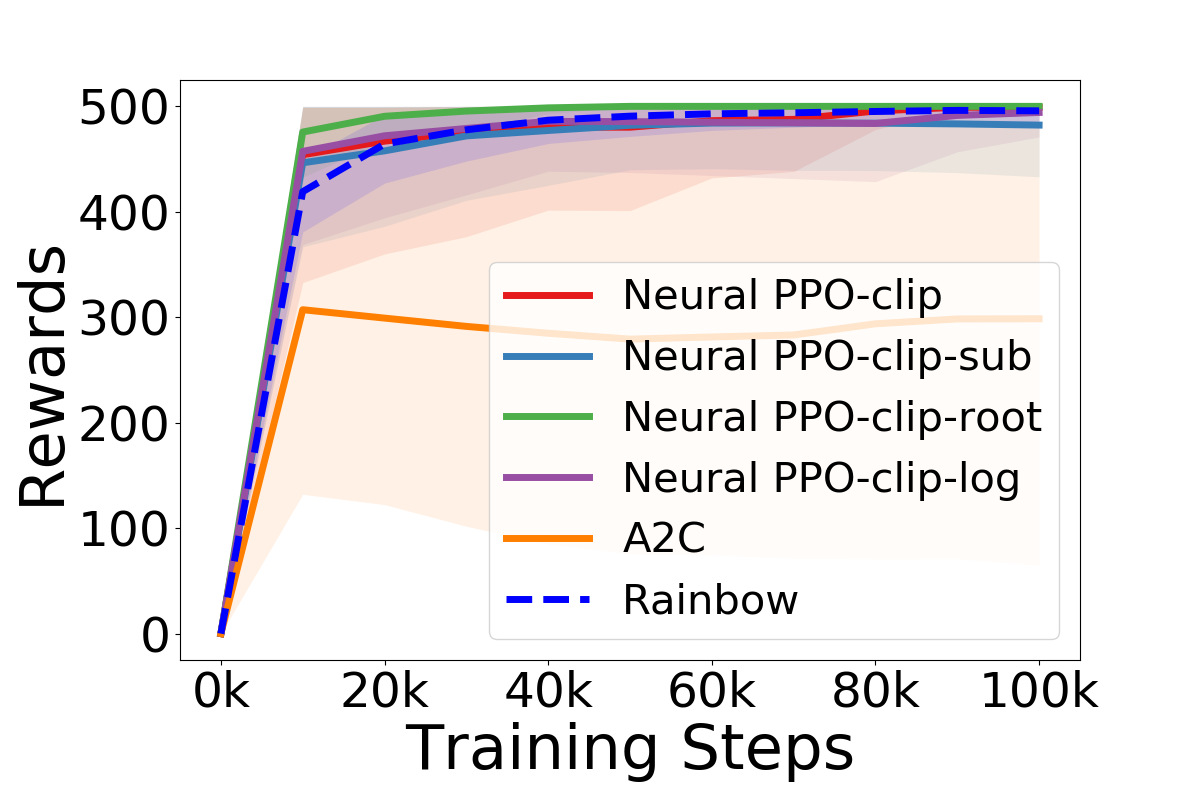}}
\caption{Evaluation of PPO-Clip with different classifiers and popular benchmark methods in MinAtar and OpenAI Gym.}
\label{fig:experiments}
\vspace{-2mm}
\end{figure*}

\subsection{Understanding the Clipping Mechanism}
\label{subsection: understanding}

In this subsection, we delve into the more profound understanding of the clipping mechanism.

\noindent \textbf{Rationale Behind the PPO-Clip Convergence.} As outlined in Section \ref{section:HPO}, the clipping mechanism establishes a connection to the hinge loss, consequently shaping the objective as (\ref{eq:HPO loss}). Notably, in the context of the original PPO-Clip, we specify the objective as follows:
\begin{equation}
    \frac{1}{\lvert\mathcal{D}\rvert}\sum_{(s,a)\in\mathcal{D}}\lvert {A}^{\pi}(s,a)\rvert\ \ell(\sgn({A}^{\pi}(s,a)), \rho_{s,a}(\theta)-1, \epsilon). \label{eq: understand PPO-Clip}
\end{equation}

We delve more deeply into this objective (\ref{eq: understand PPO-Clip}). It is important to note that if the \textit{signs} of the advantages are incorrect, it can lead to significant errors in computing the objective value during learning. However, due to the impressive empirical performance of neural networks in approximating values, erroneous signs of advantages tend to occur mainly when $\lvert {A}^{\pi}(s,a)\rvert$ is close to zero. Moreover, when $\lvert {A}^{\pi}(s,a)\rvert$ is near zero, its contribution to the objective remains relatively insignificant. Consequently, despite incorrect signs, the objective value remains reasonably accurate. 
This perspective offers an explanation for the robustness and impressive empirical performance of PPO-Clip. Additionally, this notion supports the potential of PPO-Clip to achieve convergence.
Furthermore, this concept is essential to comprehend the novel proof technique introduced in Lemma \ref{lm:ep}.
This lemma forms the cornerstone for bounding the errors in policy improvement and evaluation. For more detailed insights, please refer to Appendix \ref{app:main_thm}.

\noindent\textbf{Characterization of the Clipping Mechanism.} Our convergence analysis reveals that clipping mechanisms solely impact the pre-constant of convergence rates. Surprisingly, our analysis and results show that the clipping range $\epsilon$ only influences the \textit{pre-constant} of the Neural PPO-Clip convergence rate. This is unexpected since, intuitively, $\epsilon$ is considered analogous to the penalty parameter of PPO-KL \citep{liu2019neural}, which directly affects convergence rates. Contrary to expectations, we discover that the EMDA step size $\eta$ plays a crucial role in determining convergence rates, rather than the clipping range $\epsilon$. This result is illustrated by the involvement of the clipping mechanism in the EMDA subroutine through the indicator functions in the gradients.
Moreover, as the clipping range $\epsilon$ is contained inside the indicator function, \textit{it only influences the number of effective EMDA updates but not the magnitude of each EMDA update}. Since we know that the convergence rate is determined by the magnitude of the gradient updates (i.e., $U_C(T), L_C(T)$, which is $\eta$-dependent and $\eta$ is $T$-dependent), the clipping range can only affect the pre-constant of the convergence rate and the rate would still be $O(1 / \sqrt{T})$. 
For a more comprehensive understanding, please refer to Appendices \ref{app:main_thm} and \ref{app:add:cor}.

\section{Experiments}
\label{section:Discussions}
\textbf{Experimental Setup.}
{Given the convergence guarantees in Section \ref{section:analysis}, to better understand the empirical behavior of the generalized PPO-Clip objective, we further conduct experiments to evaluate Neural PPO-Clip with different classifiers. Specifically, we evaluate Neural PPO-Clip, Neural PPO-Clip-sub (as introduced in Section \ref{section:HPO}), and two additional classifiers, $\log(\pi_{\theta}(a|s))-\log(\pi_{\theta_t}(a|s))$ and $\sqrt{\rho_{s,a}(\theta)}-1$(termed as Neural PPO-Clip-log and Neural PPO-Clip-root), against benchmark approaches in several RL benchmark environments. Our implementations of Neural PPO-Clip are based on the RL Baseline3 Zoo framework \citep{rl-zoo3}. We test the algorithms in both MinAtar \citep{young19minatar} and OpenAI Gym environments \citep{brockman2016openai}. In addition, the algorithms are compared with popular baselines, including A2C and Rainbow. A2C follows the implementation and default settings from RL Baseline3 Zoo. For Rainbow, we adopt the configuration from \citep{ceron2021revisiting}.
Please refer to Appendix \ref{app:Experiments} for more details about our experiment settings.}

\noindent \textbf{Variants of Neural PPO-Clip Achieves Comparable Empirical Performance.}
{Figure \ref{fig:experiments} shows the training curves of Neural PPO-Clip with various classifiers and the benchmark methods. Notably, we observe that Neural PPO-Clip with various classifiers can achieve comparable or better performance than the baseline methods in both RL environments. To be mentioned, the performance of Rainbow is consistent with the results reported by \citep{ceron2021revisiting}. In summary, the outcomes depicted above underscore the practicality of the hinge loss reinterpretation of PPO-Clip within standard RL tasks. Furthermore, this approach positions classifier selection as a potential hyperparameter for the future deployment of PPO-Clip.}

\section{Concluding Remarks}
\label{section:conclusion}
The convergence behavior of PPO-Clip, a longstanding open problem, is addressed in this paper, providing the first convergence result and deeper insights. Our limitations are (i) analysis under discrete action space and (ii) reliance on NN error analysis, typically requiring large NN width. Despite the empirical success of PPO-Clip without this, our two-layer NN exploration suggests our results hold if approximation errors are well-managed. We anticipate this work will spark a deeper understanding of PPO-Clip within the RL community.

\section*{Acknowledgment}

We thank Hsuan-Yu Yao, Kai-Chun Hu, and Liang-Chun Ouyang for the helpful discussion and for providing insightful advice regarding the experiment.
This material is based upon work partially supported by the National Science and Technology Council (NSTC), Taiwan under Contract No. NSTC 112-2628-E-A49-023 and Contract No. NSTC 112-2634-F-A49-001-MBK and based upon work partially supported by the Higher Education Sprout Project of the National Yang Ming Chiao Tung University and Ministry of Education (MOE), Taiwan.

\bibliography{aaai24.bib}



\appendix
\onecolumn

\newpage
\section*{Appendix}
\label{section:appendix}

\section{Pseudo Code of Algorithms}
\label{app:pseudo_code}

\begin{algorithm}[!htbp]
\caption{Neural PPO-Clip (A More Detailed Version of Algorithm \ref{algo:1 incomplete})}
\label{algo:1}
\textbf{Input}: MDP $(\mathcal{S}, \mathcal{A}, \mathcal{P}, r, \gamma, \mu)$, Objective function $L_{\text{Hinge}}$, EMDA step size $\eta$, number of EMDA iterations $K$, number of SGD and TD update iterations $T_{\text{upd}}$, number of Neural PPO-Clip iterations $T$, the clipping range $\epsilon$ \\
\textbf{Initialization}: the policy $\pi_{\theta_0}$ as a uniform policy
    \begin{algorithmic}[1]
        \FOR{$t=1,\cdots,T-1$}
            \STATE Set temperature parameter $\tau_{t+1}$\;
            \STATE Sample the tuple $\{s_i, a_i, a_i^0, s_i',a_i'\}_{i=1}^{T_{\text{upd}}}$, where $(s_i, a_i) \sim \sigma_t, a_i^0 \sim \pi_0(\cdot|s_i)$, $s_i' \sim \mathcal{P}(\cdot|s_i, a_i)$ and $a_i' \sim \pi_{\theta_t}(\cdot|s_i')$\;
            \STATE Solve for $Q_{\omega_t} = \text{NN}(\omega_t;m_{Q})$ by using TD update as Algorithm \ref{algo:3}\;
            \STATE Calculate $V_{\omega_t}$ by Bellman expectation equation and the advantage $A_{\omega_t} = Q_{\omega_t} - V_{\omega_t}$\;
            \STATE {Use the states with nonzero advantage as the batch $\{s_i\}_{i=1}^{T_{\text{upd}}}$ for $L_{\text{Hinge}}(\theta)$ and obtain target policy $\widehat{\pi}_{t+1}$ and $C_t$ by using}
            
            {EMDA in Algorithm \ref{algo:2}}\;
            \STATE Solve for $f_{\theta_{t+1}} = \text{NN}(\theta_{t+1};m_f)$ by using SGD as Algorithm \ref{algo:4} based on the EMDA result\;
            \STATE Update the policy $\pi_{\theta_{t+1}} \propto \exp \{ \tau_{t+1}^{-1} f_{\theta_{t+1}}\}$\;
        \ENDFOR
    \end{algorithmic}
\end{algorithm}
\begin{remarkapp}
\label{remark:choice}
{In Neural PPO-Clip, there are various types of classifiers, the choices of the EMDA step size $\eta$ and the temperature parameters $\{\tau_{t}\}$ of the neural networks are important factors to the convergence rate and hence shall be configured properly according to the properties of different classifiers. As a result, we do not specify the specific choices of $\eta$ and $\{\tau_{t}\}$ in the following pseudo code of the generic Neural PPO-Clip. Please refer to Corollaries \ref{cor:PPO-Clip}-\ref{cor:sub} in Appendix \ref{app:add:cor} for the choices of $\eta$ and $\{\tau_{t}\}$ for Neural PPO-Clip with several classifiers including the standard PPO-Clip classifier $\rho_{s, a}(\theta) - 1 = \frac{\pi_{\theta}(a|s)}{\pi_{\theta_{t}}(a|s)} - 1$.}
\end{remarkapp}

For better readability, we restate EMDA (Algorithm \ref{algo:2}) here as Algorithm \ref{algo:2 copy}.
\begin{algorithm}[!htbp]
\caption{EMDA}
\label{algo:2 copy}
\textbf{Input}: $L_{\text{Hinge}}(\theta)$, EMDA step size $\eta$, number of EMDA iterations $K$, initial policy $\pi_{\theta_{t}}$, sample batch $\{s_i\}_{i=1}^{T_{\text{upd}}}$ \\
\textbf{Initialization}: $\tilde{\theta}^{(0)} = \pi_{\theta_{t}}$, $C_t(s, a) = 0$, for all $s, a$ \\
\textbf{Output}: $\widehat{\pi}_{t+1}$ and $C_t$
\begin{algorithmic}[1]
    \FOR{$k=0,\cdots,K-1$}
        \FOR{\text{each state} $s$ \text{in the batch}}
        \STATE Find $g_{s,a}^{(k)} = \left.\frac{\partial L_{\text{Hinge}}(\theta)}{\partial \theta_{s, a}}\right|_{\theta = \tilde{\theta}^{(k)}}$, for each $a$\;
        \STATE Let $w_s = (e^{-\eta g_{s,1}}, \dots, e^{-\eta g_{s,|\mathcal{A}|}})$\;
        \STATE $\tilde{\theta}^{(k+1)} = \frac{1}{\langle w_s, \tilde{\theta}^{(k)} \rangle} (w_s \circ \tilde{\theta}^{(k)})$\;
        \STATE $C_t(s, a) \leftarrow C_t(s, a) - \eta g_{s,a}^{(k)} / A_{\omega_t}(s, a)$, for 
        
        \ \ \ \qquad each $a$ with $A_{\omega_t}(s, a) \neq 0$\;
        \ENDFOR
    \ENDFOR
    \STATE $\widehat{\pi}_{t+1} = \tilde{\theta}^{(K)}$\;
\end{algorithmic}
\end{algorithm}

\begin{algorithm}[!htbp]
\caption{Policy Evaluation via TD}
\label{algo:3}
\textbf{Input}: MDP $(\mathcal{S}, \mathcal{A}, \mathcal{P}, r, \gamma)$, initial weights $b_i$, $[\omega(0)]_i \ (i \in [m_Q])$, number of iterations $T_{\text{upd}}$, sample $\{s_i, a_i, s_i', a_i\}_{i=1}^{T_{\text{upd}}}$ \\
\textbf{Output}: $Q_{\bar{\omega}}$
    \begin{algorithmic}[1]
        \STATE Set the step size $\eta_{\text{upd}} \leftarrow T_{\text{upd}}^{-1/2}$\;
        \FOR{$t=0,\cdots,T_{\text{upd}}-1$}
            \STATE $(s,a,s',a') \leftarrow (s_i,a_i,s_i',a_i')$\;
            \STATE $\omega(t + 1/2) \leftarrow \omega(t) - \eta_{\text{upd}} \cdot (Q_{\omega(t)}(s, a) -  r(s, a) - \gamma Q_{\omega(t)}(s', a')) \cdot \nabla_{\omega} Q_{\omega(t)}(s, a)$\;
            \STATE $\omega(t+1) \leftarrow \arg \min_{\omega \in B_Q}\{||\omega - \omega(t+1/2)||_2\}$\;
        \ENDFOR
        \STATE Take the average over path $\bar{\omega} \leftarrow 1/T_{\text{upd}} \cdot \sum_{t=0}^{T_{\text{upd}}-1} \omega(t)$\;
    \end{algorithmic}
\end{algorithm}

\begin{algorithm}[!htbp]
\caption{Policy Improvement via SGD}
\label{algo:4}
\textbf{Input}: MDP $(\mathcal{S}, \mathcal{A}, \mathcal{P}, r, \gamma)$, the current energy function $f_{\theta_t}$, initial weights $b_i$, $[\theta(0)]_i \ (i \in [m_f])$, number of iterations $T_{\text{upd}}$, sample $\{s_i, a_i^0\}_{i=1}^{T_{\text{upd}}}$ \\
\textbf{Output}: $f_{\bar{\theta}}$
    \begin{algorithmic}[1]
        \STATE Set the step size $\eta_{\text{upd}} \leftarrow T_{\text{upd}}^{-1/2}$\;
        \FOR{$t=0,\cdots,T_{\text{upd}}-1$}
            \STATE $(s,a) \leftarrow (s_i,a_i^0)$\;
            \STATE $\theta(t + 1/2) \leftarrow \theta(t) - \eta_{\text{upd}} \cdot (f_{\theta_t}(s,a) - \tau_{t+1} \cdot (C_t(s, a) \cdot A_{\omega_t}(s, a) + \tau_{t}^{-1}f_{\theta_t}(s, a))) \cdot \nabla_{\theta} f_{\theta_t}(s, a)$\;
            \STATE $\theta(t+1) \leftarrow \arg \min_{\theta \in B_f}\{||\theta - \theta(t+1/2)||_2\}$\;
        \ENDFOR
        \STATE Take the average over path $\bar{\theta} \leftarrow 1/T_{\text{upd}} \cdot \sum_{t=0}^{T_{\text{upd}}-1} \theta(t)$\;
    \end{algorithmic}
\end{algorithm}

\begin{algorithm}
\caption{Tabular PPO-Clip}
\label{algo:HPO-AM}
\textbf{Initialization}: policy $\pi^{(0)}=\pi(\theta^{(0)})$, initial state distribution $\mu$, step size of EMDA $\eta$, number of EMDA iterations $K^{(t)}$ \\
\textbf{Output}: Learned policy $\pi^{(\infty)}$
    \begin{algorithmic}[1]
        \FOR{$t=0, 1, \cdots$}
            \STATE Collect a set of trajectories $\tau \in \mathcal{D}^{(t)}$ under policy $\pi^{(t)}=\pi(\theta^{(t)})$\;
            \STATE Find ${A}^{(t)}$ by a policy evaluation method\;
            \STATE Compute $\hat{L}^{(t)}(\theta)$ based on ${A}^{(t)}$ and the collected samples in $\mathcal{D}^{(t)}$\;
            \STATE Update the policy by $\theta^{(t+1)}=\text{EMDA-tabular}(\hat{L}^{(t)}(\theta),\eta,K^{(t)},\cD^{(t)},\theta^{(t)})$\;
        \ENDFOR
    \end{algorithmic}
\end{algorithm}

For consistency in notation, we present the EMDA utilized in Tabular PPO-Clip as Algorithm \ref{algo:EMD}.

\begin{algorithm}
\caption{EMDA-tabular$(L(\theta),\eta,K,\cD, \theta_{\text{init}})$}
\label{algo:EMD}
\textbf{Input}: Objective $L(\theta)$, step size $\eta$, number of iteration $K$,  dataset $\cD$, and initial parameter $\theta_{\text{init}}$ \\
\textbf{Initialization}: $\widetilde{\theta}^{(0)}=\theta_{\text{init}}$, $\widetilde{\theta}=\theta_{\text{init}}$ \\
\textbf{Output}: Learned parameter $\tilde{\theta}$
    \begin{algorithmic}[1]
        \FOR{$k=0,\cdots,K-1$}
            \FOR{each state $s$ in $\cD$}
            \vspace{1mm}
                \STATE Find $g_{s,a}^{(k)}=\frac{\partial L(\theta)}{\partial \theta_{s,a}}\rvert_{\theta=\widetilde{\theta}^{(k)}}$, for each $a$\;
                \STATE Let $w_s=(e^{-\eta g_{s,1}^{(k)}},\cdots,e^{-\eta g_{s,\lvert \mathcal{A}\rvert}^{(k)}})$\;
                \STATE $\widetilde{\theta}_s^{(k+1)}=\frac{1}{\langle w_s, \widetilde{\theta}_s^{(k)}\rangle}(w_s \circ \widetilde{\theta}_s^{(k)})$\;
            \ENDFOR
        \ENDFOR
    \end{algorithmic}
\end{algorithm}

\newpage
\section{Proof of Proposition \ref{pp:PI}}
\label{app:B}
For completeness, we restate Proposition \ref{pp:PI} as follows.
\begin{propstar}[EMDA Target Policy]
    For the target policy obtained by the EMDA subroutine at the $t$-th iteration, we have
    \begin{align}
        \log \widehat{\pi}_{t+1}(a|s) \propto C_t(s, a) A_{\omega_t}(s, a) + \tau_{t}^{-1} f_{\theta_t}(s, a), 
    \end{align}
    where $C_t(s, a) A_{\omega_t}(s, a) = -\sum_{k=0}^{K-1} \eta g_{s,a}^{(k)}$ as given in Algorithm \ref{algo:2}.
\end{propstar}
\begin{proof}[Proof of Proposition \ref{pp:PI}]
    We expand the closed-form of the $\log$ of the EMDA target policy,
\begin{align}
    \log \widehat{\pi}_{t+1}(a|s) &= \log \left(\prod_{k=0}^{K^{(t)}-1} \frac{\exp(-\eta g_{s, a}^{(k)})}{\langle w_s, \tilde{\theta}^{(k)} \rangle} \cdot \pi_{\theta_t}(a|s)\right) \\
    &= \sum_{k=0}^{K^{(t)}-1} -\eta g_{s,a}^{(k)} - \sum_{k=0}^{K^{(t)}-1} \log (\langle w_s, \tilde{\theta}^{(k)}  \rangle) + \log \pi_{\theta_t}(a|s) \\
    &= \sum_{k=0}^{K^{(t)}-1} -\eta g_{s,a}^{(k)}  - \sum_{k=0}^{K^{(t)}-1} \log (\langle w_s, \tilde{\theta}^{(k)}  \rangle) + \tau_{t}^{-1} f_{\theta_t}(s, a) - \log (Z_t(s)) \\
    &\propto C_t(s, a) \cdot A_{\omega_t}(s, a) + \tau_{t}^{-1} f_{\theta_t}(s, a).
\end{align}
where $Z_t(s)$ is the normalizing factor of the policy at step $t$. Since both the $\sum_{k=0}^{K^{(t)}-1} \log (\langle w_s, \tilde{\theta}^{(k)}  \rangle)$ and $\log (Z_t(s))$ are state-dependent, we can cancel it under softmax policy. We obtain $C_t(s, a)$ from Algorithm \ref{algo:2} and complete the proof.
\end{proof}

\section{Proof of the Supporting Lemmas for Theorem \ref{thm:main}}
\label{app:main_thm}

{In the following, we slightly abuse the notations $ \mathbb{E}_{\tilde{\sigma}_t}$, $\mathbb{E}_{\sigma_t}$, and $\mathbb{E}_{\nu^*}$ to denote the expectations (over the respective distribution) conditioned on the policy $\pi_{\theta_t}$.}


\subsection{Additional Supporting Lemmas}
{Throughout this section, we slightly abuse the notation that we use $\mathbb{E}_{\text{init}}[\cdot]$ to denote the expectation over the initialization of neural networks. Also, we assume that Assumptions \ref{assump:func}, \ref{assump:reg}, and \ref{assump:con} hold in the following proofs.}
\begin{lemma}[Policy Evaluation Error]
\label{lm:PE_error}
The output $A_{\bar{\omega}} = Q_{\bar{\omega}} - V_{\bar{\omega}}$ of Algorithm \ref{algo:3} and Bellman expectation equation satisfies
\begin{align}
    \mathbb{E}_{\text{init,}\sigma_t}[(A_{\omega_t}(s, a) - A^{\pi_{\theta_t}}(s, a))^2] = O(R_Q^2 T_{\text{upd}}^{-1/2} + R_Q^{5/2} m_Q^{-1/4} + R_Q^3 m_Q^{-1/2}).
\end{align}
\end{lemma}

\noindent To prove Lemma \ref{lm:PE_error}, we start by stating a bound on the error of the estimated state-action value function.

\begin{lemma}[Theorem 4.6 in \citep{liu2019neural}]
\label{thm:4.6_liu}
The output $Q_{\bar{\omega}}$ of Algorithm \ref{algo:3} satisfies
\begin{align}
    \mathbb{E}_{\text{init,}\sigma_t}[(Q_{\omega_t}(s, a) - Q^{\pi_{\theta_t}}(s, a))^2] = O(R_Q^2 T_{\text{upd}}^{-1/2} + R_Q^{5/2} m_Q^{-1/4} + R_Q^3 m_Q^{-1/2}).
\end{align}
\end{lemma}

\begin{proof}[Proof of Lemma \ref{lm:PE_error}]
We are ready to show the policy evaluation error of the advantage function. First, we find the bound of $|A_{\omega_t}(s, a) - A^{\pi_{\theta_t}}(s, a)|$. We have
\begin{align}
    |A_{\omega_t}(s, a) - A^{\pi_{\theta_t}}(s, a)| &= |Q_{\omega_t}(s, a) - V_{\omega_t}(s) - Q^{\pi_{\theta_t}}(s, a) + V^{\pi_{\theta_t}}(s)| \\
    &= \left|Q_{\omega_t}(s, a) - Q^{\pi_{\theta_t}}(s, a) + \sum_{a'} \pi_{\theta_t}(a'|s) \cdot (Q^{\pi_{\theta_t}}(s, a') - Q_{\omega_t}(s, a'))\right| \\
    &= \left|Q_{\omega_t}(s, a) - Q^{\pi_{\theta_t}}(s, a) + \mathbb{E}_{a'\sim \pi_{\theta_t}}[Q^{\pi_{\theta_t}}(s, a') - Q_{\omega_t}(s, a')]\right| \\
    &\le |Q^{\pi_{\theta_t}}(s, a) - Q_{\omega_t}(s, a)| + |\mathbb{E}_{a'\sim \pi_{\theta_t}}[Q^{\pi_{\theta_t}}(s, a') - Q_{\omega_t}(s, a')]|.
\end{align}
Then, we can derive the bound of $(A^{\pi_{\theta_t}}(s, a) - A_{\omega_t}(s, a))^2$ as follows,
\begin{align}
    (A^{\pi_{\theta_t}}(s, a) - A_{\omega_t}(s, a))^2 &\le 2(Q^{\pi_{\theta_t}}(s, a) - Q_{\omega_t}(s, a))^2 + 2(\mathbb{E}_{a'\sim \pi_{\theta_t}}[Q^{\pi_{\theta_t}}(s, a') - Q_{\omega_t}(s, a')])^2 \label{eq:PE_error 1}\\
    &\le 2(Q^{\pi_{\theta_t}}(s, a) - Q_{\omega_t}(s, a))^2 + 2\mathbb{E}_{a'\sim \pi_{\theta_t}}[Q^{\pi_{\theta_t}}(s, a') - Q_{\omega_t}(s, a')^2],\label{eq:PE_error 2}
\end{align}
{where (\ref{eq:PE_error 2}) holds by Jensen's inequality.}
By taking the expectation of (\ref{eq:PE_error 1})-(\ref{eq:PE_error 2}) over the state-action distribution $\sigma_t$, we have
\begin{align}
    \mathbb{E}_{\sigma_t}[(&A^{\pi_{\theta_t}}(s, a) - A_{\omega_t}(s, a))^2] \label{eq:PE_error 3}\\
    &\le 2\mathbb{E}_{\sigma_t}[(Q^{\pi_{\theta_t}}(s, a) - Q_{\omega_t}(s, a))^2] + 2\mathbb{E}_{\sigma_t}[\mathbb{E}_{ a'\sim \pi_{\theta_t}}[(Q^{\pi_{\theta_t}}(s, a') - Q_{\omega_t}(s, a'))^2]] \label{eq:PE_error 4}\\
    &= 4\mathbb{E}_{\sigma_t}[(Q^{\pi_{\theta_t}}(s, a) - Q_{\omega_t}(s, a))^2].\label{eq:PE_error 5},
\end{align}
where the last equality in (\ref{eq:PE_error 5}) is obtained by the actions are directly sampled by $\pi_{\theta_t}$ so we can ignore it in the latter term. Last, we leverage Lemma \ref{thm:4.6_liu} to obtain the result of Lemma \ref{lm:PE_error}.
\end{proof}

\begin{lemma}[Policy Improvement Error]
\label{lm:PI_error}
The output $f_{\bar{\theta}}$ of Algorithm \ref{algo:4} satisfies
\begin{align}
    \mathbb{E}_{\text{init,}\tilde{\sigma}_t}&[(f_{\bar{\theta}}(s, a) - \tau_{t+1} \cdot (C_t(s, a) \cdot A_{\omega_t}(s, a) + \tau_t^{-1} f_{\theta_t}(s, a)))^2] \\
    = \ &O(R_f^2 T_{\text{upd}}^{-1/2} + R_f^{5/2} m_f^{-1/4} + R_f^3 m_f^{-1/2}), \nonumber
\end{align}
\end{lemma}

\noindent To prove Lemma \ref{lm:PI_error}, we first state the following useful result noindently proposed by \citep{liu2019neural}.
\begin{theorem}[\citep{liu2019neural}, Meta-Algorithm of Neural Networks]
\label{thm:meta}
    Consider a meta-algorithm with the following update:
    \begin{align}
        \alpha(t + 1/2) &\leftarrow \alpha(t) - \eta_{\text{upd}} \cdot (u_{\alpha(t)}(s, a) - v(s, a) - \mu \cdot u_{\alpha(t)}(s', a')) \cdot \nabla_{\alpha} u_{\alpha(t)}(s, a),     \label{eq:meta-algo 1}\\
        \alpha(t + 1) &\leftarrow \prod_{B_{\alpha}}(\alpha(t + 1/2)) = \mathop{\arg \min}_{\alpha \in B_{\alpha}} \lVert\alpha - \alpha(t + 1/2)\rVert_2, \label{eq:meta-algo 2}
    \end{align}
    where $\mu \in [0, 1)$ is a constant, $(s, a, s', a')$ is sampled from some stationary distribution $d$, $u_{\alpha}$ is parameterized as a two-layer neural network $\text{NN}(\alpha;m)$, and $v(s, a)$ satisfies
    \begin{align}
    \label{condition:v}
        \mathbb{E}_{d}[(v(s,a))^2] \le \bar{v}_1 \cdot \mathbb{E}_{d}[(u_{\alpha(0)}(s, a))^2] + \bar{v}_2 \cdot R_{u}^2 + \bar{v}_3,
    \end{align} 
    for some constants $\bar{v}_1, \bar{v}_2, \bar{v}_3 \ge 0$.
    We define the update operator $\gT u(s, a) = \mathbb{E}[v(s, a) + \mu \cdot u(s', a')|s' \sim \mathcal{P}(\cdot|s, a), a' \sim \pi(\cdot|s)]$, and define $\alpha^*$ as the \textit{approximate stationary point} (cf. (D.18) in \citep{liu2019neural}), which inherently have the property $u^0_{\alpha^*} = \prod_{\mathcal{F}_{R_u, m}}\mathcal{T}u^0_{\alpha^*}$, where $u_{\alpha^*}^0$ is the linearization of $u$ at $\alpha^*$.
    Suppose we run the above meta-algorithm in (\ref{eq:meta-algo 1})-(\ref{eq:meta-algo 2}) for $T$ iterations with $T \ge 64 / (1 - \mu)^2$ and set the step size $\eta_{\text{upd}} = T^{-1/2}$. Then, we have
    \begin{align}
        \label{thm:meta:eq1}
        \mathbb{E}_{\text{init,}d}[(u_{\bar{\alpha}}(s, a) - u_{\alpha*}^0(s, a))^2] &= O(R_u^2 T_{\text{upd}}^{-1/2} + R_u^{5/2} m^{-1/4} + R_u^3 m^{-1/2}), \\
        \label{thm:meta:eq2}
        \mathbb{E}_{\text{init,}d}[(u_{\alpha'}(s, a) - u_{\alpha'}^0(s, a))^2] &= O(R_u^3 m^{-1/2}),
    \end{align}
    where $\bar{\alpha} \coloneqq 1/T \cdot (\sum_{t=0}^{T-1} \alpha(t))$
    and $\alpha'$ is a parameter in $B_{\alpha}$.
\end{theorem}

\begin{proof}[Proof of Lemma \ref{lm:PI_error}]
Now we are ready to prove Lemma \ref{lm:PI_error} as follows. To begin with, (\ref{eq:meta-algo 1})-(\ref{eq:meta-algo 2}) match the policy improvement update of {Neural PPO-Clip} if we put $u(s, a) = f(s, a)$, $v(s, a) = \tau_{t+1} (C_t(s, a) \cdot A_{
\omega_t}(s, a) + \tau_t^{-1} f_{\theta_t}(s, a))$, $\mu = 0$, $d = \tilde{\sigma}_t$, and $R_{u} = R_f$. For $\mathbb{E}_{\tilde{\sigma}_t}[(v(s,a))^2]$, we have
\begin{align}
    \mathbb{E}_{\tilde{\sigma}_t}[(v(s,a))^2] &\le 2\tau_{t+1}^2 (U_{C}^2 \cdot \mathbb{E}_{\tilde{\sigma}_t}[(A_{\omega_t}(s, a))^2] + \tau^{-2}_t \mathbb{E}_{\tilde{\sigma}_t}[(f_{\theta_t}(s, a))^2]) \\
    &\le 20 \mathbb{E}_{\tilde{\sigma}_t}[(f_{\theta_0}(s, a))^2] + 20 R_f^2.\label{eq:PI_error 1}
\end{align}
{Here, since $C_t$ and $\bar{C_t}$ are dependent only on the EMDA step size $\eta$ and the indicator function that depends on the sign of the advantage (either under the true advantage $A^{\pi_{\theta_t}}$ or the approximated advantage $A_{\omega_t}$), one can always find one common upper bound $U_C(T)$ for both $C_t$ and $\bar{C_t}$.
In particular, as shown in Corollary \ref{cor:PPO-Clip}, we set $U_{C} = \sum_{k=0}^{K-1} \eta$ for PPO-Clip, which is independent from the advantage function.}
{
The inequality in (\ref{eq:PI_error 1}) holds by the condition that $\tau_{t+1}^2 (U_{C}^2 + \tau_{t}^{-2}) \le 1$, $(a + b)^2 \le 2a^2 + 2b^2$, $\mathbb{E}_{\tilde{\sigma}_t}[(A_{\omega_t}(s, a))^2] \le 4 \mathbb{E}_{\tilde{\sigma}_t}[(Q_{\omega_t}(s, a))^2] $, and $\mathbb{E}_{\tilde{\sigma}_t}[(u_{\alpha_t}(s, a))^2] \le 2 \mathbb{E}_{\tilde{\sigma}_t}[(u_{\alpha_0}(s, a))^2] + 2 R_f^2$ which holds by using the Lipschitz property of neural networks where $u_{\alpha} = f_{\theta}, A_{\omega}$.
The condition $\tau_{t+1}^2 (U_{C}^2 + \tau_{t}^{-2}) \le 1$ can be satisfied by configuring proper $\{\tau_t\}$, as described momentarily in Appendix \ref{app:add:cor}.}
We also use that $\mathbb{E}_{\tilde{\sigma}_t}[Q_{\omega(0)}] = \mathbb{E}_{\tilde{\sigma}_t}[f_{\theta(0)}]$ because they share the same initialization. Thus, we have $\bar{v}_1 = \bar{v}_2 = 20$ and $\bar{v}_3 = 0$ in (\ref{condition:v}).

Due to that $\theta^*$ is the approximate stationary point, we have $f^0_{\theta^*} = \prod_{\mathcal{F}_{R_f, m_f}}\mathcal{T}f^0_{\theta^*} = \prod_{\mathcal{F}_{R_f, m_f}} \tau_{t+1} (C_t \circ A_{\omega_t} + \tau_t^{-1} f_{\theta_t})$. Thus, 
\begin{align}
    f^0_{\theta^*} = \mathop{\arg\min}_{f \in \mathcal{F}_{R_f, m_f}} \lVert f - \tau_{t+1} (C_t \circ A_{\omega_t} + \tau_t^{-1} f_{\theta_t})\rVert_{2, \tilde{\sigma}_t},
\end{align}
where $\lVert \cdot \rVert_{2, \tilde{\sigma}_t} = \mathbb{E}_{\text{init,}\tilde{\sigma}_t}[\lVert \cdot \rVert_{2}]^{1/2}$ is the $\tilde{\sigma}_t$-weighted $\ell_2$-norm.
Then, by the fact that $\tau_{t+1} (C_t(s, a) \cdot A_{\omega_t}^0(s, a) + \tau_t^{-1} f_{\theta_t}^0(s, a))\in\mathcal{F}_{R_f, m_f}$ and that $A_{\omega_t}^0(s, a) = Q_{\omega_t}^0(s, a) - \sum_{a\in\mathcal{A}} \pi(a|s)Q_{\omega_t}^0(s, a)$, we obtain
\begin{align}
    &\mathbb{E}_{\text{init,}\tilde{\sigma}_t}[(f^0_{\theta^*}(s, a) - \tau_{t+1} (C_t(s, a) \cdot A_{\omega_t}(s, a) + \tau_t^{-1} f_{\theta_t}(s, a)))^2] \\
    &\le \mathbb{E}_{\text{init,}\tilde{\sigma}_t}[(\tau_{t+1} (C_t(s, a)  A_{\omega_t}^0(s, a) + \tau_t^{-1} f_{\theta_t}^0(s, a)) - (\tau_{t+1} (C_t(s, a) A_{\omega_t}(s, a) + \tau_t^{-1} f_{\theta_t}(s, a))))^2] \\
    &\le 2\tau_{t+1}^2 U_{C}^2 \mathbb{E}_{\text{init,}\tilde{\sigma}_t}[((Q_{\omega_t}^0(s, a) - \sum_{a' \in\mathcal{A}} \pi(a'|s)Q_{\omega_t}^0(s, a')) - (Q_{\omega_t}(s, a) - \sum_{a'in\mathcal{A}} \pi(a'|s)Q_{\omega_t}(s, a')))^2] \nonumber \\
    &\qquad + 2\tau_{t+1}^2 \tau_{t}^{-2} \mathbb{E}_{\text{init,}\tilde{\sigma}_t}[(f_{\theta_t}^0(s, a) - f_{\theta_t}(s, a))^2] \\
    \label{eq:82}
    &\le 8 \tau_{t+1}^2 U_{C}^2 \mathbb{E}_{\text{init,}\tilde{\sigma}_t}[(Q_{\omega_t}^0(s, a) - Q_{\omega_t}(s, a))^2] + 2\tau_{t+1}^2 \tau_{t}^{-2} \mathbb{E}_{\text{init,}\tilde{\sigma}_t}[(f_{\theta_t}^0(s, a) - f_{\theta_t}(s, a))^2] \\
    \label{short:result}
    &= O(R_f^3 m_f^{-1/2}).
\end{align}
We obtain (\ref{eq:82}) as the same reason in (\ref{eq:PE_error 1})-(\ref{eq:PE_error 5}) in the proof of Lemma \ref{lm:PE_error}. The terms in (\ref{eq:82}) are both the designated form as the (\ref{thm:meta:eq2}), we leverage the (\ref{thm:meta:eq2}) in Theorem \ref{thm:meta} and obtain the result in (\ref{short:result}).

Last, we bound the error of our policy improvement, we have
\begin{align}
    &\mathbb{E}_{\text{init,}\tilde{\sigma}_t}[(f_{\bar{\theta}}(s, a) - \tau_{t+1} \cdot (C_t(s, a) \cdot A_{\omega_t}(s, a) + \tau_t^{-1} f_{\theta_t}(s, a)))^2] \\
    \label{eq1:final:meta}
    &\le 2\mathbb{E}_{\text{init,}\tilde{\sigma}_t}[(f_{\bar{\theta}}(s, a) - f^0_{\theta^*}(s, a))^2] \\
    \label{eq2:final:meta}
    &\qquad + 2 \mathbb{E}_{\text{init,}\tilde{\sigma}_t}[(f^0_{\theta^*}(s, a) - \tau_{t+1} (C_t(s, a) \cdot A_{\omega_t}(s, a) + \tau_t^{-1} f_{\theta_t}(s, a)))^2] \\
    \label{meta:result}
    &= O(R_f^2 T_{\text{upd}}^{-1/2} + R_f^{5/2} m_f^{-1/4} + R_f^3 m_f^{-1/2}),
\end{align}
where (\ref{eq1:final:meta}) is bounded as $O(R_f^2 T_{\text{upd}}^{-1/2} + R_f^{5/2} m_f^{-1/4} + R_f^3 m_f^{-1/2})$ by (\ref{thm:meta:eq1}) of Theorem \ref{thm:meta}, and (\ref{eq2:final:meta}) is bounded as $O(R_f^3 m_f^{-1/2})$ by the derivation of (\ref{short:result}). Thus, we obtain (\ref{meta:result}) and complete the proof.
\end{proof}

\begin{lemma}[Error Probability of Advantage]
\label{lm:EPA}
    {Given the policy $\pi_{\theta_t}$, the probability of the event that the advantage error is greater than $\epsilon_{\text{err}}$ can be bounded as}
    \begin{align}
        \mathbb{P}(|A_{\omega_t}(s, a) - A^{\pi_{\theta_t}}(s, a)| > \epsilon_{\text{err}}) \le \frac{\mathbb{E}_{\text{init,}\sigma_t}[(A_{\omega_t}(s, a) - A^{\pi_{\theta_t}}(s, a))^2]}{\epsilon_{\text{err}}^2}.\label{eq:Markov inequality}
    \end{align}
\end{lemma}
\begin{proof}[Proof of Lemma \ref{lm:EPA}]
    By applying Markov's inequality, we have
    \begin{align}
        \mathbb{P}(|A_{\omega_t}(s, a) - A^{\pi_{\theta_t}}(s, a)| > \epsilon_{\text{err}}) &= \mathbb{P}(|A_{\omega_t}(s, a) - A^{\pi_{\theta_t}}(s, a)|^2 > \epsilon_{\text{err}}^2) \\
        &\le \frac{\mathbb{E}[(A_{\omega_t}(s, a) - A^{\pi_{\theta_t}}(s, a))^2]}{\epsilon_{\text{err}}^2}.
    \end{align}
\end{proof}
\noindent Notice that the randomness of the above event in (\ref{eq:Markov inequality}) comes from the state-action visitation distribution $\sigma_t$ and the initialization of the neural networks.

\begin{lemma}[Error Propagation]
\label{lm:ep}
    Let $\pi_{t+1}$ be the target policy obtained by EMDA with the true advantage. Suppose the policy improvement error satisfies 
    \begin{align}
    \label{lm:ep:eq1}
        \mathbb{E}_{\tilde{\sigma}_t}&[(f_{\theta_{t+1}}(s, a) - \tau_{t+1} \cdot (C_t(s, a) \cdot A_{\omega_t}(s, a) + \tau_t^{-1} f_{\theta_t}(s, a)))^2] \le \epsilon_{t+1},
    \end{align}
    and the policy evaluation error satisfies
    \begin{align}
    \label{lm:ep:eq2}
        \mathbb{E}_{\sigma_t}[(A_{\omega_t}(s, a) - A^{\pi_{\theta_t}}(s, a))^2] \le \epsilon_t'.
    \end{align}
    Then, the following holds, 
    \begin{align}
        |\mathbb{E}_{\nu^*}[\langle \log\pi_{\theta_{t+1}}(\cdot|s) - \log  \pi_{t+1}(\cdot|s), \pi^*(\cdot|s) - \pi_{\theta_t}(\cdot|s) \rangle]| \le \varepsilon_t + \varepsilon_{\text{err}}
        \label{lm:ep:eq3}
    \end{align}
    where $\varepsilon_t = C_{\infty} \tau_{t+1}^{-1} \phi^* \epsilon_{t+1}^{1/2} + U_{C} X^{1/2} \psi^* \epsilon_t'^{1/2}$ and $\varepsilon_{\text{err}} = \sqrt{2} U_{C} \epsilon_{\text{err}} \psi^*$, and $X = \left[(2 / \epsilon_{\text{err}}^2)(M' + (R_{\max} / (1 - \gamma))^2 - \epsilon_t'/2)\right]$, and $M' = 4\mathbb{E}_{\nu_t}[\max_{a} (Q_{\omega_0}(s, a))^2] + 4R_f^2$.
\end{lemma}

\begin{remarkapp}
\label{remark:varepislon}
    Notice that $\epsilon_{t+1}$ in (\ref{lm:ep:eq1}) and $\epsilon_t'$ in (\ref{lm:ep:eq2}) can be controlled by the width of neural networks and the number of iteration for each SGD and TD updates based on Lemma \ref{lm:PE_error} and \ref{lm:PI_error}. Therefore, $\varepsilon_t$ could be made sufficiently small per our requirement.
\end{remarkapp}

\begin{proof}[Proof of Lemma \ref{lm:ep}]
{For ease of exposition, let us first fix a policy $\pi_{\theta_t}$. Through the analysis, we will show that one can derive an upper bound (in the form of (\ref{lm:ep:eq3})) that holds regardless of the policy $\pi_{\theta_t}$. 
Recall that $C_t(s, a)= -\sum_{k=0}^{K^{(t)}-1} \eta g_{s,a}^{(k)}$, where $g_{s,a}^{(k)}$ is obtained in the EMDA subroutine and depends on the sign of the estimated advantage $A_{\omega_t}$. Similarly, we define $\bar{C_t}(s, a)$ as the counterpart of $C_t(s,a)$ by replacing $A_{\omega_t}$ with the true advantage $A^{\pi_{\theta_t}}$.}
We first simplify $\langle \log\pi_{\theta_{t+1}}(\cdot|s) - \log  \pi_{t+1}(\cdot|s), \pi^*(\cdot|s) - \pi_{\theta_t}(\cdot|s) \rangle$. The normalizing factor $Z$ of the policies $\pi_{\theta_{t+1}}$ and $\pi_{t+1}$ is state-dependent, {and the inner product between any state-dependent function and the policy difference $\pi^*(\cdot|s) - \pi_{\theta_t}(\cdot|s)$ is always zero.}
Thus, we have
\begin{align}
    \langle \log\pi_{\theta_{t+1}}(\cdot|s) &- \log  \pi_{t+1}(\cdot|s), \pi^*(\cdot|s) - \pi_{\theta_t}(\cdot|s) \rangle  \\
    &= \langle \tau_{t+1}^{-1} f_{\theta_{t+1}}(s, \cdot) - (\bar{C_t}(s, \cdot) \circ A^{\pi_{\theta_t}}(s, \cdot) + \tau_{t}^{-1} f_{\theta_{t}}(s, \cdot)), \pi^*(\cdot|s) - \pi_{\theta_t}(\cdot|s) \rangle.
\end{align}
Then, {we decompose the above equation into two terms:} (i) the error in the policy improvement and (ii) the error between the true advantage and the approximated advantage, i.e., 
\begin{align}
    \langle \tau_{t+1}^{-1} &f_{\theta_{t+1}}(s, \cdot) - (\bar{C_t}(s, \cdot) \circ A^{\pi_{\theta_t}}(s, \cdot) + \tau_{t}^{-1} f_{\theta_{t}}(s, \cdot)), \pi^*(\cdot|s) - \pi_{\theta_t}(\cdot|s) \rangle \label{eq:ep proof 1}\\
    &= \langle \tau_{t+1}^{-1} f_{\theta_{t+1}}(s, \cdot) - (C_t(s, \cdot) \circ A_{\omega_t}(s, \cdot) + \tau_{t}^{-1} f_{\theta_{t}}(s, \cdot)), \pi^*(\cdot|s) - \pi_{\theta_t}(\cdot|s) \rangle  \label{eq:ep proof 2}\\
    &\qquad + \langle C_t(s, \cdot) \circ A_{\omega_t}(s, \cdot) - \bar{C_t}(s, \cdot) \circ A^{\pi_{\theta_t}}(s, \cdot), \pi^*(\cdot|s) - \pi_{\theta_t}(\cdot|s) \rangle \label{eq:ep proof 3}
\end{align}
We first bound the expectation of (i) over $\nu^*$ as follows.
\begin{align}
    &|\mathbb{E}_{\nu^*}[\langle \tau_{t+1}^{-1} f_{\theta_{t+1}}(s, \cdot) - (C_t(s, \cdot) \circ A_{\omega_t}(s, \cdot) + \tau_{t}^{-1} f_{\theta_{t}}(s, \cdot)), \pi^*(\cdot|s) - \pi_{\theta_t}(\cdot|s) \rangle]|  \label{eq:ep proof 4}\\
    &= \left|\int_{\mathcal{S}} \langle \tau_{t+1}^{-1} f_{\theta_{t+1}}(s, \cdot) - (C_t(s, \cdot) \circ A_{\omega_t}(s, \cdot) + \tau_{t}^{-1} f_{\theta_{t}}(s, \cdot)), \pi^*(\cdot|s) - \pi_{\theta_t}(\cdot|s) \rangle \cdot \nu^*(s) ds \right|  \label{eq:ep proof 5}\\
    &= \left|\int_{\mathcal{S} \times \mathcal{A}} (\tau_{t+1}^{-1} f_{\theta_{t+1}}(s, a) - (C_t(s, a) A_{\omega_t}(s, a) + \tau_{t}^{-1} f_{\theta_{t}}(s, a)))  \left(\frac{\pi^*(a|s)}{\pi_0(a|s)} - \frac{\pi_{\theta_t}(a|s)}{\pi_0(a|s)}\right) \frac{\nu^*(s)}{\nu_t (s)} d \tilde{\sigma}_t(s, a) \right|  \label{eq:ep proof 6}\\
    &\le C_{\infty}\mathbb{E}_{\tilde{\sigma}_t}\left[(\tau_{t+1}^{-1} f_{\theta_{t+1}}(s, a) - (C_t(s, a) A_{\omega_t}(s, a) + \tau_{t}^{-1} f_{\theta_{t}}(s, a)))^2\right]^{1/2} \cdot \mathbb{E}_{\tilde{\sigma}_t}\left[\left|\frac{d \pi^*}{d\pi_0} - \frac{d \pi_{\theta_t}}{d \pi_0}\right|^2\right]^{1/2}  \label{eq:ep proof 7}\\
    &\le C_{\infty} \tau_{t+1}^{-1}\epsilon_{t+1}^{1/2} \phi^*_t, \label{eq:ep proof 8}
\end{align}
{where (\ref{eq:ep proof 6}) follows from the definition of $\tilde{\sigma}_t$, (\ref{eq:ep proof 7}) is obtained by Cauchy-Schwarz inequality and Assumption \ref{assump:con}, and the last inequality in (\ref{eq:ep proof 8}) holds by the condition in (\ref{lm:ep:eq1}) and that $\lVert \nu^* / \nu\rVert_{\infty}< C_{\infty}$.}

{Similarly, we consider the expectation of (ii) over $\nu^*$ as follows.}
\begin{align}
    &|\mathbb{E}_{\nu^*}[\langle C_t(s, \cdot) \circ A_{\omega_t}(s, \cdot) - \bar{C_t}(s, \cdot) \circ A^{\pi_{\theta_t}}(s, \cdot), \pi^*(\cdot|s) - \pi_{\theta_t}(\cdot|s) \rangle]|  \label{eq:ep proof 9}\\
    &= \left|\int_{\mathcal{S}} \langle C_t(s, \cdot) \circ A_{\omega_t}(s, \cdot) - \bar{C_t}(s, \cdot) \circ A^{\pi_{\theta_t}}(s, \cdot), \pi^*(\cdot|s) - \pi_{\theta_t}(\cdot|s) \rangle \nu^*(s) ds \right|  \label{eq:ep proof 10}\\
    &= \left|\int_{\mathcal{S} \times \mathcal{A}} (C_t(s, a) A_{\omega_t}(s, a) - \bar{C_t}(s, a)  A^{\pi_{\theta_t}}(s, a)) \left(\frac{\pi^*(a|s)}{\pi_{\theta_t}(a|s)} - \frac{\pi_{\theta_t}(a|s)}{\pi_{\theta_t}(a|s)}\right) \frac{\nu^*(s)}{\nu_t(s)} d\sigma_t(s, a)\right|  \label{eq:ep proof 11}\\
    &= \left|\int_{\mathcal{S} \times \mathcal{A}} (C_t(s, a) A_{\omega_t}(s, a) - \bar{C_t}(s, a)  A^{\pi_{\theta_t}}(s, a)) \left(\frac{\sigma^*(s, a)}{\sigma_t(s, a)} - \frac{\nu^*(s)}{\nu_t(s)}\right) d\sigma_t(s, a)\right|  \label{eq:ep proof 12}\\
    &\le \mathbb{E}_{\sigma_t}[(C_t(s, a) A_{\omega_t}(s, a) - \bar{C_t}(s, a)  A^{\pi_{\theta_t}}(s, a))^2]^{1/2} \cdot \mathbb{E}_{\sigma_t}\left[\left|\frac{d \sigma^*}{d \sigma_t} - \frac{d \nu^*}{d \nu_t}\right|^2\right]^{1/2}, \label{eq:ep proof 13}
\end{align}
where (\ref{eq:ep proof 13}) holds by the Cauchy-Schwarz inequality.
Next, we bound for the term $\mathbb{E}_{\sigma_t}[(C_t(s, a) A_{\omega_t}(s, a) - \bar{C_t}(s, a)  A^{\pi_{\theta_t}}(s, a))^2]$.
For ease of notation, let $D = (C_t(s, a) A_{\omega_t}(s, a) - \bar{C_t}(s, a)  A^{\pi_{\theta_t}}(s, a))^2$ and simply write $\mathbb{E}_{\text{init, }\sigma_t}$ as $\mathbb{E}$. 
{Also, we slightly abuse the notation by using $A_{\omega_t}$ as the random variable $A_{\omega_t}(s, a)$, whose randomness results from the state-action pairs sampled from $\sigma_t$ and the initialization of neural networks, and using $A^{\pi_{\theta_t}}$ as the random variable $A^{\pi_{\theta_t}}(s, a)$, whose randomness comes from the state-action pairs sampled from $\sigma_t$.} 
To establish the bound of $\mathbb{E}[D]$, we consider two different cases for $\mathbb{E}[D]$: one is that the error is greater than $\epsilon_{\text{err}}$, and the other is that the error is less than or equal to $\epsilon_{\text{err}}$. 
Specifically,
\begin{align}
    \mathbb{E}[D] &= \mathbb{E}[D \mid |A_{\omega_t} - A^{\pi_{\theta_t}}| > \epsilon_{\text{err}}] \cdot \mathbb{P}(|A_{\omega_t} - A^{\pi_{\theta_t}}| > \epsilon_{\text{err}}) \nonumber\\
    \label{lm:ep:proof:eq1}
    &\quad + \mathbb{E}[D \mid |A_{\omega_t} - A^{\pi_{\theta_t}}| \le \epsilon_{\text{err}}] \cdot \mathbb{P}(|A_{\omega_t} - A^{\pi_{\theta_t}}| \le \epsilon_{\text{err}})
\end{align}
Then, we upper bound the two terms in (\ref{lm:ep:proof:eq1}) separately. 
Regarding the first term in (\ref{lm:ep:proof:eq1}), we have
\begin{align}
    \mathbb{E}[&D \mid |A_{\omega_t} - A^{\pi_{\theta_t}}| > \epsilon_{\text{err}}] \cdot \mathbb{P}(|A_{\omega_t} - A^{\pi_{\theta_t}}| > \epsilon_{\text{err}}) \nonumber \\
    &\le 2 U_{C}^2 (\mathbb{E}_{\nu_t}[\lVert A_{\omega_t}(s,\cdot)\rVert_{\infty}^{2}] + (A^{\pi_{\theta_t}}_{\max})^2) \cdot \mathbb{P}(|A_{\omega_t} - A^{\pi_{\theta_t}}| > \epsilon_{\text{err}}),\label{eq:ep proof 14}
\end{align}
where (\ref{eq:ep proof 14}) holds by that $(a + b)^2 \le 2a^2 + 2b^2$.
Next, regarding the second term in (\ref{lm:ep:proof:eq1}), we further consider two cases based on whether the absolute value of $A^{\pi_{\theta_t}}$ is greater than $\epsilon_{\text{err}}$ or not. 
Specifically,
\begin{align}
    &\mathbb{E}[D \mid |A_{\omega_t} - A^{\pi_{\theta_t}}| \le \epsilon_{\text{err}}] \nonumber \\
    &= \mathbb{E}[D \mid |A_{\omega_t} - A^{\pi_{\theta_t}}| \le \epsilon_{\text{err}}, |A^{\pi_{\theta_t}}| > \epsilon_{\text{err}}] \cdot \mathds{1}\{|A^{\pi_{\theta_t}}| > \epsilon_{\text{err}}\} \nonumber \\
    &\quad + \mathbb{E}[D \mid |A_{\omega_t} - A^{\pi_{\theta_t}}| \le \epsilon_{\text{err}}, |A^{\pi_{\theta_t}}| \le \epsilon_{\text{err}}] \cdot \mathds{1}\{|A^{\pi_{\theta_t}}| \le \epsilon_{\text{err}}\} \label{eq:ep proof 15}\\
    &\le \mathbb{E}[D \mid |A_{\omega_t} - A^{\pi_{\theta_t}}| \le \epsilon_{\text{err}}, |A^{\pi_{\theta_t}}| > \epsilon_{\text{err}}] + \mathbb{E}[D \mid |A_{\omega_t} - A^{\pi_{\theta_t}}| \le \epsilon_{\text{err}}, |A^{\pi_{\theta_t}}| \le \epsilon_{\text{err}}] \label{eq:ep proof 16}\\
    &\le U_{C}^2 \cdot \mathbb{E}[(A_{\omega_t}(s, a) - A^{\pi_{\theta_t}}(s, a))^2] + 4 U_{C}^2 \epsilon_{\text{err}}^2 \label{eq:ep proof 17}
\end{align}
{where (\ref{eq:ep proof 15}) holds by the fact that we fix a policy $\pi_{\theta_t}$ and hence $A^{\pi_{\theta_t}}$ is determined, (\ref{eq:ep proof 16}) holds by that the indicator function is no larger than 1, the first term in (\ref{eq:ep proof 17}) holds by the fact that $A_{\omega_t}$ and $A^{\pi_{\theta_t}}$ have the same sign and hence $C_t$ is equal to $\bar{C}_t$, and the second term in (\ref{eq:ep proof 17}) follows from that $(a + b)^2 \le 2a^2 + 2b^2$.}
Then, by combining the above terms, we have
\begin{align}
    \mathbb{E}[D] &\le 2 U_{C}^2 (\mathbb{E}_{\nu_t}[\lVert A_{\omega_t}(s,\cdot)\rVert_{\infty}^{2}] + (A^{\pi_{\theta_t}}_{\max})^2) \cdot \mathbb{P}(|A_{\omega_t} - A^{\pi_{\theta_t}}| > \epsilon_{\text{err}}) \nonumber \\
    &\quad + [U_{C}^2 \cdot \mathbb{E}[(A_{\omega_t}(s, a) - A^{\pi_{\theta_t}}(s, a))^2] + 4 U_{C}^2 \epsilon_{\text{err}}^2] \cdot \mathbb{P}(|A_{\omega_t} - A^{\pi_{\theta_t}}| \le \epsilon_{\text{err}}) \\
    &= 2 U_{C}^2 (\mathbb{E}_{\nu_t}[\lVert A_{\omega_t}(s,\cdot)\rVert_{\infty}^{2}] + (A^{\pi_{\theta_t}}_{\max})^2) \cdot \mathbb{P}(|A_{\omega_t} - A^{\pi_{\theta_t}}| > \epsilon_{\text{err}}) \nonumber \\
    &\quad + [U_{C}^2 \cdot \mathbb{E}[(A_{\omega_t}(s, a) - A^{\pi_{\theta_t}}(s, a))^2] + 4 U_{C}^2 \epsilon_{\text{err}}^2] \cdot (1 - \mathbb{P}(|A_{\omega_t} - A^{\pi_{\theta_t}}| > \epsilon_{\text{err}}))
\end{align}
Recall that $\epsilon_t' = \mathbb{E}[(A_{\omega_t}(s, a) - A^{\pi_{\theta_t}}(s, a))^2]$. 
{As we could choose an $\epsilon_{\text{err}}$ small enough and use the neural network power to make $\epsilon_t'$ is also small by Lemma \ref{lm:PE_error} such that we have $2 U_{C}^2 (\mathbb{E}_{\nu_t}[\lVert A_{\omega_t}(s,\cdot)\rVert_{\infty}^{2}] + A^{\pi_{\theta_t}}_{\max}) > U_{C}^2 \epsilon_t' + 4 U_{C}^2 \epsilon_{\text{err}}^2$, then by Lemma \ref{lm:EPA} we have}
\begin{align}
\label{lm4:eq}
    \mathbb{E}[D] &\le 2 U_{C}^2 (\mathbb{E}_{\nu_t}[\lVert A_{\omega_t}(s,\cdot)\rVert_{\infty}^{2}] + (A^{\pi_{\theta_t}}_{\max})^2) \cdot \frac{\epsilon_t'}{\epsilon_{\text{err}}^2} + [U_{C}^2 \epsilon_t' + 4 U_{C}^2 \epsilon_{\text{err}}^2]\cdot(1 - \frac{\epsilon_t'}{\epsilon_{\text{err}}^2}).
\end{align}
Rearranging the terms in (\ref{lm4:eq}), we have
\begin{align}
    \mathbb{E}[D] &\le \epsilon_t' U_{C}^2 \cdot \left[\frac{2}{\epsilon_{\text{err}}^2}(M' + (A^{\pi_{\theta_t}}_{\max})^2 - \frac{\epsilon_t'}{2}) - 1\right] + 4 U_{C}^2 \epsilon_{\text{err}}^2 \\
    &\le \epsilon_t' U_{C}^2 \cdot \left[\frac{2}{\epsilon_{\text{err}}^2}(M' + (A^{\pi_{\theta_t}}_{\max})^2 - \frac{\epsilon_t'}{2})\right] + 4 U_{C}^2 \epsilon_{\text{err}}^2 
\end{align}
where $M' \coloneqq 4\mathbb{E}_{\nu_t}[\max_{a} (Q_{\omega_0}(s, a))^2] + 4R_f^2$. 
By introducing the notation $X = \left[(2 / \epsilon_{\text{err}}^2)(M' + (A^{\pi_{\theta_t}}_{\max})^2 - \epsilon_t'/2)\right]$ and combining all the above results, we have
\begin{align}
    |\mathbb{E}_{\nu^*}[\langle &\log\pi_{\theta_{t+1}}(\cdot|s) - \log  \pi_{t+1}(\cdot|s), \pi^*(\cdot|s) - \pi_{\theta_t}(\cdot|s) \rangle]| \\
    &\le C_{\infty} \tau_{t+1}^{-1}\epsilon_{t+1}^{1/2} \phi^*_t + (\epsilon_t' U_{C}^2 X + 4 U_{C}^2  \epsilon_{\text{err}}^2)^{1/2} \psi^*_t \\
    &\le \epsilon_{t+1}^{1/2} C_{\infty} \tau_{t+1}^{-1} \phi^*_t + \epsilon_t'^{1/2} U_{C} X^{1/2} \psi^*_t + 2 U_{C}  \epsilon_{\text{err}} \psi^*_t,\label{eq:ep proof 19} \\
    &< \epsilon_{t+1}^{1/2} C_{\infty} \tau_{t+1}^{-1} \phi^* + \epsilon_t'^{1/2} U_{C} X^{1/2} \psi^* + 2 U_{C}  \epsilon_{\text{err}} \psi^*,\label{eq:ep proof 20}
\end{align}
where (\ref{eq:ep proof 19}) follows from the inequality $\sqrt{a+b} \le \sqrt{a} + \sqrt{b}$ and that $\varepsilon_t = \epsilon_{t+1}^{1/2} C_{\infty} \tau_{t+1}^{-1} \phi^* + \epsilon_t'^{1/2} U_{C} X^{1/2} \psi^*$ and $\varepsilon_{\text{err}} = 2 U_{C}  \epsilon_{\text{err}} \psi^*$. The proof is complete.
\end{proof}

\begin{lemma}[Stepwise Energy $\ell_{\infty}$-Difference]
\label{lm:sed}
    \begin{align}
        \mathbb{E}_{\nu^*}[\lVert\tau_{t+1}^{-1} f_{\theta_{t+1}}(s,\cdot) - \tau_{t}^{-1}f_{\theta_{t}}(s,\cdot)\rVert_{\infty}^2] \le 2\varepsilon'_t + 2 U_{C}^2 M,
    \end{align}
    where $\varepsilon'_t = |\mathcal{A}| \cdot C_{\infty} \tau_{t+1}^{-2} \epsilon_{t+1}$ and $M = 4\mathbb{E}_{\nu^*}[\max_{a} (Q_{\omega_0}(s, a))^2] + 4 R_f^2$.
\end{lemma}
\begin{remarkapp}
    As described in Remark \ref{remark:varepislon}, $\epsilon_{t+1}$ can be sufficiently small due to Lemma \ref{lm:PI_error}. Similarly, $\varepsilon_t'$ can also be made arbitrarily small.
\end{remarkapp}

\begin{proof}[Proof of Lemma \ref{lm:sed}]
We first find an explicit bound for $\lVert \tau_{t+1}^{-1} f_{\theta_{t+1}}(s,\cdot) - \tau_{t}^{-1}f_{\theta_{t}}(s,\cdot)\rVert_{\infty}^2$. 
Note that
\begin{align}
\label{lm:sed:eq1}
    \lVert \tau_{t+1}^{-1} f_{\theta_{t+1}}(s,\cdot) - \tau_{t}^{-1}f_{\theta_{t}}(s,\cdot)\rVert_{\infty}^2 &\le 2 \lVert\tau_{t+1}^{-1} f_{\theta_{t+1}}(s,\cdot) - \tau_{t}^{-1}f_{\theta_{t}}(s,\cdot) - C_t(s, \cdot)\circ A_{\omega_t}(s,\cdot)\rVert_{\infty}^2 \\
    &\qquad + 2 \lVert C_t(s, \cdot)\circ A_{\omega_t}(s,\cdot)\rVert_{\infty}^{2}. \nonumber
\end{align}
Next, we consider the expectation of (\ref{lm:sed:eq1}) over $\nu^*$:
For the first term in (\ref{lm:sed:eq1}), we have
\begin{align}
    \mathbb{E}_{\nu^*}&[\lVert\tau_{t+1}^{-1} f_{\theta_{t+1}}(s,\cdot) - \tau_{t}^{-1}f_{\theta_{t}}(s,\cdot) - C_t(s, \cdot)\circ A_{\omega_t}(s,\cdot)\rVert_{\infty}^2] \\
    &= \int_{\mathcal{S}} \lVert\tau_{t+1}^{-1} f_{\theta_{t+1}}(s,\cdot) - \tau_{t}^{-1}f_{\theta_{t}}(s,\cdot) - C_t(s, \cdot)\circ A_{\omega_t}(s,\cdot)\rVert_{\infty}^2 \nu^*(s) ds \\
    &= \int_{\mathcal{S} \times \mathcal{A}} \frac{1}{\pi_0(a|s) }\cdot (\tau_{t+1}^{-1} f_{\theta_{t+1}}(s,a) - \tau_{t}^{-1}f_{\theta_{t}}(s,a) - C_t(s, a)\cdot A_{\omega_t}(s,a))^2 \frac{\nu^*(s)}{\nu_t(s)} d\tilde{\sigma_t}(s, a) \\
    &< |\mathcal{A}| \cdot C_{\infty} \tau_{t+1}^{-2} \epsilon_{t+1},\label{lm:sed:proof eq2}
\end{align}
where (\ref{lm:sed:proof eq2}) holds by the condition in (\ref{lm:ep:eq1}), the definition of the concentrability coefficient, and the fact that $\pi_0$ is a uniform policy.
Furthermore, we bound $ \mathbb{E}_{\nu^*}[\lVert C_t(s, \cdot)\circ A_{\omega_t}(s,\cdot)\rVert_{\infty}^{2}]$, we have
\begin{align}
    \mathbb{E}_{\nu^*}[\lVert C_t(s, \cdot)\circ A_{\omega_t}(s,\cdot)\rVert_{\infty}^{2}] &\le U_{C}^2 \cdot \mathbb{E}_{\nu^*}[\lVert A_{\omega_t}(s,\cdot)\rVert_{\infty}^{2}] \\
    &= U_{C}^2 \cdot \mathbb{E}_{\nu^*}[\lVert Q_{\omega_t}(s,\cdot) - \sum_{a} \pi_{\theta_t}(a|s) Q_{\omega_t}(s,a)\rVert_{\infty}^{2}] \\
    &= U_{C}^2 \cdot \mathbb{E}_{\nu^*}[\lVert Q_{\omega_t}(s,\cdot) - \mathbb{E}_{a \sim \pi_{\theta_t}}[ Q_{\omega_t}(s,a)]\rVert_{\infty}^{2}] \\
    &\le 2 U_C(T)^2 \mathbb{E}_{\nu^*}[\lVert Q_{\omega_t}(s,\cdot)\rVert_{\infty}^2] + 2U_C(T)^2\mathbb{E}_{\nu^*}[\mathbb{E}_{a \sim \pi_{\theta_t}}[ (Q_{\omega_t}(s,a))]^2] \\
    &\le 2 U_C(T)^2 \mathbb{E}_{\nu^*}[\lVert Q_{\omega_t}(s,\cdot)\rVert_{\infty}^2] + 2 U_C(T)^2 \mathbb{E}_{\nu^*}[\lVert Q_{\omega_t}(s,\cdot)\rVert_{\infty}^2] \label{eq:Q_t}\\
    &\le U_{C}^2 \cdot 4\mathbb{E}_{\nu^*}[\lVert Q_{\omega_t}(s,\cdot)\rVert_{\infty}^2] \\
    &\le 4 U_{C}^2 \cdot [\mathbb{E}_{\nu^*}[\max_{a} (Q_{\omega_0}(s, a))^2] + R_f^2],\label{lm:sed:proof eq3}
\end{align}
where (\ref{eq:Q_t}) holds by using Jensen's inequality and leveraging the $\ell_{\infty}$-norm instead of the expectation $\mathbb{E}_{a\sim\pi_{\theta_t}}[\cdot]$, and the last inequality in (\ref{lm:sed:proof eq3}) holds by the 1-Lipschitz property of neural networks with respect to the weights. 
By setting $\varepsilon_t' = |\mathcal{A}| \cdot C_{\infty} \tau_{t+1}^{-2} \epsilon_{t+1}$ and $M = 4\mathbb{E}_{\nu^*}[\max_{a} (Q_{\omega_0}(s, a))^2] + 4R_f^2$, we complete the proof of Lemma \ref{lm:sed}.
\end{proof}

\begin{lemma}[Stepwise KL Difference]
\label{lm:OSD}
The KL difference is as follows,
\begin{align}
    &\text{KL}(\pi^*(\cdot|s) \rVert \pi_{\theta_{t+1}}(\cdot|s)) - \text{KL}(\pi^*(\cdot|s) \rVert \pi_{\theta_t}(\cdot|s)) \\
    &\le \langle \log\pi_{\theta_{t+1}}(\cdot|s) - \log  {\pi}_{t+1}(\cdot|s), \pi_{\theta_t}(\cdot|s) -  \pi^*(\cdot|s) \rangle  - \langle \bar{C_t}(s, \cdot) \circ A^{\pi_{\theta_t}}(s, \cdot), \pi^*(\cdot|s) - \pi_{\theta_t}(\cdot|s) \rangle \nonumber \\
    &\qquad  - \frac{1}{2} \lVert\pi_{\theta_{t+1}}(\cdot|s) - \pi_{\theta_t}(\cdot|s)\rVert_1^2 - \langle \log\pi_{\theta_{t+1}}(\cdot|s) - \log  \pi_{\theta_t}(\cdot|s), \pi_{\theta_t}(\cdot|s) - \pi_{\theta_{t+1}}(\cdot|s) \rangle
\end{align}
\end{lemma}
\begin{proof}[Proof of Lemma \ref{lm:OSD}]
We directly expand the one-step KL divergence difference as
\begin{align}
    \text{KL}(\pi^*(&\cdot|s) \rVert \pi_{\theta_{t+1}}(\cdot|s)) - \text{KL}(\pi^*(\cdot|s) \rVert \pi_{\theta_t}(\cdot|s)) = \left\langle \log \frac{\pi_{\theta_{t}}(\cdot|s)}{\pi_{\theta_{t+1}}(\cdot|s)}, \pi^*(\cdot|s) \right\rangle \\
    &= \left\langle \log \frac{\pi_{\theta_{t+1}}(\cdot|s)}{\pi_{\theta_t}(\cdot|s)}, \pi_{\theta_{t+1}}(\cdot|s) - \pi^*(\cdot|s)  \right\rangle - \text{KL}(\pi_{\theta_{t+1}}(\cdot|s) \rVert \pi_{\theta_t}(\cdot|s)) \\
    &= \left\langle \log \frac{\pi_{\theta_{t+1}}(\cdot|s)}{\pi_{\theta_t}(\cdot|s)} - \bar{C_t}(s, \cdot) \circ A^{\pi_{\theta_t}}(s, \cdot), \pi_{\theta_{t}}(\cdot|s) - \pi^*(\cdot|s) \right\rangle \\
    &\qquad - \langle \bar{C_t}(s, \cdot) \circ A^{\pi_{\theta_t}}(s, \cdot), \pi^*(\cdot|s) - \pi_{\theta_{t}}(\cdot|s) \rangle - \text{KL}(\pi_{\theta_{t+1}}(\cdot|s) \rVert \pi_{\theta_t}(\cdot|s)) \nonumber \\
    &\qquad - \left\langle \log \frac{\pi_{\theta_{t+1}}(\cdot|s)}{\pi_{\theta_t}(\cdot|s)}, \pi_{\theta_{t}}(\cdot|s) - \pi_{\theta_{t+1}}(\cdot|s) \right\rangle. \nonumber
\end{align}
Then, by Pinsker's inequality, we have
\begin{align}
    \text{KL}(\pi^*(&\cdot|s) \rVert \pi_{\theta_{t+1}}(\cdot|s)) - \text{KL}(\pi^*(\cdot|s) \rVert \pi_{\theta_t}(\cdot|s)) \\
    & = \left\langle \log \frac{\pi_{\theta_{t+1}}(\cdot|s)}{\pi_{\theta_t}(\cdot|s)} - \bar{C_t}(s, \cdot) \circ A^{\pi_{\theta_t}}(s, \cdot), \pi_{\theta_{t}}(\cdot|s) - \pi^*(\cdot|s) \right\rangle \\
    &\qquad - \langle \bar{C_t}(s, \cdot) \circ A^{\pi_{\theta_t}}(s, \cdot), \pi^*(\cdot|s) - \pi_{\theta_{t}}(\cdot|s) \rangle - \text{KL}(\pi_{\theta_{t+1}}(\cdot|s) \rVert \pi_{\theta_t}(\cdot|s)) \nonumber \\
    &\qquad - \left\langle \log \frac{\pi_{\theta_{t+1}}(\cdot|s)}{\pi_{\theta_t}(\cdot|s)}, \pi_{\theta_{t}}(\cdot|s) - \pi_{\theta_{t+1}}(\cdot|s) \right\rangle \nonumber \\
    &\le \langle \log\pi_{\theta_{t+1}}(\cdot|s) - \log  \pi_{\theta_t}(\cdot|s) - \bar{C_t}(s, \cdot) \circ A^{\pi_{\theta_t}}(s, \cdot), \pi_{\theta_t}(\cdot|s) -  \pi^*(\cdot|s) \rangle \label{eq:OSD proof 1}\\
    &\qquad - \langle \bar{C_t}(s, \cdot) \circ A^{\pi_{\theta_t}}(s, \cdot), \pi^*(\cdot|s) - \pi_{\theta_t}(\cdot|s) \rangle - \frac{1}{2} \lVert\pi_{\theta_{t+1}}(\cdot|s) - \pi_{\theta_t}(\cdot|s)\rVert_1^2 \nonumber \\
    &\qquad - \langle \log\pi_{\theta_{t+1}}(\cdot|s) - \log  \pi_{\theta_t}(\cdot|s), \pi_{\theta_t}(\cdot|s) - \pi_{\theta_{t+1}}(\cdot|s) \rangle. \nonumber
\end{align}
Finally, by Proposition \ref{pp:PI}, we have $\log \pi_{t+1}(\cdot|s) = \log  \pi_{\theta_t}(\cdot|s) + \bar{C_t}(s, \cdot) \circ A^{\pi_{\theta_t}}(s, \cdot)$ and then apply this to the first term in (\ref{eq:OSD proof 1}). The proof is complete.
\end{proof}
\begin{lemma}[Performance Difference Using Advantage]
\label{lm:PDL}
Recall that $\mathcal{L}(\pi) = \mathbb{E}_{\nu^*}[V^{\pi}(s)]$. We have 
\begin{align}
    \mathcal{L}(\pi^*) - \mathcal{L}(\pi) = (1 - \gamma) ^ {-1} \cdot \mathbb{E}_{\nu^*}[\langle A^{\pi}(s, \cdot), \pi^*(\cdot|s) - \pi(\cdot|s)\rangle].
\end{align}
\end{lemma}
Before proving Lemma \ref{lm:PDL}, we first state the following property.
\begin{lemma}[\citep{liu2019neural}, Lemma 5.1]
\label{lm:PDQ}
\begin{align}
    \mathcal{L}(\pi^*) - \mathcal{L}(\pi) = (1 - \gamma) ^ {-1} \cdot \mathbb{E}_{\nu^*}[\langle Q^{\pi}(s, \cdot), \pi^*(\cdot|s) - \pi(\cdot|s)\rangle].
\end{align}
\end{lemma}
\begin{proof}[Proof of Lemma \ref{lm:PDL}]
As the value function $V^{\pi}(\cdot)$ is state-dependent, we have
\begin{align}
    \mathbb{E}_{\nu^*}[\langle V^{\pi}(s), \pi^*(\cdot|s) - \pi(\cdot|s)\rangle] &= \mathbb{E}_{\nu^*}[V^{\pi}(s) \cdot \sum_{a \in \mathcal{A}} (\pi^*(a|s) - \pi(a|s))] \\
    &= \mathbb{E}_{\nu^*}\left[V^{\pi}(s) \cdot \left(\sum_{a \in \mathcal{A}} \pi^*(a|s) - \sum_{a \in \mathcal{A}}\pi(a|s)\right)\right] = 0.\label{eq:performance difference 1}
\end{align}
Therefore, by (\ref{eq:performance difference 1}) and Lemma \ref{lm:PDQ}, we have
\begin{align}
    \mathcal{L}(\pi^*) - \mathcal{L}(\pi)  &= (1 - \gamma) ^ {-1} \cdot\mathbb{E}_{\nu^*}[\langle Q^{\pi}(s, \cdot) - V^{\pi}(s), \pi^*(\cdot|s) - \pi(\cdot|s)\rangle] \\
    &= (1 - \gamma) ^ {-1} \cdot \mathbb{E}_{\nu^*}[\langle A^{\pi}(s, \cdot), \pi^*(\cdot|s) - \pi(\cdot|s)\rangle].
\end{align}
\end{proof}
\subsection{Proof of Theorem \ref{thm:main}}

By taking expectation of the KL difference in Lemma \ref{lm:OSD} over $\nu^*$, we obtain
\begin{align}
    &\mathbb{E}_{\nu^*}[\text{KL}(\pi^*(\cdot|s) || \pi_{\theta_{t+1}}(\cdot|s)) - \text{KL}(\pi^*(\cdot|s) || \pi_{\theta_t}(\cdot|s))] \\ 
    &\le \varepsilon_t + \varepsilon_{\text{err}}
    - \mathbb{E}_{\nu^*}[\langle \bar{C}_t(s,\cdot) \circ A^{\pi_{\theta_t}}(s, \cdot), \pi^*(\cdot | s) - \pi_{\theta_t}(\cdot|s)\rangle] - \frac{1}{2} \mathbb{E}_{\nu^*}[\lVert\pi_{\theta_{t+1}}(\cdot|s) - \pi_{\theta_{t}}(\cdot|s)\rVert_{1}^{2}] \nonumber \\
    &\quad -\mathbb{E}_{\nu^*}[\langle \tau_{t+1}^{-1} f_{\theta_{t+1}}(s, \cdot) - \tau_{t}^{-1} f_{\theta_{t}}(s, \cdot), \pi_{\theta_{t}}(\cdot|s) - \pi_{\theta_{t+1}}(\cdot|s)\rangle] \\
    &\le \varepsilon_t + \varepsilon_{\text{err}}
    - \mathbb{E}_{\nu^*}[\langle \bar{C}_t(s,\cdot) \circ A^{\pi_{\theta_t}}(s, \cdot), \pi^*(\cdot | s) - \pi_{\theta_t}(\cdot|s)\rangle] - \frac{1}{2} \mathbb{E}_{\nu^*}[\lVert\pi_{\theta_{t+1}}(\cdot|s) - \pi_{\theta_{t}}(\cdot|s)\rVert_{1}^{2}] \nonumber \\
    &\quad + \mathbb{E}_{\nu^*}[\lVert\tau_{t+1}^{-1} f_{\theta_{t+1}}(s, \cdot) - \tau_{t}^{-1} f_{\theta_{t}}(s, \cdot)\rVert_{\infty} \cdot \lVert\pi_{\theta_{t+1}}(\cdot|s) - \pi_{\theta_{t}}(\cdot|s)\rVert_{1}] \\
    &\le \varepsilon_t + \varepsilon_{\text{err}}
    - \mathbb{E}_{\nu^*}[\langle \bar{C}_t(s,\cdot) \circ A^{\pi_{\theta_t}}(s, \cdot), \pi^*(\cdot |s) - \pi_{\theta}(\cdot|s) \rangle] \nonumber \\
    &\quad + \frac{1}{2} \mathbb{E}_{\nu^*}[\lVert\tau_{t+1}^{-1} f_{\theta_{t+1}}(s, \cdot) - \tau_{t}^{-1} f_{\theta_{t}}(s, \cdot)\rVert_{\infty}^{2}],
\end{align}
where the first inequality follows from Lemma \ref{lm:OSD} and Lemma \ref{lm:ep}, the second inequality holds by the Hölder's inequality, and the last inequality holds by the fact that $2xy - x^2 \le y^2$ and merging the last two terms. Then, by Lemma \ref{lm:sed} and rearranging the terms, we obtain that 
\begin{align}
\label{proof:eq:each_t}
    \mathbb{E}_{\nu^*}&[\langle \bar{C}_t(s,\cdot) \circ A^{\pi_{\theta_t}}(s, \cdot), \pi^*(\cdot | s) - \pi_{\theta_t}(\cdot|s)\rangle] \nonumber \\
    &\le \mathbb{E}_{\nu^*}[\text{KL}(\pi^*(\cdot|s) \rVert \pi_{\theta_{t}}(\cdot|s)) - \text{KL}(\pi^*(\cdot|s) \rVert \pi_{\theta_{t+1}}(\cdot|s))] + \varepsilon_t + \varepsilon_{\text{err}} + \varepsilon_t' + U_{C}^2 M.
\end{align}
By the first condition of (\ref{suff:1}), we have $L_{C}  \mathbb{E}_{\nu^*}[\langle A^{\pi_{\theta_t}}(s, \cdot), \pi^*(\cdot|s) - \pi_{\theta_t}(\cdot|s)\rangle] \le \mathbb{E}_{\nu^*}[\langle \bar{C}_t(s,\cdot) \circ A^{\pi_{\theta_t}}(s, \cdot), \pi^*(\cdot | s) - \pi_{\theta_t}(\cdot|s)\rangle]$. 
By obtaining the performance difference via Lemma \ref{lm:PDL}, we have
\begin{align}
\label{app:proof:thm:main:eq_t}
    &(1 - \gamma) L_{C} (\mathcal{L}(\pi^*) - \mathcal{L}(\pi_{\theta_t}))\nonumber \\
    &\le \mathbb{E}_{\nu^*}[\text{KL}(\pi^*(\cdot|s) \rVert \pi_{\theta_{t}}(\cdot|s)) - \text{KL}(\pi^*(\cdot|s) \rVert \pi_{\theta_{t+1}}(\cdot|s))] + \varepsilon_t + \varepsilon_{\text{err}} + \varepsilon_t' + U_{C}^2 M.
\end{align}
Then, by taking the telescoping sum of (\ref{app:proof:thm:main:eq_t}) from $t = 0$ to $T-1$, we have
\begin{align}
    &(1 - \gamma) L_{C} \sum_{t=0}^{T-1} (\mathcal{L}(\pi^*) - \mathcal{L}(\pi_{\theta_t})) &\\
    &\le \mathbb{E}_{\nu^*}[\text{KL}(\pi^*(\cdot|s) \rVert \pi_{\theta_{0}}(\cdot|s))] - \mathbb{E}_{\nu^*}[\text{KL}(\pi^*(\cdot|s) \rVert \pi_{\theta_{T}}(\cdot|s))] + \sum_{t=0}^{T-1} (\varepsilon_t + \varepsilon_{\text{err}} + \varepsilon_t') + T U_{C}^2 M. &
\end{align}
By the facts that (i) $\mathbb{E}_{\nu^*}[\text{KL}(\pi^*(\cdot|s) \rVert \pi_{\theta_{0}}(\cdot|s))] \le \log |\mathcal{A}|$, (ii) KL divergence is nonnegative, (iii) $\sum_{t=0}^{T-1} (\mathcal{L}(\pi^*) - \mathcal{L}(\pi_{\theta_t})) \ge T \cdot \min_{0\le t \le T} \{\mathcal{L}(\pi^*) - \mathcal{L}(\pi_{\theta_t})\}$, we have
\begin{align}
\label{proof:thm:main:pre_result}
    \min_{0\le t \le T} \{\mathcal{L}(\pi^*) - \mathcal{L}(\pi_{\theta_t})\} \le \frac{\log |\mathcal{A}| + \sum_{t=0}^{T-1} (\varepsilon_t + \varepsilon_t') + T (\varepsilon_{\text{err}} + M U_{C}^2)}{T L_{C} (1 - \gamma)}.
\end{align}
Since we have $\varepsilon_{\text{err}} = 2 U_{C} \epsilon_{\text{err}} \psi^*$ and the condition of (\ref{suff:2}), we know that if we set $\epsilon_{\text{err}} = U_C(T)$ and $T$ to be sufficiently large, {$\epsilon_{\text{err}}$ shall be sufficiently small and hence satisfy the condition required by (\ref{lm4:eq}).} 
Thus, by plugging $\epsilon_{\text{err}} = U_C(T)$ into (\ref{proof:thm:main:pre_result}), we have $\varepsilon_{\text{err}} = 2 U_C(T)^2 \psi^*$ and $\varepsilon_t = \epsilon_{t+1}^{1/2} C_{\infty} \tau_{t+1}^{-1} \phi^* + \epsilon_t'^{1/2} U_{C} \left[\left[(2 / U_C(T)^2)(M + (A^{\pi_{\theta_t}}_{\max})^2 - \epsilon_t'/2)\right]\right]^{1/2} \psi^* = \epsilon_{t+1}^{1/2} C_{\infty} \tau_{t+1}^{-1} \phi^* + \epsilon_t'^{1/2} U_{C} Y^{1/2} \psi^*$, where $Y = 2M + 2(R_{\max} / (1 - \gamma))^2 - \epsilon_t' \le 2M + 2(R_{\max} / (1 - \gamma))^2$. Finally, we have
\begin{align}
    \min_{0\le t \le T} \{\mathcal{L}(\pi^*) - \mathcal{L}(\pi_{\theta_t})\} \le \frac{\log |\mathcal{A}| + \sum_{t=0}^{T-1} (\varepsilon_t + \varepsilon_t') + T U_{C}^2 (2 \psi^* + M)}{T L_{C} (1 - \gamma)}.\label{eq:rate proof 1}
\end{align}
By the condition (\ref{suff:2}), $U_C(T)^2$ can always cancel out $T$ in the numerator of (\ref{eq:rate proof 1}).
Moreover, in the denominator of (\ref{eq:rate proof 1}), $L_C(T) = \omega(T^{-1})$ is large enough to attain convergence, and we complete the proof.
\hfill \qedsymbol

\begin{remarkapp}
    {As mentioned in Remark \ref{remark:choice}, the choices of $\eta$ and $\{\tau_t\}$ would affect the convergence rate and need to be configured properly for Neural PPO-Clip with different classifiers. As will be shown in Appendix \ref{app:add:cor}, this fact can be further explained through the bounds $U_C(T)$ and $L_C(T)$ obtained in (\ref{eq:U_C, L_C for PPO-Clip}) and (\ref{eq:U_C, L_C for NeuralHPO-sub}).}
\end{remarkapp}

\section{Additional Corollaries and Proofs}
\label{app:add:cor}

\subsection{Proof of Corollary \ref{cor:PPO-Clip}}
\label{proof:Cor:PPO-Clip}
For ease of exposition, we restate the corollary as follows.
\begin{corollarystar}[Global Convergence of {Neural PPO-Clip} with Convergence Rate]
    Consider Neural PPO-Clip with the standard PPO-Clip classifier $\rho_{s, a}(\theta) - 1$ and the objective function $L^{(t)}(\theta)$ in each iteration $t$ as 
    \begin{align}
         \mathbb{E}_{\nu_t}[\langle \pi_{\theta_t}(\cdot|s), |A^{\pi_{\theta_t}}(s, \cdot)| \circ \ell (\sgn(A^{\pi_{\theta_t}}(s, \cdot)), \rho_{s, \cdot}(\theta) - 1, \epsilon) \rangle].
    \end{align}
    (i) If we specify the EMDA step size $\eta = T^{-\alpha}$ where $\alpha \in [1/2, 1)$ and the temperature parameter $\tau_t = T^{\alpha} / (Kt)$. Recall that $K$ is the maximum number of EMDA iterations. Let the neural networks' widths $m_f = \Omega(R_f^{10} \phi^{*8} K^8 C_{\infty}^8 T^{12} + R_f^{10} K^8 T^8 C_{\infty}^4 |\mathcal{A}|^4)$, $m_Q = \Omega(R_Q^{10} \psi^{*8} Y^4 T^8)$, and the SGD and TD updates $T_{\text{upd}} = \Omega(R_f^4 \phi^{*4} K^4 C_{\infty}^4 T^6 + R_Q^4 \psi^{*4} Y^2 T^4 + R_f^4 T^4 K^4 C_{\infty}^2 |\mathcal{A}|^2)$, we have
    \begin{align}
        \min_{0\le t \le T} \{\mathcal{L}(\pi^*) - \mathcal{L}(\pi_{\theta_t})\} \le \frac{\log |\mathcal{A}| + K^2 (2 \psi^* + M) + O(1)}{T^{\alpha} (1 - \gamma)},
    \end{align}
     {Hence, Neural PPO-Clip has $O(T^{-\alpha})$ convergence rate.}
    (ii) Furthermore, let the $\alpha = 1/2$, we obtain the fastest convergence rate, which is $O(1 / \sqrt{T})$.
\end{corollarystar}
\begin{proof}[Proof of Corollary \ref{cor:PPO-Clip}]
We find the lower and upper bounds $L_C(T), U_C(T)$ for PPO-Clip. We first consider the derivative $g_{s, a}$ of the objective with the true advantage function $A^{\pi_{\theta_t}}$. 
\begin{align}
    g_{s, a} = \left.\frac{\partial L(\theta)}{\partial \theta}\right|_{\theta = \tilde{\theta}_{s, a}} = -A^{\pi_{\theta_t}}(s, a) \cdot \mathds{1}\left\{\left(\frac{\tilde{\theta}_{s, a}}{\pi_{\theta_t}(a|s)} - 1\right) \cdot \sgn (A^{\pi_{\theta_t}}(s, a)) < \epsilon \right\}.
\end{align}

Then, we check the sufficient conditions (\ref{suff:1}) and (\ref{suff:2}). Recall that $K$ is the maximum number of EMDA iteration for each $t$. We sum up the gradients with $\eta$ and rearrange the terms into $\bar{C_t}(s, a)$. Then, we have the upper bound as
\begin{align}
    \bar{C_t}(s ,a) \cdot |A^{\pi_{\theta_t}}(s, a)| \le \left[\sum_{k=0}^{K^{(t)}-1} \eta\right] \cdot |A^{\pi_{\theta_t}}(s, a)| \le K \eta \cdot |A^{\pi_{\theta_t}}(s, a)|.
\end{align}
Regarding the lower bound, as we know that under PPO-Clip, the first step of EMDA shall always make an update, i.e., it will never be clipped, and hence we have
\begin{align}
    \eta \cdot |A^{\pi_{\theta_t}}(s, a)| \le \bar{C_t}(s ,a) \cdot |A^{\pi_{\theta_t}}(s, a)|.
\end{align}
Lastly, by setting $\eta = T^{-\alpha}$ and {selecting the temperature as $\tau_{t} = T^{\alpha} / (K t)$ to satisfy the condition $\tau_{t+1}^2 (U_{C}^2 + \tau_{t}^{-2}) \le 1$ that we use in (\ref{eq:PI_error 1})}, we obtain 
\begin{align}
    \omega(T^{-1}) = T^{-1/2} |A^{\pi_{\theta_t}}(s, a)| \le \bar{C_t}(s, a) \cdot |A^{\pi_{\theta_t}}(s, a)| \le K T^{-1/2} \cdot |A^{\pi_{\theta_t}}(s, a)| = O(T^{-1/2}).\label{eq:U_C, L_C for PPO-Clip}
\end{align}
We have checked the sufficient conditions of Theorem \ref{thm:main}. Thus, we obtain,
\begin{align}
    \min_{0\le t \le T} \{\mathcal{L}(\pi^*) - \mathcal{L}(\pi_{\theta_t})\} \le \frac{\log |\mathcal{A}| + \sum_{t=0}^{T-1} (\varepsilon_t + \varepsilon_t') + K^2 (2 \psi^* + M)}{T^{\alpha} (1 - \gamma)}.
\end{align}
Then, we show the minimum widths and the number of iterations of SGD and TD updates to attain convergence. We must force the summation of errors $\varepsilon_t, \varepsilon_t'$ to be $O(1)$. By Lemma \ref{lm:PE_error}, \ref{lm:PI_error}, where $\epsilon_{t+1} = O(R_f^2 T_{\text{upd}}^{-1/2} + R_f^{5/2} m_f^{-1/4} + R_f^3 m_f^{-1/2}), \epsilon_t' = O(R_Q^2 T_{\text{upd}}^{-1/2} + R_Q^{5/2} m_Q^{-1/4} + R_Q^3 m_Q^{-1/2})$, we have
\begin{align}
    C_{\infty} \tau_{t+1}^{-1} \phi^* \epsilon_{t+1}^{1/2} = &O(C_{\infty} Kt T^{-1/2} \phi^* \cdot (R_f^2 T_{\text{upd}}^{-1/2} + R_f^{5/2}m_f^{-1/4})^{1/2}), \\
    Y^{1/2} \psi^* \epsilon_t'^{1/2} = &O(Y^{1/2} \psi^* (R_Q^2 T_{\text{upd}}^{-1/2} + R_Q^{5/2} m_Q^{-1/4})^{1/2}) \\
    |\mathcal{A}|C_{\infty} \tau_{t+1}^2 \epsilon_{t+1} = &O(|\mathcal{A}|C_{\infty} K^2 t^2 T^{-1} (R_f^2 T_{\text{upd}}^{-1/2} + R_f^{5/2}m_f^{-1/4})),
\end{align}
when $m_f = \Omega(R_f^2)$ and $m_Q = \Omega(R_Q^2)$. Then, by taking $m_f = \Omega(R_f^{10} \phi^{*8} K^8 C_{\infty}^8 T^{12}), m_Q = \Omega(R_Q^{10} \psi^{*8} Y^4 T^8),$ and $T_{\text{upd}} = \Omega(R_f^4 \phi^{*4} K^4 C_{\infty}^4 T^6 + R_Q^4 \psi^{*4} Y^2 T^4)$, we have
\begin{equation}
\label{eq:cor1:vart}
\varepsilon_t = C_{\infty} \tau_{t+1}^{-1} \phi^* \epsilon_{t+1}^{1/2} + Y^{1/2} \psi^* \epsilon_t'^{1/2} = O(T^{-1}).
\end{equation}
Moreover, we further put $m_f = \Omega(R_f^{10} T^8 K^8 C_{\infty}^4 |\mathcal{A}|^4)$ and $T_{\text{upd}} = \Omega(R_f^4 T^4 K^4 C_{\infty}^2 |\mathcal{A}|^2)$, we have
\begin{equation}
\label{eq:cor1:vartprime}
    \varepsilon_t' = |\mathcal{A}|C_{\infty} \tau_{t+1}^2 \epsilon_{t+1} = O(T^{-1}).
\end{equation}
Last, we add up the lower bound of each term of $m_f, m_Q$, and $T_{\text{upd}}$, and then sum the errors in (\ref{eq:cor1:vart}) and (\ref{eq:cor1:vartprime}) for all $t$ from $0$ to $T-1$, we obtain
\begin{equation}
    \min_{0\le t \le T} \{\mathcal{L}(\pi^*) - \mathcal{L}(\pi_{\theta_t})\} \le \frac{\log |\mathcal{A}| + K^2 (2 \psi^* + M) + O(1)}{T^{\alpha} (1 - \gamma)},
\end{equation}
which completes the proof and obtains the $O(T^{-\alpha})$ convergence rate.

Furthermore, if we set $\alpha = 1/2$, $\eta$ will be $1 / \sqrt{T}$, and we plug into the result above, we have the $O(1 / \sqrt{T})$ convergence rate.
\end{proof}

\subsection{Convergence Rate of Neural PPO-Clip With an Alternative Classifier}

\begin{corollary}[Global Convergence of {Neural PPO-Clip} with subtraction classifier with Convergence Rate]
\label{cor:sub}
    Consider Neural PPO-Clip with the subtraction classifier $\pi_{\theta}(a|s) - \pi_{\theta_t}(a|s)$ (\textit{termed Neural PPO-Clip-sub}) and the objective function $L^{(t)}(\theta)$ in each iteration $t$ as 
    \begin{align}
         \mathbb{E}_{\sigma_t}[|A^{\pi_{\theta_t}}(s, a)| \cdot \ell (\sgn(A^{\pi_{\theta_t}}(s, a)), \pi_{\theta}(a|s) - \pi_{\theta_t}(a|s), \epsilon)].
    \end{align}
    We specify the EMDA step size $\eta = 1 / \sqrt{T}$ and the temperature parameter $\tau_t = \sqrt{T} / (Kt)$. Recall that $K$ is the maximum number of EMDA iterations. Let the neural networks' widths $m_f = \Omega(R_f^{10} \phi^{*8} K^8 C_{\infty}^8 T^{12} + R_f^{10} K^8 T^8 C_{\infty}^4 |\mathcal{A}|^4)$, $m_Q = \Omega(R_Q^{10} \psi^{*8} Y^4 T^8)$, and the SGD and TD updates $T_{\text{upd}} = \Omega(R_f^4 \phi^{*4} K^4 C_{\infty}^4 T^6 + R_Q^4 \psi^{*4} Y^2 T^4 + R_f^4 T^4 K^4 C_{\infty}^2 |\mathcal{A}|^2)$, we have
    \begin{align}
        \min_{0\le t \le T} \{\mathcal{L}(\pi^*) - \mathcal{L}(\pi_{\theta_t})\} \le \frac{\log |\mathcal{A}| + K^2 (2 \psi^* + M) + O(1)}{\sqrt{T} (1 - \gamma)},
    \end{align}
     Hence, we provide the $O(1 / \sqrt{T})$ convergence rate of Neural PPO-Clip-sub.
\end{corollary}

\begin{proof}[Proof of Corollary \ref{cor:sub}]
    Similar to Corollary \ref{cor:PPO-Clip}, we derive the gradient of our objective with the true advantage function $A^{\pi_{\theta_t}}(s, a)$. Specifically, we have
    \begin{align}
        g_{s, a} = \left.\frac{\partial L(\theta)}{\partial \theta}\right|_{\theta = \tilde{\theta}_{s, a}} = -A^{\pi_{\theta_t}}(s, a) \cdot \mathds{1}\left\{\left(\tilde{\theta}_{s, a} - \pi_{\theta_t}(a|s)\right) \cdot \sgn (A^{\pi_{\theta_t}}(s, a)) < \epsilon \right\}.
    \end{align}
    Thus, similar to \ref{proof:Cor:PPO-Clip}, we have
    \begin{align}
        \eta \cdot |A^{\pi_{\theta_t}}(s, a)| \le C_t(s ,a) \cdot |A^{\pi_{\theta_t}}(s, a)| \le K \eta \cdot |A^{\pi_{\theta_t}}(s, a)|.
    \end{align}
    We also set $\eta = 1 / \sqrt{T}$ and {pick $\tau_{t} = \sqrt{T} / (K t)$ to satisfy the condition $\tau_{t+1}^2 (U_{C}^2 + \tau_{t}^{-2}) \le 1$ that we use in (\ref{eq:PI_error 1})}. Accordingly, we obtain 
    \begin{align}
        \omega(T^{-1}) = T^{-1/2} |A^{\pi_{\theta_t}}(s, a)| \le C_t(s, a) \cdot |A^{\pi_{\theta_t}}(s, a)| \le K T^{-1/2} \cdot |A^{\pi_{\theta_t}}(s, a)| = O(T^{-1/2}).\label{eq:U_C, L_C for NeuralHPO-sub}
    \end{align}
    We have checked the sufficient condition of Theorem \ref{thm:main}. Therefore, by plugging in $L_C(T)$ and $U_C(T)$, we obtain
    \begin{align}
        \min_{0\le t \le T} \{\mathcal{L}(\pi^*) - \mathcal{L}(\pi_{\theta_t})\} \le \frac{\log |\mathcal{A}| + \sum_{t=0}^{T-1} (\varepsilon_t + \varepsilon_t') + K^2 (2 \psi^* + M )}{\sqrt{T} (1 - \gamma)}.
    \end{align}
    Similar to the proof of Corollary \ref{proof:Cor:PPO-Clip}, we set the same minimum widths and number of iterations to attain convergence, which directly implies 
    \begin{equation}
        \min_{0\le t \le T} \{\mathcal{L}(\pi^*) - \mathcal{L}(\pi_{\theta_t})\} \le \frac{\log |\mathcal{A}| + K^2 (2 \psi^* + M) + O(1)}{\sqrt{T} (1 - \gamma)}.
    \end{equation}
    Then, we complete the proof and obtain the $O(1/\sqrt{T})$ convergence rate of PPO-Clip with a subtraction classifier.
\end{proof}

\section{Tabular PPO-Clip and Proof}
\label{app:mini-batch thm}

\subsection{Supporting Lemmas for the Proof of Theorem~\ref{thm:mini-batch}}
\label{Supporting Lemmas for mini-batch}

For completeness, we state the state-wise policy improvement Lemma in \citep{kakade2002} and provide the proof.
\begin{lemma}
\label{prop:second}
    Given policies $\pi_{1}$ and $\pi_{2}$, $V^{\pi_{1}}(s)\geq V^{\pi_{2}}(s)$ for all $s\in \cS$ if the following holds:
    \begin{equation}
        (\pi_{1}(a|s)-\pi_{2}(a|s))A^{\pi_{2}}(s,a)\geq0,\ \forall(s,a)\in\mathcal{S}\times\mathcal{A}.
        \label{eq:prop 2 condition}
    \end{equation}
\end{lemma}
\begin{proof}[Proof of Lemma \ref{prop:second}]
    By the performance difference lemma \cite{kakade2002}, we have
    \begin{equation}
        V^{\pi_1}(s) - V^{\pi_2}(s) = \frac{1}{1 - \gamma} \sum_{s' \in \mathcal{S}} d^{\pi_1}_s(s') \sum_{a \in \mathcal{A}} \pi_1(a|s') A^{\pi_2}(s', a).
    \end{equation}
    Also, since we have $\sum_{a \in \mathcal{A}} \pi_2(a|s) A^{\pi_2}(s, a) = 0$ holds for any $s \in \mathcal{S}$, if $\sum_{a \in \mathcal{A}} (\pi_1(a|s) - \pi_2(a|s)) A^{\pi_2}(s, a) \ge 0$ holds for any $(s, a) \in \mathcal{S} \times \mathcal{A}$, then $\sum_{a \in \mathcal{A}} \pi_1(a|s) A^{\pi_2}(s, a) \ge 0$. Hence, we will obtain $V^{\pi_1}(s) \ge V^{\pi_2}(s)$ for all $s \in \mathcal{S}$.
\end{proof}

Notably, Lemma \ref{prop:second} offers a useful insight that policy improvement can be achieved by simply adjusting the action distribution based solely on the \textit{sign of the advantage} of the state-action pairs, regardless of their magnitude. We provide the proof in Appendix \ref{Supporting Lemmas for mini-batch}. Interestingly, one can draw an analogy between (\ref{eq:prop 2 condition}) in Lemma \ref{prop:second} and learning a linear binary classifier: (i) \textit{Features}: The state-action representation can be viewed as the feature vector of a training sample; (ii)
\textit{Labels}: The sign of $A^{\pi_2}(s,a)$ resembles a binary label; (iii) \textit{Classifiers}: $\pi_{1}(a|s)-\pi_{2}(a|s)$ serves as the 
prediction of a linear classifier.
{We provide the intuition behind using $\pi_{1}(a|s)-\pi_{2}(a|s)$ as a classifier. Let's fix $\pi_{2}$ and let $\pi_{1}$ be the improved policy. If the sign of $A^{\pi_2}(s,a) \ge 0$, which implies that the action $a$ has a positive effect on the total return, it is desired to slightly tune up the probability of acting in action $a$. Thus, the update $\pi_{1}$ must have a greater probability on action $a$ in order to obtain the sufficient condition of the state-wise policy improvement, i.e., $(\pi_{1}(a|s)-\pi_{2}(a|s))A^{\pi_{2}}(s,a) \ge 0$.}
{Next, we substantiate this insight and rethink PPO-Clip via hinge loss.}

As described in Section \ref{section:HPO}, one major component of the proof of Theorem \ref{thm:mini-batch} is the state-wise policy improvement property of PPO-Clip.
For ease of exposition, we introduce the following definition regarding the partial ordering over policies.
\begin{definition}[Partial ordering over policies]
    Let $\pi_{1}$ and $\pi_{2}$ be two policies. 
    Then, $\pi_{1}\geq\pi_{2}$, called \textit{$\pi_{1}$ improves upon $\pi_{2}$}, if and only if $V^{\pi_{1}}(s)\geq V^{\pi_{2}}(s),\ \forall s\in\mathcal{S}$. Moreover, we say $\pi_1>\pi_2$, called \textit{$\pi_{1}$ strictly improves upon $\pi_{2}$}, if and only if $\pi_{1}\geq\pi_{2}$ and there exists at least one state $s$ such that $V^{\pi_{1}}(s)> V^{\pi_{2}}(s)$.
\end{definition}

\begin{lemma}[Sufficient condition of state-wise policy improvement]
\label{prop:first}
    Given any two policies $\pi_{1}$ and $\pi_{2}$, we have ${\pi_{1}}\geq {\pi_{2}}$ if the following condition holds:
    \begin{equation}
        \sum_{a\in\mathcal{A}}\pi_{1}(a|s)A^{\pi_{2}}(s,a)\geq0,\ \forall s\in\mathcal{S}.
    \end{equation}
\end{lemma}
\begin{proof}[Proof of Lemma \ref{prop:first}]
This is the same result of the proof of Lemma \ref{prop:second}.
\end{proof}
Next, we present two critical properties that hold under PPO-Clip for every sample path.
\begin{lemma}[Strict improvement and strict positivity of policy under PPO-Clip with direct tabular parameterization]
\label{lemma:policy improvement mini-batch}
In any iteration $t$, suppose $\pi^{(t)}$ is strictly positive in all state-action pairs, i.e., $\pi^{(t)}(a\rvert s)>0$, for all $(s,a)$. 
Under PPO-Clip in Algorithm \ref{algo:HPO-AM}, $\pi^{(t+1)}$ satisfies that (i) $\pi^{(t+1)}>\pi^{(t)}$ and (ii) $\pi^{(t+1)}(a\rvert s)>0$, for all $(s,a)$.
\end{lemma}
\begin{proof}[Proof of Lemma \ref{lemma:policy improvement mini-batch}]
Consider the $t$-th iteration of PPO-Clip (cf. Algorithm \ref{algo:HPO-AM}) and the corresponding update from $\pi^{(t)}$ to $\pi^{(t+1)}$.
Regarding (ii), recall from Algorithm \ref{algo:EMD} that $K^{(t)}$ denotes the number of iterations undergone by the EMDA subroutine for the update from $\pi^{(t)}$ to $\pi^{(t+1)}$ and that $K^{(t)}$ is designed to be finite.
Therefore, it is easy to verify that $\pi^{(t+1)}(a\rvert s)>0$ for all $(s,a)$ by the exponentiated gradient update scheme of EMDA and the strict positivity of $\pi^{(t)}$.

Next, for ease of exposition, for each $k\in \{0,1,\cdots,K^{(t)}\}$ and for each state-action pair $(s,a)$, let $\widetilde{\theta}^{(k)}_{s,a}$ denote the policy parameter after $k$ EMDA iterations.
Regarding (i), recall that we define $g^{(k)}_{s,a}:=\frac{\partial \mathcal{L}(\theta)}{\partial \theta_{s,a}}\big\rvert_{\theta=\widetilde{\theta}^{(k)}_{s}}$ and $w_s^{(k)}:=(e^{-\eta g^{(k)}_{s,1}},\cdots,e^{-\eta g^{(k)}_{s,\lvert \mathcal{A}\rvert}})$.
Note that as the weights in the loss function only affects the effective step sizes of EMDA, we simply set the weights of PPO-Clip to be one, without loss of generality.
By EMDA in Algorithm \ref{algo:EMD}, for every $(s,a)\in \cD^{(t)}$, we have
\begin{equation}
\label{eq:pi_t+1 and pi_t}
    \pi^{(t+1)}(a\rvert s)=\frac{\prod_{k=0}^{K^{(t)}-1}\exp({-\eta}g_{s,a}^{(k)})}{{\prod_{k=0}^{K^{(t)}-1}}\langle w_s^{(k)},\widetilde{\theta}_s^{(k)}\rangle}\cdot\pi^{(t)}(a\rvert s).
\end{equation}
Note that $g^{(k)}_{s,a}$ can be written as
\begin{align}
\label{eq:g_sa^k}
    g^{(k)}_{s,a}=
    \begin{cases}
        -\frac{1}{\pi^{(t)}(a\rvert s)}\sgn({A}^{(t)}(s,a))&, \text{if }\big(\frac{\widetilde{\theta}_{s,a}^{(k)}}{\pi^{(t)}(a\rvert s)}-1\big)\sgn({A}^{(t)}(s,a))< \epsilon, (s,a)\in \cD^{(t)}\\ 
        0&, \text{otherwise }    
    \end{cases}
\end{align}
By (\ref{eq:pi_t+1 and pi_t})-(\ref{eq:g_sa^k}), it is easy to verify that for those $(s,a)\in \cD^{(t)}$ with positive advantage, we must have $\prod_{k=0}^{K^{(t)}-1}\exp({-\eta}g_{s,a}^{(k)})>1$.
Similarly, for those $(s,a)\in \cD^{(t)}$ with negative advantage, we have $\prod_{k=0}^{K^{(t)}-1}\exp({-\eta}g_{s,a}^{(k)})<1$.
Now we are ready to check the condition of strict policy improvement given by Lemma \ref{prop:first}: For each $s\in \cS$, we have
\begin{align}
\label{eq:sufficient condition mini-batch}
    \sum_{a\in \cA}\pi^{(t+1)}(a\rvert s)A^{(t)}(a\rvert s)=\frac{1}{\prod_{k=0}^{K^{(t)}-1}\langle w_s^{(k)},\widetilde{\theta}_s^{(k)}\rangle}\sum_{a\in \cA}\Big(\prod_{k=0}^{K^{(t)}-1}\exp({-\eta}g_{s,a}^{(k)})\Big)\pi^{(t)}(a\rvert s)A^{(t)}(a\rvert s)>0.
\end{align}
Hence, we conclude that the strict state-wise policy improvement property indeed holds, i.e., $\pi^{(t+1)}>\pi^{(t)}$.
\end{proof}

Note that Lemma \ref{lemma:policy improvement mini-batch} implies that the limits $V^{(\infty)}(s)$, $Q^{(\infty)}(s,a)$, $A^{(\infty)}(s,a)$ exist, for every sample path: By the strict policy improvement shown in Lemma \ref{lemma:policy improvement mini-batch}, we know that the sequence of state values is point-wise monotonically increasing, i.e., $V^{(t+1)}(s)\geq V^{(t)}(s),\ \forall s\in\mathcal{S}$. 
Moreover, by the bounded reward and the discounted setting, we have $-\frac{R_{\max}}{1-\gamma}\leq V^{(t)}(s)\leq \frac{R_{\max}}{1-\gamma}$. 
The above monotone increasing property and boundedness imply convergence, i.e., $V^{(t)}(s)\rightarrow V^{(\infty)}(s)$, for each sample path.
Similarly, we also know that $Q^{(t)}(s,a)\rightarrow Q^{(\infty)}(s,a)$, and thus $A^{(t)}(s,a)\rightarrow A^{(\infty)}(s,a)$.
As a result, we can define the three sets $I_{s}^{+}$, $I_{s}^{0}$ and $I_{s}^{-}$ as
\begin{align}
    I_{s}^{+} & :=\{a\in\mathcal{A}|A^{(\infty)}(s,a)>0\}, \\
    I_{s}^{0} & :=\{a\in\mathcal{A}|A^{(\infty)}(s,a)=0\}, \\
    I_{s}^{-} & :=\{a\in\mathcal{A}|A^{(\infty)}(s,a)<0\}.
\end{align}
Note that for each sample path, the sets $I_{s}^{+}$, $I_{s}^{0}$ and $I_{s}^{-}$ are well-defined, based on the limit $A^{(\infty)}(s,a)$.

\begin{lemma}
\label{lemma:pi convergence mini-batch}
Conditioned on the event that each state-action pair occurs infinitely often in $\{\cD^{(t)}\}$, if $I_s^{+}$ is not an empty set, then we have $\sum_{a\in I_{s}^{-}}\pi^{(t)}(a\rvert s)\rightarrow 0$, as $t\rightarrow \infty$. 
\end{lemma}
\begin{proof}[Proof of Lemma \ref{lemma:pi convergence mini-batch}]
We discuss each state separately as it is sufficient to show that for each state $s$, given some fixed $a'\in I_{s}^{+}$, for any $a''\in I_{s}^{-}$, we have $\frac{\pi^{(t)}(a''\rvert s)}{\pi^{(t)}(a'\rvert s)}\rightarrow 0$, as $t\rightarrow \infty$.
For ease of exposition, we reuse some of the notations from the proof of Lemma \ref{lemma:policy improvement mini-batch}.
Recall that we let $K^{(t)}$ denote the number of iterations undergone by the EMDA subroutine for the update from $\pi^{(t)}$ to $\pi^{(t+1)}$, and $K^{(t)}$ is designed to be finite.
For each $k\in \{0,1,\cdots,K^{(t)}\}$ and for each state-action pair $(s,a)$, let $\widetilde{\theta}^{(k)}_{s,a}$ denote the policy parameter after $k$ EMDA iterations.
Recall from Algorithm \ref{algo:EMD} that $g^{(k)}_{s,a}:=\frac{\partial \mathcal{L}(\theta)}{\partial \theta_{s,a}}\big\rvert_{\theta=\widetilde{\theta}^{(k)}_{s}}$ and $w_s^{(k)}:=(e^{-\eta g^{(k)}_{s,1}},\cdots,e^{-\eta g^{(k)}_{s,\lvert \mathcal{A}\rvert}})$.
Define $\Delta_*:=\min_{a\in I_{s}^{+}\cup I_{s}^{-}}\lvert A^{(\infty)}(s,a)\rvert>0$ (and here $\Delta_*$ is a random variable as $A^{(\infty)}(s,a)$ is defined with respect to each sample path).
By the definition of $I_s^{+}$, $I_s^{-}$ and $\Delta_*$, we know that for each sample path, there must exist finite $T_{+}$ and $T_{-}$ such that: (i) for every $a\in I_{s}^{+}$, $A^{(t)}(s,a)\geq \frac{\Delta_*}{2}$, for all $t>T_{+}$, and (ii) for every $a\in I_{s}^{-}$, $A^{(t)}(s,a)\leq -\frac{\Delta_*}{2}$, for all $t>T_{-}$.
Under Assumption \ref{assumption:distinct states}, at each iteration $t$ with $t>\max\{T_{+},T_{-}\}$, there are three possible cases regarding the state-action pairs $(s,a')$ and $(s,a'')$:

\vspace{-2mm}
\begin{itemize}
    \item \textbf{Case 1:} $(s,a')\in \cD^{(t)}$, $(s,a'')\notin \cD^{(t)}$\\
    By the EMDA subroutine and (\ref{eq:pi_t+1 and pi_t}), we have
    \begin{align}
    \label{eq:case 1}
        \frac{\pi^{(t+1)}(a''\rvert s)}{\pi^{(t+1)}(a'\rvert s)}=\frac{\pi^{(t)}(a''\rvert s)}{\pi^{(t)}(a'\rvert s)}\cdot \prod_{k=0}^{K^{(t)}-1}\exp({\eta}g_{s,a'}^{(k)})\leq \frac{\pi^{(t)}(a''\rvert s)}{\pi^{(t)}(a'\rvert s)}\cdot \underbrace{\exp(-\eta)}_{<1},
    \end{align}
    where the last inequality holds by (\ref{eq:g_sa^k}), $a'\in I_{s}^{+}$, and $\pi^{(t)}(a'\rvert s)\leq 1$.
    \item \textbf{Case 2:} $(s,a')\notin \cD^{(t)}$, $(s,a'')\in \cD^{(t)}$\\
    By the EMDA subroutine, we have $-g_{s,a''}^{(0)}<0$ and $-g_{s,a''}^{(k)}\leq 0$ for all $k\in \{1,\cdots, K^{(t)}\}$. Therefore, we have
    \begin{align}
    \label{eq:case 2}
        \frac{\pi^{(t+1)}(a''\rvert s)}{\pi^{(t+1)}(a'\rvert s)}<\frac{\pi^{(t)}(a''\rvert s)}{\pi^{(t)}(a'\rvert s)}.
    \end{align}    
    \item \textbf{Case 3:} $(s,a')\notin \cD^{(t)}$, $(s,a'')\notin \cD^{(t)}$\\
    Under EMDA, as neither $(s,a')$ nor $(s,a'')$ is in $\notin \cD^{(t)}$, the action probability ratio between these two actions remains unchanged (despite that the values of $\pi^{(t)}(a''\rvert s)$ and $\pi^{(t)}(a''\rvert s)$ can still change if there is an action $a'''$ such that $a'''\neq a'$, $a'''\neq a''$, and $(s,a''')\in \cD^{(t)}$), i.e.,
    \begin{align}
    \label{eq:case 3}
        \frac{\pi^{(t+1)}(a''\rvert s)}{\pi^{(t+1)}(a'\rvert s)}=\frac{\pi^{(t)}(a''\rvert s)}{\pi^{(t)}(a'\rvert s)}.
    \end{align}
\end{itemize}
Conditioned on the event that each state-action pair occurs infinitely often in $\{\cD^{(t)}\}$, we know Case 1 and (\ref{eq:case 3}) must occur infinitely often. 
By (\ref{eq:case 1})-(\ref{eq:case 3}),  we conclude that $\frac{\pi^{(t)}(a''\rvert s)}{\pi^{(t)}(a'\rvert s)}\rightarrow 0$, as $t\rightarrow \infty$, for every $a''\in I_{s}^{-}$.
\end{proof}

\begin{lemma}
\label{lemma:epsilon lower bound mini-batch}
Conditioned on the event that each state-action pair occurs infinitely often in $\{\cD^{(t)}\}$, if $I_s^{+}$ is not an empty set, then there exists some constant $c>0$ such that $\sum_{a\in I_{s}^{-}}\pi^{(t)}(a\rvert s)\geq c$, for infinitely many $t$.
\end{lemma}
\begin{proof}[Proof of Lemma \ref{lemma:epsilon lower bound mini-batch}]
For each $(s,a)$, define $\sT_{s,a}:=\{t: (s,a)\in \cD^{(t)}\}$ to be the index set that collects the time indices at which $(s,a)$ is contained in the mini-batch.
Given that each state-action pair occurs infinitely often, we know $\sT_{s,a}$ is a (countably) infinite set.

For ease of exposition, define a positive constant $\chi$ as
\begin{equation}
    \chi:=\frac{e\cdot \eta}{e\cdot \eta+1}<1.
\end{equation}
Define $\Delta:=\min_{a\in I_{s}^{+}}A^{(\infty)}(s,a)>0$ (and here $\Delta$ is a random variable as $A^{(\infty)}(s,a)$ is defined with respect to each sample path).
By the definition of $I_s^{+}$ and $\Delta$, we know that there must exist a finite $T^{(+)}$ such that for every $a\in I_{s}^{+}$, $A^{(t)}(s,a)\geq \frac{3\Delta}{4}$, for all $t>T^{(+)}$.
Similarly, by the definition of $I_s^{0}$, there must exist a finite $T^{(0)}$ such that for every $a\in I_{s}^{0}$, $\lvert A^{(t)}(s,a)\rvert \leq \frac{\chi\Delta}{4}$, for all $t>T^{(0)}$.
We also define $T^*:=\max\{T^{(+)}, T^{(0)}\}$.

We reuse some of the notations from the proof of Lemma \ref{lemma:policy improvement mini-batch}.
Recall that we let $K^{(t)}$ denote the number of iterations undergone by the EMDA subroutine for the update from $\pi^{(t)}$ to $\pi^{(t+1)}$, and $K^{(t)}$ is a finite positive integer.
For ease of exposition, for each $k\in \{0,1,\cdots,K^{(t)}\}$ and for each state-action pair $(s,a)$, let $\widetilde{\theta}^{(k)}_{s,a}$ denote the policy parameter after $k$ EMDA iterations.
Recall that we define $g^{(k)}_{s,a}:=\frac{\partial \mathcal{L}(\theta)}{\partial \theta_{s,a}}\big\rvert_{\theta=\widetilde{\theta}^{(k)}_{s}}$ and $w_s^{(k)}:=(e^{-\eta g^{(k)}_{s,1}},\cdots,e^{-\eta g^{(k)}_{s,\lvert \mathcal{A}\rvert}})$.
If $I_{s}^{+}$ is not an empty set, then we can select an arbitrary action $a'\in I_{s}^{+}$.
For any $t$ with $t>T^{(+)}$ and $t\in \sT_{s,a'}$, by (\ref{eq:pi_t+1 and pi_t}) we have
\begin{align}
    \pi^{(t+1)}(a'\rvert s)& =\frac{\prod_{k=0}^{K^{(t)}-1}\exp({-\eta}g_{s,a'}^{(k)})}{{\prod_{k=0}^{K^{(t)}-1}}\langle w_s^{(k)},\widetilde{\theta}_s^{(k)}\rangle}\cdot\pi^{(t)}(a'\rvert s)\label{eq:pi_t+1 1}\\
    &\geq \frac{\pi^{(t)}(a'\rvert s) \exp(-\eta g_{s,a'}^{(0)})}{\pi^{(t)}(a'\rvert s) \exp(-\eta g_{s,a'}^{(0)})+1}\label{eq:pi_t+1 2}\\
    &\geq \frac{\pi^{(t)}(a'\rvert s) \exp(\eta/\pi^{(t)}(a'\rvert s))}{\pi^{(t)}(a'\rvert s) \exp(\eta/ \pi^{(t)}(a'\rvert s))+1}\label{eq:pi_t+1 3}\\
    &\geq \frac{e\cdot\eta}{e\cdot\eta+1}=\chi, \label{eq:pi_t+1 4}
\end{align}
where (\ref{eq:pi_t+1 2}) holds due to the fact that $\widetilde{\theta}_{s,a}^{(k)}$ is non-decreasing with $k$ under Assumption \ref{assumption:distinct states} and that $K^{(t)}\geq 1$, (\ref{eq:pi_t+1 3}) follows from (\ref{eq:g_sa^k}) and that $a'\in I_{s}^{+}$, and (\ref{eq:pi_t+1 4}) holds by that the function $q(z)=z\cdot \exp(\eta/z)$ has a unique minimizer at $z=\eta$ with minimum value $e\cdot\eta$.
For all $t$ that satisfies $(t-1)\in \sT_{s,a}$ and $t>T^*$, we have
\begin{align}
    \sum_{a\in I_{s}^{-}}\pi^{(t)}(a|s)&\geq\frac{\sum_{a\in I_{s}^{+}}\pi^{(t)}(a|s)A^{(t)}(s,a) + \sum_{a\in I_{s}^{0}}\pi^{(t)}(a|s)A^{(t)}(s,a)}{\max_{a\in I_{s}^{-}} \lvert A^{(t)}(s,a)\rvert}\label{eq:lower bound mini-batch case 1-2}\\
    &\geq\frac{\chi(3\Delta/4) - 1\cdot (\chi\Delta/4) }{\frac{2R_{\max}}{1-\gamma}}\label{eq:lower bound mini-batch case 1-3}\\
    &=\frac{\chi\Delta }{\frac{4R_{\max}}{1-\gamma}},\label{eq:lower bound mini-batch case 1-4}
\end{align}
where (\ref{eq:lower bound mini-batch case 1-2}) follows from that $\sum_{a\in \cA}\pi^{(t)}(a\rvert s)=0$ and $A^{(t)}(s,a)<0$ for all $a\in I_{s}^{-}$, and (\ref{eq:lower bound mini-batch case 1-3}) follows from the definition of $T^{(+)}, T^{(0)}$ as well as the boundedness of rewards.
Since $\sT_{s,a}$ is a countably infinite set, we know $ \sum_{a\in I_{s}^{-}}\pi^{(t)}(a|s)\geq \frac{\chi\Delta}{\frac{4R_{\max}}{1-\gamma}}$, for infinitely many $t$.

\end{proof}

\subsection{Proof of Theorem~\ref{thm:mini-batch}}
Now we are ready to show Theorem~\ref{thm:mini-batch}. 
For ease of exposition, we restate Theorem~\ref{thm:mini-batch} as follows.
\begin{theoremstar}[Global Convergence of PPO-Clip]
Under PPO-Clip, we have $V^{(t)}(s)\rightarrow V^{\pi^*}(s)\text{ as }t\rightarrow\infty,\ \forall s\in\mathcal{S}$, with probability one.
\end{theoremstar}
\begin{proof}
We establish that $\pi^{(t)}$ converges to an optimal policy by showing that $I_{s}^{+}$ is an empty set for all $s$.
Under Assumption~\ref{assumption:infinite visit}, the analysis below is presumed to be conditioned on the event that each state-action pair occurs infinitely often in $\{\cD^{(t)}\}$.
The proof proceeds by contradiction as follows:
Suppose $I_{s}^{+}$ is non-empty.
From Lemma~\ref{lemma:pi convergence mini-batch}, we have that $\sum_{a\in I_{s}^{-}}\pi^{(t)}(a\rvert s)\rightarrow 0$, as $t\rightarrow \infty$. 
However, Lemma \ref{lemma:epsilon lower bound mini-batch} suggests that there exists some constant $c>0$ such that $\sum_{a\in I_{s}^{-}} \pi^{(t)}(a\rvert s)\geq c$ infinitely often.
This leads to a contraction, and thus completes the proof.  
\end{proof}

\section{{Global Convergence of Tabular PPO-Clip With Alternative Classifiers}}
\label{app:secondthm}
\begin{theorem}
\label{thm:second}
    Theorem \ref{thm:mini-batch} also holds under the following algorithms: 
    (i) {PPO-Clip} with the classifier $\log(\pi_{\theta}(a|s))-\log(\pi(a|s))$ (termed {PPO-Clip-log});
    (ii) {PPO-Clip} with the classifier $\sqrt{\rho_{s,a}(\theta)}-1$ (termed {PPO-Clip-root}).
\end{theorem}
\begin{proof}[Proof of Theorem \ref{thm:second}]
We show that Theorem~\ref{thm:mini-batch} can be extended to these two alternative classifiers by following the proof procedure of Theorem~\ref{thm:mini-batch}. Specifically, we extend the supporting lemmas (cf. Lemma \ref{lemma:policy improvement mini-batch}, Lemma \ref{lemma:pi convergence mini-batch}, and Lemma \ref{lemma:epsilon lower bound mini-batch}) as follows:
\begin{itemize}
    \item To extend Lemma \ref{lemma:policy improvement mini-batch} to the alternative classifiers, we can reuse (\ref{eq:pi_t+1 and pi_t}) and rewrite (\ref{eq:g_sa^k HPO-AM-log}) for each classifier. That is, for PPO-Clip-log, we have
    \begin{align}
    \label{eq:g_sa^k HPO-AM-log}
    g^{(k)}_{s,a}=
    \begin{cases}
        -\frac{1}{\widetilde{\theta}_{s,a}^{(k)}}\sgn({A}^{(t)}(s,a))&, \text{if }\log\big(\frac{\widetilde{\theta}_{s,a}^{(k)}}{\pi^{(t)}(a\rvert s)}\big)\sgn({A}^{(t)}(s,a))< \epsilon, (s,a)\in \cD^{(t)}\\ 
        0&, \text{otherwise }    
    \end{cases}
    \end{align}
    On the other hand, for PPO-Clip-root, we have
    \begin{align}
    \label{eq:g_sa^k HPO-AM-root}
    g^{(k)}_{s,a}=
    \begin{cases}
        -\frac{1}{2\sqrt{\widetilde{\theta}_{s,a}^{(k)}\pi^{(t)}(a\rvert s)}}\sgn({A}^{(t)}(s,a))&, \text{if }\Big(\sqrt{\frac{\widetilde{\theta}_{s,a}^{(k)}}{\pi^{(t)}(a\rvert s)}}-1\Big)\sgn({A}^{(t)}(s,a))< \epsilon, (s,a)\in \cD^{(t)}\\ 
        0&, \text{otherwise }    
    \end{cases}
    \end{align}
    As the sign of $g_{s,a}^{(k)}$ depends only on the sign of the advantage, it is easy to verify that (\ref{eq:sufficient condition mini-batch}) still goes through and hence the sufficient condition of Lemma \ref{prop:first} is satisfied under these two alternative classifiers.
    Moreover, by using the same argument of EMDA as that in Lemma \ref{lemma:policy improvement mini-batch}, it is easy to verify that $\pi^{(t+1)}(a\rvert s)>0$ for all $(s,a)$.
    \vspace{-2pt}
    \item Regarding Lemma~\ref{lemma:pi convergence mini-batch}, we can extend this result again by considering the three cases as in Lemma~\ref{lemma:pi convergence mini-batch}. For Case 1, given the $g_{s,a}^{(k)}$ in (\ref{eq:g_sa^k HPO-AM-log}) and (\ref{eq:g_sa^k HPO-AM-root}), we have: For PPO-Clip-log,
    \begin{align}
    \label{eq:case 1 HPO-AM-log}
        \frac{\pi^{(t+1)}(a''\rvert s)}{\pi^{(t+1)}(a'\rvert s)}=\frac{\pi^{(t)}(a''\rvert s)}{\pi^{(t)}(a'\rvert s)}\cdot \prod_{k=0}^{K^{(t)}-1}\exp({\eta}g_{s,a'}^{(k)})\leq \frac{\pi^{(t)}(a''\rvert s)}{\pi^{(t)}(a'\rvert s)}\cdot \underbrace{\exp(-\eta)}_{<1}.
    \end{align}
    Similarly, for PPO-Clip-root, we have
    \begin{align}
    \label{eq:case 1 HPO-AM-root}
        \frac{\pi^{(t+1)}(a''\rvert s)}{\pi^{(t+1)}(a'\rvert s)}=\frac{\pi^{(t)}(a''\rvert s)}{\pi^{(t)}(a'\rvert s)}\cdot \prod_{k=0}^{K^{(t)}-1}\exp({\eta}g_{s,a'}^{(k)})\leq \frac{\pi^{(t)}(a''\rvert s)}{\pi^{(t)}(a'\rvert s)}\cdot \underbrace{\exp(-\frac{\eta}{2})}_{<1}.
    \end{align}
    Moreover, it is easy to verify that the arguments in Case 2 and Case 3 still hold under these two alternative classifiers. Hence, Lemma~\ref{lemma:pi convergence mini-batch} still holds.
    \item Regarding Lemma \ref{lemma:epsilon lower bound mini-batch}, we can reuse all the setup and slightly revise (\ref{eq:pi_t+1 1})-(\ref{eq:pi_t+1 4}) for the two alternative classifiers: For PPO-Clip-log, by (\ref{eq:g_sa^k HPO-AM-log}), we have
    \begin{align}
    \pi^{(t+1)}(a'\rvert s)& =\frac{\prod_{k=0}^{K^{(t)}-1}\exp({-\eta}g_{s,a'}^{(k)})}{{\prod_{k=0}^{K^{(t)}-1}}\langle w_s^{(k)},\widetilde{\theta}_s^{(k)}\rangle}\cdot\pi^{(t)}(a'\rvert s)\label{eq:pi_t+1 HPO-AM-log 1}\\
    &\geq \frac{\pi^{(t)}(a'\rvert s) \exp(-\eta g_{s,a'}^{(0)})}{\pi^{(t)}(a'\rvert s) \exp(-\eta g_{s,a'}^{(0)})+1}\label{eq:pi_t+1 HPO-AM-log 2}\\
    &\geq \frac{\pi^{(t)}(a'\rvert s) \exp(\eta/\pi^{(t)}(a'\rvert s))}{\pi^{(t)}(a'\rvert s) \exp(\eta/ \pi^{(t)}(a'\rvert s))+1}\label{eq:pi_t+1 HPO-AM-log 3}\\
    &\geq \frac{e\cdot\eta}{e\cdot\eta+1}. \label{eq:pi_t+1 HPO-AM-log 4}
    \end{align}
    Similarly, for PPO-Clip-root, by (\ref{eq:g_sa^k HPO-AM-root}), we have
    \begin{align}
    \pi^{(t+1)}(a'\rvert s)& =\frac{\prod_{k=0}^{K^{(t)}-1}\exp({-\eta}g_{s,a'}^{(k)})}{{\prod_{k=0}^{K^{(t)}-1}}\langle w_s^{(k)},\widetilde{\theta}_s^{(k)}\rangle}\cdot\pi^{(t)}(a'\rvert s)\label{eq:pi_t+1 HPO-AM-root 1}\\
    &\geq \frac{\pi^{(t)}(a'\rvert s) \exp(-\eta g_{s,a'}^{(0)})}{\pi^{(t)}(a'\rvert s) \exp(-\eta g_{s,a'}^{(0)})+1}\label{eq:pi_t+1 HPO-AM-root 2}\\
    &\geq \frac{\pi^{(t)}(a'\rvert s) \exp(\eta/2\pi^{(t)}(a'\rvert s))}{\pi^{(t)}(a'\rvert s) \exp(\eta/2 \pi^{(t)}(a'\rvert s))+1}\label{eq:pi_t+1 HPO-AM-root 3}\\
    &\geq \frac{e\cdot\frac{\eta}{2}}{e\cdot\frac{\eta}{2}+1}. \label{eq:pi_t+1 HPO-AM-root 4}
    \end{align}
    Accordingly, (\ref{eq:lower bound mini-batch case 1-2})-(\ref{eq:lower bound mini-batch case 1-4}) still go through and hence Lemma \ref{lemma:epsilon lower bound mini-batch} indeed holds under PPO-Clip-log and PPO-Clip-root.
\end{itemize}
In summary, since all the supporting lemmas hold for these alternative classifiers, we complete this part of the proof by obtaining a contradiction similar to that in Theorem \ref{thm:mini-batch}.
\end{proof}

\section{Experiments and Detailed Configuration}
\label{app:Experiments}

\subsection{Experimental Settings}
For our experiments, we implement Neural PPO-Clip with different classifiers on the open-source RL baseline3-zoo framework \citep{rl-zoo3}. Specifically, we consider four different classifiers as follows: (i) $\rho_{s,a}(\theta) - 1$ (the standard PPO-Clip classifier); (ii) $\pi_{\theta}(a|s) - \pi_{\theta_t}(a|s)$ (PPO-Clip-sub); (iii) $\sqrt{\rho_{s,a}(\theta)} - 1$ (PPO-Clip-root); (iv) $\log(\pi_{\theta}(a|s)) - \log(\pi_{\theta_t}(a|s))$ (PPO-Clip-log). We test these variants in the MinAtar environments \citep{young19minatar} such as Breakout and Space Invaders. On the other hand, we evaluate them in OpenAI Gym environments \citep{brockman2016openai}, which are LunarLander, Acrobot, and CartPole, as well. For the comparison with other benchmark methods, we consider A2C and Rainbow. The training curves are drawn by the averages over 5 random seeds. For the computing resources we use to run the experiment, we use (i) CPU: Intel(R) Xeon(R) CPU E5-2630 v4 @ 2.20GHz; (ii)
GPU: NVIDIA GeForce GTX 1080.

\subsection{Model Parameters}

The neural networks architecture of policy and value function in the experiments share two full-connected layers and connect to respective output layers. We provide the parameters of the algorithms for each environment in the following tables \ref{tab:params in Breakout}-\ref{tab:params in Acrobot}. 
Notice that lin\_5e-4 means that the learning rate decays linearly from $5 \times 10^{-4}$ to 0. Also, the \texttt{vf\_coef} is the weight of the value loss and \texttt{temperature\_lambda} is the pre-constant of the adaptive temperature parameter for energy-based neural networks. We also give the parameters searching range in table \ref{tab: params search}.

\begin{table}[!htbp]
    \centering
    \caption[Short Title]%
    {Parameters for MinAtar Breakout experiments.}
    \begin{tabular}{|l|c|c|c|c|c|}
    \hline
    Hyperparameters & PPO-Clip & PPO-Clip-sub & PPO-Clip-root & PPO-Clip-log & A2C \\
    \hline
    batch\_size & 256 & 256 & 256 & 256 & 80 \\
    \hline
    learning\_rate & lin\_1e-3 & lin\_1e-3 & lin\_1e-3 & lin\_1e-3 & 7e-4\\
    \hline
    vf\_coef & 0.00075 & 0.00075 & 0.00075 & 0.00075 & 0.25 \\
    \hline
    EMDA step size & 0.005 & 0.005 & 0.005 & 0.005 & -\\
    \hline
    EMDA iteration & 2 & 2 & 2 & 2 & - \\
    \hline
    clipping range & 0.3 & 0.3 & 0.3 & 0.3 & - \\
    \hline
    temperature\_lambda & 25 & 25 & 25 & 25 & -\\
    \hline
    \end{tabular}
    \label{tab:params in Breakout}
\end{table}
\begin{table}[!htbp]
    \centering
    \caption[Short Title]%
    {Parameters for MinAtar Space Invaders experiments.}
    \begin{tabular}{|l|c|c|c|c|c|}
    \hline
    Hyperparameters & PPO-Clip & PPO-Clip-sub & PPO-Clip-root & PPO-Clip-log & A2C \\
    \hline
    batch\_size & 256 & 256 & 256 & 256 & 80 \\
    \hline
    learning\_rate & lin\_1e-3 & lin\_1e-3 & lin\_1e-3 & lin\_1e-3 & 7e-4\\
    \hline
    vf\_coef & 0.00075 & 0.00075 & 0.00075 & 0.00075 & 0.25 \\
    \hline
    EMDA step size & 0.005 & 0.005 & 0.005 & 0.005 & -\\
    \hline
    EMDA iteration & 5 & 5 & 2 & 5 & - \\
    \hline
    clipping range & 0.5 & 0.5 & 0.5 & 0.5 & - \\
    \hline
    temperature\_lambda & 10 & 10 & 10 & 10 & -\\
    \hline
    \end{tabular}
    \label{tab:params in Space Invaders}
\end{table}

\begin{table}[!htbp]
    \centering
    \caption[Short Title]%
    {Parameters for OpenAI Gym LunarLander-v2 experiments.}
    \begin{tabular}{|l|c|c|c|c|c|}
    \hline
    Hyperparameters & PPO-Clip & PPO-Clip-sub & PPO-Clip-root & PPO-Clip-log & A2C \\
    \hline
    batch\_size & 64 & 8 & 64 & 64 & 40 \\
    \hline
    learning\_rate & lin\_5e-4 & lin\_5e-4 & lin\_5e-4 & lin\_5e-4 & lin\_8.3e-4\\
    \hline
    vf\_coef & 0.75 & 0.75 & 0.75 & 0.75 & 0.5 \\
    \hline
    EMDA step size & 0.01 & 0.002 & 0.01 & 0.01 & -\\
    \hline
    EMDA iteration & 5 & 5 & 5 & 5 & - \\
    \hline
    clipping range & 0.3 & 0.5 & 0.3 & 0.3 & - \\
    \hline
    temperature\_lambda & 10 & 10 & 10 & 10 & -\\
    \hline
    \end{tabular}
    \label{tab:params in LunarLander}
\end{table}
\begin{table}[!htbp]
    \centering
    \caption[Short Title]%
    {Parameters for OpenAI Gym Acrobot-v1 experiments.}
    \begin{tabular}{|l|c|c|c|c|c|}
    \hline
    Hyperparameters & PPO-Clip & PPO-Clip-sub & PPO-Clip-root & PPO-Clip-log & A2C \\
    \hline
    batch\_size & 64 & 64 & 64 & 64 & 40 \\
    \hline
    learning\_rate & lin\_7.5e-4 & lin\_7.5e-4 & lin\_7.5e-4 & lin\_7.5e-4 & lin\_8.3e-4\\
    \hline
    vf\_coef & 0.5 & 0.5 & 0.5 & 0.5 & 0.5 \\
    \hline
    EMDA step size & 0.01 & 0.01 & 0.01 & 0.01 & -\\
    \hline
    EMDA iteration & 5 & 5 & 5 & 5 & - \\
    \hline
    clipping range & 0.3 & 0.3 & 0.3 & 0.3 & - \\
    \hline
    temperature\_lambda & 10 & 10 & 10 & 10 & -\\
    \hline
    \end{tabular}
    \label{tab:params in Acrobot}
\end{table}
\begin{table}[!htbp]
    \centering
    \caption[Short Title]%
    {Parameters for OpenAI Gym CartPole-v1 experiments.}
    \begin{tabular}{|l|c|c|c|c|c|}
    \hline
    Hyperparameters & PPO-Clip & PPO-Clip-sub & PPO-Clip-root & PPO-Clip-log & A2C \\
    \hline
    batch\_size & 64 & 64 & 64 & 64 & 40 \\
    \hline
    learning\_rate & lin\_7.5e-4 & lin\_7.5e-4 & lin\_7.5e-4 & lin\_7.5e-4 & lin\_8.3e-4\\
    \hline
    vf\_coef & 0.5 & 0.5 & 0.5 & 0.5 & 0.5 \\
    \hline
    EMDA step size & 0.01 & 0.01 & 0.01 & 0.01 & -\\
    \hline
    EMDA iteration & 5 & 5 & 5 & 5 & - \\
    \hline
    clipping range & 0.3 & 0.3 & 0.3 & 0.3 & - \\
    \hline
    temperature\_lambda & 10 & 10 & 10 & 10 & -\\
    \hline
    \end{tabular}
    \label{tab:params in CartPole}
\end{table}
\begin{table}[!htbp]
    \centering
    \caption[Short Title]%
    {Parameters searching range for the experiments.}
    \begin{tabular}{|l|c|}
    \hline
    Hyperparameters & Searching Range \\
    \hline
    batch\_size & 64, 128, 256 \\
    \hline
    learning\_rate & lin\_1e-3, lin\_7.5e-4, lin\_5e-4, lin\_2.5e-4\\
    \hline
    vf\_coef & 0.00075, 0.0005, 0.3, 0.5, 0.75, 0.8 \\
    \hline
    EMDA step size & 0.001, 0.005, 0.075, 0.02, 0.05, 0.01, 0.1 \\
    \hline
    EMDA iteration & 2, 5, 10 \\
    \hline
    clipping range & 0.3, 0.5, 0.7 \\
    \hline
    temperature\_lambda & 0.1, 0.5, 1, 5, 10, 25, 40, 60, 75\\
    \hline
    \end{tabular}
    \label{tab: params search}
\end{table}

\subsection{Additional Experimental Validation}

\textbf{Ablation study of EMDA iterations.} As shown in Algorithm \ref{algo:2}, the number of EMDA iteration $K$ is one of the hyperparameters of the algorithm. We conduct ablation studies on it, specifically for $K = 2, 5, 10$. In the LunarLander environment, their scores are 247, 253, and 237, respectively. This shows empirically that the performance is not sensitive to $K$.

\noindent \textbf{Empirical comparison between SGD-based PPO and EMDA-based PPO.} We report the results under Breakout and 5 seeds. After 5M steps, the conventional PPO has a mean 21.48 with std. dev. 19.41. On the other hand, EMDA-based PPO has a mean 18.24 with std. dev. 3.97.
Also in LunarLander, we show that EMDA-based PPO achieves comparable or better performance than conventional PPO in these RL benchmark environments.

\noindent \textbf{Experiments of the generalized objective using different classifiers for SGD-based PPO.} Experiments of the generalized objective using different classifiers:} We conduct the experiments for the generalized objective under the conventional PPO-Clip approach. In Breakout with 5 seeds, the mean scores of the root-, log-, and sub-classifiers are 18.08, 12.20, and 17.09, respectively. Also, the standard deviations are 8.83, 0.99, and 7.42, respectively. Moreover, our experiment results show that other classifiers outperform the original objective in some environments, which implies that each of them has its own advantage.

\section{Supplementary Related Works}

\noindent \textbf{Global Convergence of Policy Gradient Methods.}  One related line of recent research is on the global convergence of policy gradient methods. 
\citep{agarwal2019theory,agarwal2020optimality} establishes global convergence results of various policy gradient approaches, including the vanilla policy gradient (with both tabular and softmax policy parametrizations) and the natural policy gradient method (with a softmax policy parametrization).
Concurrently, \citep{bhandari2019global} provides an alternative analysis of global optimality of the policy gradient method.
\citep{wang2019neural} provides the global optimality guarantees for both the vanilla policy gradient and natural policy gradient methods under the overparameterized two-layer neural parameterization.
\citep{mei2020global} establishes the convergence rates of both vanilla softmax policy gradient and the entropy-regularized policy gradient.
\citep{liu2020improved} further establishes the global convergence rates of various variance-reduced policy gradient methods.
Inspired by the analysis of \citep{agarwal2019theory}, we establish the global convergence of the proposed HPO-AM.

\noindent \textbf{Global Convergence of TRPO and PPO.} 
Regarding TRPO, \citep{shani2020adaptive} presents the global convergence rates of an adaptive TRPO, which is established by connecting TRPO and the mirror descent method.
\citep{liu2019neural} proves global convergence in expected total reward for a neural variant of PPO with adaptive Kullback-Leibler penalty (PPO-KL).
To the best of our knowledge, \citep{liu2019neural} appears to be the only global convergence result for PPO-KL.
By contrast, our focus is PPO-clip.
Given the salient algorithmic difference between PPO-KL and PPO-clip, there remains no proof of global convergence to an optimal policy for PPO with a clipped objective.
In this paper, we rigorously provide the first global convergence guarantee for a variant of PPO-clip.

\noindent \textbf{RL as Classification.} Regarding the general idea of casting RL as a classification problem, it has been investigated by the existing literature \citep{lagoudakis2003reinforcement,lazaric2010analysis,farahmand2014classification}, which view the one-step greedy update (e.g. in Q-learning) as a binary classification problem.
However, a major difference is the labeling: classification-based approximate policy iteration labels the action with the largest Q value as positive; Generalized PPO-Clip labels the actions with positive advantage as positive.
Despite the high-level resemblance, our paper is fundamentally different from the prior works \citep{lagoudakis2003reinforcement,lazaric2010analysis,farahmand2014classification} as our paper is meant to study the theoretical foundation of PPO-Clip, from the perspective of hinge loss.

\section{{Comparison of the Clipped Objective and the Generalized PPO-Clip Objective}}
\label{app:compare PPO-Clip and HPO}
Recall that the original objective of PPO-Clip is
\begin{equation}
    L^{\text{clip}}(\theta) =\mathbb{E}_{s\sim d_{\mu_{0}}^{\pi},a\sim\pi(\cdot|s)}\big[\min\{\rho_{s,a}(\theta)A^{\pi}(s,a), \text{clip}(\rho_{s,a}(\theta),1-\epsilon,1+\epsilon)A^{\pi}(s,a)\}\big],
\end{equation}
where $\rho_{s,a}(\theta)=\frac{\pi_{\theta}(a\rvert s)}{\pi(a\rvert s)}$.
In practice, $L^{\text{clip}}(\theta)$ is approximated by the sample average as
\begin{align}
    L^{\text{clip}}(\theta) &\approx \hat{L}^{\text{clip}}(\theta)=\frac{1}{\lvert \cD\rvert} \sum_{(s,a)\in \cD}\min\{\rho_{s,a}(\theta){A}^{\pi}(s,a), \text{clip}(\rho_{s,a}(\theta),1-\epsilon,1+\epsilon){A}^{\pi}(s,a)\}\\
    &=\frac{1}{\lvert \cD\rvert} \sum_{(s,a)\in \cD}\lvert {A}^{\pi}(s,a)\rvert \cdot\underbrace{\min\{\rho_{s,a}(\theta)\sgn({A}^{\pi}(s,a)), \text{clip}(\rho_{s,a}(\theta),1-\epsilon,1+\epsilon)\sgn({A}^{\pi}(s,a))\}}_{=:{H}^{\text{clip}}_{s,a}(\theta)}.
\end{align}
Note that ${H}^{\text{clip}}_{s,a}(\theta)$ can be further written as
\[    {H}^{\text{clip}}_{s,a}(\theta)=
\begin{cases}
        1+\epsilon&, \text{if } {A}^{\pi}(s,a)>0 \text{ and } \rho_{s,a}(\theta)\geq 1+\epsilon  \\
        \rho_{s,a}(\theta)&, \text{if } {A}^{\pi}(s,a)>0 \text{ and } \rho_{s,a}(\theta)<1+\epsilon\\
        -\rho_{s,a}(\theta)&, \text{if } {A}^{\pi}(s,a)<0 \text{ and } \rho_{s,a}(\theta)>1-\epsilon\\
        -(1-\epsilon)&, \text{if } {A}^{\pi}(s,a)<0 \text{ and } \rho_{s,a}(\theta)\leq 1-\epsilon\\
        0&, \text{otherwise}
    \end{cases}\]
{Recall that the generalized objective of PPO-Clip with hinge loss takes the form as}
\begin{align}
    L(\theta) &\approx \hat{L}(\theta)=\frac{1}{\lvert \cD\rvert} \sum_{(s,a)\in \cD}\lvert {A}^{\pi}(s,a)\rvert \cdot\underbrace{\max\big\{0,\epsilon-(\rho_{s,a}(\theta)-1)\sgn({A}^{\pi}(s,a))\big\}}_{=:{H}_{s,a}(\theta)}.
\end{align}
Similarly, ${H}_{s,a}(\theta)$ can be further written as
\[    {H}_{s,a}(\theta)=
\begin{cases}
        0&, \text{if } {A}^{\pi}(s,a)>0 \text{ and } \rho_{s,a}(\theta)\geq 1+\epsilon  \\
        -\rho_{s,a}(\theta)+(1+\epsilon)&, \text{if } {A}^{\pi}(s,a)>0 \text{ and } \rho_{s,a}(\theta)<1+\epsilon\\
        \rho_{s,a}(\theta)-(1-\epsilon)&, \text{if } {A}^{\pi}(s,a)<0 \text{ and } \rho_{s,a}(\theta)>1-\epsilon\\
        0&, \text{if } {A}^{\pi}(s,a)<0 \text{ and } \rho_{s,a}(\theta)\leq 1-\epsilon\\
        \epsilon&, \text{otherwise}
    \end{cases}\]
Therefore, it is easy to verify that $\hat{L}^{\text{clip}}(\theta)$ and $-\hat{L}(\theta)$ only differ by a constant with respect to $\theta$. This also implies that $\nabla_{\theta}\hat{L}^{\text{clip}}(\theta)= -\nabla_{\theta}\hat{L}(\theta)$.

\end{document}